\documentclass[chaparabic,ee,phd,12pt,oneandhalf]{metu}
\pdfoutput=1
\usepackage{appendix}

\author{Mehmet Altan Toksöz}
\title{Basic Thresholding Classification}
\turkishtitle{Temel Eşikleme Sınıflandırma}
\date{March 2016}
\director[prof]{Gülbin Dural Ünver}
\headofdept[prof]{Gönül Turhan Sayan}
\supervisor[assocprof]{İlkay Ulusoy}
\departmentofsupervisor{Electrical and Electronics Eng. Dept., METU}
\committeememberi[prof]{Çağatay Candan}
\affiliationi{Electrical and Electronics Engineering Department, METU}
\committeememberii[assocprof]{İlkay Ulusoy}
\affiliationii{Electrical and Electronics Engineering Department, METU}
\committeememberiii[assocprof]{Alptekin Temizel}
\affiliationiii{Graduate School of Informatics, METU}
\committeememberiv[assocprof]{Selim Aksoy}
\affiliationiv{Computer Engineering Department, Bilkent University}
\committeememberv[assocprof]{Pınar Duygulu Şahin}
\affiliationv{Computer Engineering Department, Hacettepe University}
\keywords{basic thresholding classifier (BTC), kernel basic thresholding classifier (KBTC), sufficient indentification condition (SIC), face identification, hyper-spectral image classification, support vector machines (SVM), multinomial logistic regression (MLR), simultaneous orthogonal matching pursuit (SOMP)}
\anahtarklm{temel eşikleme sınıflandırıcı (BTC), çekirdek temel eşikleme sınıflandırıcı (KBTC), yeterli tanıma şartı (SIC), yüz tanıma, hiper-spektral imge sınıflandırma, destek vektör makineleri (SVM),  çok terimli lojistik regresyon (MLR), eşzamanlı dik eşleştirme takibi (SOMP)}
\abstract{In this thesis, we propose a light-weight sparsity-based algorithm, basic thresholding classifier (BTC), for classification applications (such as face identification, hyper-spectral image classification, etc.) which is capable of identifying test samples extremely rapidly and performing high classification accuracy. Originally BTC is a linear classifier which works based on the assumption that the samples of the classes of a given dataset are linearly separable. However, in practice those samples may not be linearly separable. In this context, we also propose another algorithm namely kernel basic thresholding classifier (KBTC) which is a non-linear kernel version of the BTC algorithm. KBTC can achieve promising results especially when the given samples are linearly non-separable. For both proposals, we introduce sufficient identification conditions (SICs) under which BTC and KBTC can identify any test sample in the range space of a given dictionary. By using SICs, we develop parameter estimation procedures which do not require any cross validation. Both BTC and KBTC algorithms provide efficient classifier fusion schemes in which individual classifier outputs are combined to produce better classification results. For instance, for the application of face identification, this is done by combining the residuals having different random projectors. For spatial applications such as hyper-spectral image classification, the fusion is carried out by incorporating the spatial information, in which the output residual maps are filtered using a smoothing filter. Numerical results on publicly available face and hyper-spectral datasets show that our proposal outperforms well-known support vector machines (SVM)-based techniques, multinomial logistic regression (MLR)-based methods, and sparsity-based approaches like $l_1$-minimization and simultaneous orthogonal matching pursuit (SOMP) in terms of both classification accuracy and computational cost.}
\oz{Bu tezde, bazı sınıflandırma uygulamaları (yüz tanıma, hiper-spektral imge sınıflandırma vb.) için yüksek sınıflandırma doğruluğu ile test numunelerini son derece hızlı bir şekilde sınıflandırabilen seyreklik tabanlı temel eşikleme sınıflandırıcı (BTC) önerilmektedir. Orjinalinde BTC doğrusal bir sınıflandırıcı olup verilen bir veri kümesinin sınıflarına ait örneklerin doğrusal olarak ayırt edilebilir varsayımı üzerine çalışmaktadır. Ancak pratikte bu örnekler doğrusal olarak her zaman ayırt edilemeyebilmektedir. Bu kapsamda, BTC'nin doğrusal doğrusal olmayan çekirdek versiyonu KBTC de ayrıca takdim edilmektedir. Özellikle doğrusal bir şekilde ayırt edilemeyen örnekler verildiğinde KBTC gelecek vaat eden sonuçlar elde edebilmektedir. Verilen bir sözlüğün değer kümesi uzayında bulunan herhangi bir test örneği, takdim edilen yeterli tanıma koşulu (SIC) altında sınıflandırılabilmektedir. Bu koşul kullanılarak çapraz doğrulama gerektirmeyen parametre kestirim yöntemleri geliştirilmiştir. BTC ve KBTC algoritmaları, sınıflandırma doğruluğunu arttırmak için bireysel sınıflandırıcı çıkışlarını birleştirerek füzyon tekniklerinin etkili bir şekilde uygulanmasını sağlamaktadır. Örneğin bu işlem, yüz tanıma uygulamaları için farklı rasgele projektörlere sahip sınıflandırıcıların çıkıştaki artık değerlerinin birleştirilmesiyle yapılır. Füzyon işlemi, hiper-spektral imge sınıflandırma gibi uzamsal uygulamalarda ise uzamsal bir filtre kullanılarak çıkıştaki artık değer haritalarının düzleştirilmesiyle yapılmaktadır. Bazı yaygın yüz ve hiper-spektral veri setleri kullanılarak gerçekleştirilen deneylerde, önerdiğimiz BTC ve KBTC algoritmaları, sınıflandırma doğruluğu ve maliyeti açısından, tanınmış destek vektör makineleri (SVM) tabanlı teknikler, çok terimli lojistik regresyon tabanlı metotlar, $l_1$-minimizasyonu ve eşzamanlı dik eşleştirme takibi (SOMP) gibi seyreklik tabanlı yaklaşımlardan daha iyi sonuçlar vermektedir.} 
\dedication{\textit{To my family and people who are reading this page}}
\acknowledgments{
I gratefully thank my supervisor Assoc. Prof. Dr. İlkay Ulusoy, who has shared her expertise with me throughout the development of the thesis, for her support, constructive criticism and supervision. I greatly appreciate her patience and belief on me and also her guidance related to my academic life.

I express my gratitude to thesis monitoring committee members Prof. Dr. Çağatay Candan and Assoc. Prof. Dr. Alptekin Temizel. They have closely monitored the progress of the study and the quality of the thesis has been remarkably improved with their valuable comments and fruitful discussions with them. In particular, I would like to say how lucky I was to meet Dr. Candan who has taught a lot of advanced subjects in signal analysis and processing course. 

I would like to thank thesis jury members Assoc. Prof. Dr. Selim Aksoy and Assoc. Prof. Dr. Pınar Duygulu Şahin for sparing their time and effort in the evaluation of the thesis.

I would like to express my gratitude to Tübitak Bilgem İltaren and also my colleagues there for all the support on the study.

I would also like to thank  Joel Tropp, Xudong Kang, Jun Li, Julien Mairal, Laurens van der Maaten, Chih-Chung Chang, and Dani Lischinski for kindly providing the source codes and toolboxes used in this thesis.
 
}
\usepackage[pdftex]{hyperref}
\usepackage[all]{hypcap}
\usepackage{todonotes}
\usepackage{epstopdf}
\usepackage{graphicx}
\usepackage{caption}
\usepackage{subcaption}
\usepackage{rotating}
\usepackage{xy} 
\usepackage{pdfpages}
\usepackage{array}
\usepackage{color}
\usepackage{tcolorbox}
\definecolor{light-gray}{gray}{0.95}
\tcbset{width=1.0\textwidth,boxrule=0pt,colback=light-gray,arc=0pt,auto outer arc,left=0pt,right=0pt,boxsep=5pt}
\usepackage{amsmath}
\usepackage{amssymb}
\usepackage{algorithm}
\usepackage{algpseudocode}
\usepackage{amsthm}
\newcommand{\norm}[1]{\left\lVert#1\right\rVert}
\newtheorem{theorem}{Theorem}[section]
\newtheorem{prop}[theorem]{Proposition}
\usepackage{secdot}
\newcommand{\noop}[1]{} 
\usepackage{textcomp}
\definecolor{dkgreen}{rgb}{0,0.6,0}
\definecolor{gray}{rgb}{0.5,0.5,0.5}
\usepackage[autolinebreaks]{mcode1}
\usepackage{textcomp}

\usepackage[utf8]{inputenc}

\begin{document}
\begin{preliminaries}

\begin{theglossary}{LONGESTABBRV}
\item[AA] Average Accuracy
\item[ANN] Artificial Neural Networks
\item[ASOMP] Adaptive Simultaneous Orthogonal Matching Pursuit
\item[AVIRIS] Airborne Visible/Infrared Imaging Spectrometer
\item[BP] Back Propagation
\item[BPNN] Back Propagation Neural Networks
\item[BT] Basic Thresholding
\item[BTC] Basic Thresholding Classifier
\item[COMP] Cholesky-based Orthogonal Matching Pursuit
\item[FFBPNN] Feed Forward Back Propagation Neural Networks
\item[GF] Guided Filter
\item[HSI] Hyper-spectral Image
\item[KBTC] Kernel Basic Thresholding Classifier
\item[LDA] Linear Discriminant Analysis
\item[LORSAL] Logistic Regression via Splitting and Augmented Lagrangian
\item[MASR] Multi-scale Adaptive Sparse Representation
\item[MLL] Multilevel Logistic
\item[MLR] Multinomial Logistic Regression 
\item[NN] Nearest Neighbor
\item[OA] Overall Accuracy
\item[OAA] One-Against-All
\item[OAO] One-Against-One
\item[OMP] Orthogonal Matching Pursuit
\item[PCA] Principal Component Analysis
\item[RBF] Radial Basis Functions
\item[ROC] Receiver Operating Characteristics
\item[ROSIS] Reflective Optics System Imaging Spectrometer
\item[SCI] Sparsity Concentration Index
\item[SIC] Sufficient Identification Condition
\item[SOMP] Simultaneous Orthogonal Matching Pursuit
\item[SRC] Sparse Representation-based Classification
\item[SVM] Support Vector Machines
\item[WLS] Weighted Least Squares
\item[WSOMP] Weighted Simultaneous Orthogonal Matching Pursuit

\end{theglossary}

\end{preliminaries}
%
%
%


\chapter{INTRODUCTION}
\label{chp:introduction}

The term classification is usually referred to as assigning objects into different categories. It can also be interpreted as assigning predefined class labels to each object under testing. Those objects could be human faces, fingerprints, characters, digits, signals, documents, diseases, proteins, genes, speech, emotions, galaxies, objects in a scene, hyper-spectral pixels, etc.

The classification / recognition / identification process can easily be performed by human brains. On the other hand, it is not quite easy for the artificial classifiers or recognizers because the process involves quite complex stages such as sensing, preprocessing, feature extraction, dimension reduction, and decision making. Fig. \ref{classification} shows a typical classification scheme. Sensors such as camera, microphone or other type of acquisition devices are able to capture high quality raw data in today's technology. Although the sensing and preprocessing stages are performed easily, the feature extraction stage could be problematic. Perhaps, the performance of a classifier is mostly affected by the quality of the features. A good feature is considered to be discriminative, informative, robust, reliable, independent, invariant to scale and transformation.

\begin{figure*}
\centering
\includegraphics[width = 1.0\textwidth]{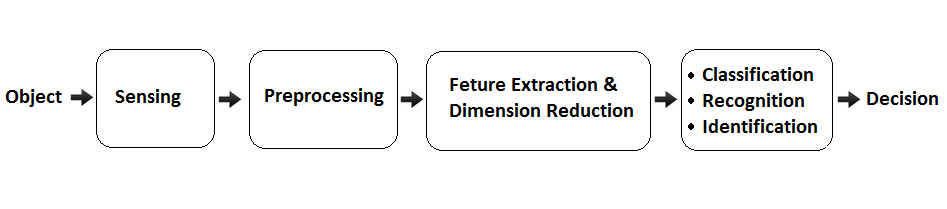}
\centering
\caption{A typical classification scheme}
\label{classification}
\end{figure*}

Features in a classification problem could be color, shape, texture, size, sound, intensity of a pixel, gender, height, weight, measured frequency value, etc. They are generally mapped to real line and stored in the vectors. An $N$-dimensional feature vector can also be interpreted as a point in an $N$-dimensional feature space. Sometimes using all of the extracted features does not improve the classification process. Instead, it may cause performance degradation in terms of accuracy and speed. To overcome this problem, an information reduction stage namely dimension reduction is integrated to the whole system.

The final stage of a classification process, decision making, is performed by a classifier. if high quality features are passed to the classifier, the objects under test can easily be classified. However, in real world applications, extracting good features from the raw data is not always possible. In this case, in order to increase classification accuracy, we need to design sophisticated classifiers. These kinds of classifiers not only perform high classification accuracy but also act quickly. They also require few processing steps and memory.

In the literature, variety of classifiers have been proposed addressing the classification problem. They can be divided into two main categories namely parametric and non-parametric approaches. One example of parametric techniques is the Bayesian decision theory in which a priori probability densities are known for each class. Those densities are converted to a posteriori probabilities and final decisions are made based on them \cite{duda2012pattern}. Unfortunately, in practice, those densities are generally unknown. Therefore, it is inevitable to use non-parametric approaches. There are various methods in this category such as density or parameter estimation-based techniques in which the underlying densities and parameters are estimated based on the provided training data. Sometimes decision boundaries are formed using the training data, which is referred to as learning. If the class labels of the training data are known, then this kind of learning is called supervised learning. If there is no labeled data, in this case, the learning process becomes unsupervised learning or clustering.

Recently, sparse representation-based classification algorithms in the category of supervised techniques have got significant attention \cite{wright2009robust,wright2010sparse,mei2011robust,yuan2012visual,zhang2011sparse,gao2010kernel,yang2011fisher}. Over the last two decades, tremendous research activities have been observed in the area of sparse signal representation and compressed sensing \cite{donoho2006compressed,donoho2003optimally,candes2006stable,donoho2006most,chen2001atomic}. This is mainly because of the fact that significant portion of the signals in the nature are sparse, that is, most of the components of them are zero (Fig. \ref{sparseSignal}). Sparsity provides that real world signals can be represented by the combinations of a few basis vectors. For instance, a typical image can be successfully compressed via JPEG technique which works based on the assumption that an image can be represented by a few discrete cosine basis.

\begin{figure*}
\centering
\includegraphics[width = 1.0\textwidth]{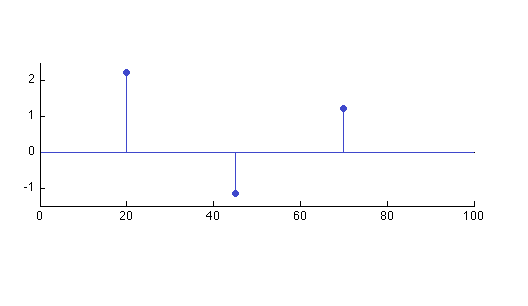}
\centering
\caption{A sparse signal}
\label{sparseSignal}
\end{figure*}

While the nice properties of compressible or sparse signals in the area of signal processing are inspiring, the computer vision community is more interested in the semantic information of a signal rather than compressed sensing and compact representation. For instance, state-of-the-art results have been achieved based on sparse representation in image de-noising and in-painting \cite{elad2006image,mairal2009non,mairal2008sparse}, image super-resolution \cite{yang2010image,yang2008image,kim2010single,yang2012coupled}, object tracking \cite{mei2011robust,jia2012visual,wang2012online}, face recognition \cite{wright2009robust,zhang2011sparse,wagner2012toward,yang2010gabor,zhang2016sample,xu2016approximately}, image smoothing \cite{xu2011image}, image classification \cite{yang2009linear,gao2010kernel,rigamonti2011sparse,liu2016structure,huang2016efficient}, etc. 

Although sparse representation-based techniques achieve promising results, the underlying framework, sparse signal recovery or reconstruction, is not an easy task. Until now, various algorithms have been proposed addressing the problem of sparse reconstruction. Convex relaxation and $l_1$-minimization-based techniques \cite{donoho2003optimally,candes2005l1,candes2008enhancing,becker2011nesta}, greedy approaches \cite{tropp2004greed,needell2009cosamp,tropp2006algorithms,donoho2012sparse}, Bregman iteration-based procedures \cite{goldstein2009split,yin2008bregman}, and linear programming-based \cite{candes2005decoding} methods have been deeply investigated. Unfortunately, while some of those approaches are extremely costly, the others are highly sensitive to noise and corruption. As we stated previously, most of the classification applications involve noisy and corrupted features such as face recognition under illumination variations, noise, and corruption. In some applications such as hyper-spectral image classifications, the problems involve classification of hundred thousands of pixels, which requires extremely cost effective classifiers.           

In this thesis, we propose two sparsity-based, light-weight, and easy-to-implement classification algorithms which achieve state-of-the-art results in terms of both accuracy and computational cost. While the first algorithm addresses the applications in which the samples of data are linearly separable, the other one refers the problems involving non-linearly separable data. The following section briefly presents the organization of the thesis.

\section{Outline}

This thesis provides the following contributions to the field of computer vision and classification:

\begin{itemize}
\item In Chapter 2, we briefly discuss some of the common existing techniques in classification including non-parametric approaches, neural network-based methods, and sparsity-based classifiers. 
\item In Chapter 3, we introduce the basic thresholding classification (BTC) algorithm and give the construction of it step by step. We also provide necessary guidance for parameter estimation.
\item Chapter 4 introduces the kernel basic thresholding classification (KBTC) algorithm which achieves promising results in the problems especially involving non-linearly separable data. We present full guidance of the parameter estimation steps by utilizing the propositions related to the algorithm. 
\item In Chapter 5, the performance of the BTC algorithm is compared to those of the state-of-the-art sparsity-based techniques in the application of face identification. We also provide an effective classification fusion technique in which individual classifiers are combined to achieve better classification performance. At the end of the chapter, an efficient validation scheme is presented in order to reject invalid test samples.
\item Chapter 6 compares the performance of the BTC technique with those of the powerful non-linear kernel methods such as SVM in the application of hyper-spectral image classification. This chapter also introduces a spatial-spectral framework in which the residual maps produced by the proposed algorithms are smoothed using edge preserving filtering techniques. This intermediate step extremely improves the classification accuracy.
\item In Chapter 7, we compare the performance of the KBTC algorithm with those of the state-of-the-art non-linear kernel approaches as well as the linear BTC in hyper-spectral image classification. This chapter shows how the non-linear similarity-based KBTC achieves significant performance improvements over the linear ones as well as the other techniques. We also provide fixed training sets by which efficient comparison of the algorithms is performed.
\item Finally, Chapter 8 concludes this thesis by presenting summary and future directions.
\end{itemize}


\begin{tcolorbox}
Please note that the Chapter 5, 6, and 7 are based on the following papers:\\

M. A. Toksoz and I. Ulusoy, “Hyperspectral image classification via kernel basic thresholding classifier,” IEEE Transactions on Geoscience and Remote Sensing, vol. 55, no. 2, pp. 715–728, 2017.\\

M. A. Toksöz and I. Ulusoy, “Hyperspectral image classification via basic thresholding classifier,” IEEE Transactions on Geoscience and Remote Sensing, vol. 54, no. 7, pp. 4039–4051, 2016.\\

M. A. Toksöz and I. Ulusoy, “Classification via ensembles of basic thresholding classifiers,” IET Computer Vision, vol. 10, no. 5, pp. 433–442, 2016.
\end{tcolorbox}

\section{Notations}

Throughout the thesis we will use some notations which are described as follows:

\begin{itemize}
\item We will use capital letters for matrices and sets. For instance, a given dictionary consisting of training samples will be shown by the matrix $A$. Exceptionally, the threshold parameter, the number of features, and the number of classes will be represented by $M$, $B$, and $C$, respectively.
\item Small letters will be used to represent vectors. In classification applications, typically the small letter $y$ is used to describe a given test sample which is a vector containing features. Exceptionally, the small letters $i$, $j$, $k$, $m$, and $n$ will be used for indexes. 
\item The $i$th column of a dictionary $A$ will be shown by $A(i)$ which corresponds to a training sample. If we want to extract the sub matrix whose indexes in $\Lambda$, we will use the notation $A(\Lambda)$. 
\item Small Greek letters such as $\alpha$ and $\gamma$ will denote the constants. 
\item The notation $A_i$ will represent the sub matrix which contains only the training samples of the $i$th class. A sample belonging to $i$th class will be denoted by $a_i$.
\item Finally, the range space of a vector will be shown by $\mathcal{R}(.)$ and the notations $\norm{.}_0$, $\norm{.}_1$, and $\norm{.}_2$ will represent the $l_0$, $l_1$, and $l_2$ norms, respectively.    
\end{itemize}


\chapter{EXISTING METHODS}
\label{chp:existingMethods}
In this chapter, we briefly discuss commonly used classification algorithms in the literature. 

\section{Nearest Neighbor Classifier}
One of the most intuitive and primitive methods in the class of non-parametric techniques is the nearest neighbor (NN) classifier. The classification is simply performed based on the Euclidean distances between testing and training samples in the feature space. It has been shown in \cite{cover1967nearest} that when the size of the training data goes to infinity, its error rate does not exceed the double Bayes error rate. 

NN classifier is commonly used in the classification applications such as face recognition. Usually features are transformed before using it. One of the most popular subspace-based transformation methods is the principal component analysis (PCA) approach \cite{turk1991eigenfaces,turk1991face}. In PCA, high dimensional data is projected to a lower dimensional subspace in which first principal component carries most discriminative information. Another popular approach of the subspace-based algorithms is the linear discriminant analysis (LDA) method, also known as Fisher's LDA \cite{martinez2001pca}. The goal of the LDA is to seek a projection to a lower dimensional subspace such that maximum separability is obtained between samples of different classes. This is achieved by maximizing the ratio $\frac{|V^TS_bV|}{|V^TS_wV|}$, where $S_b$ is the between-class scatter matrix, $S_w$ is the within-class scatter matrix, and $V$ is the projection which can be obtained from the eigenvectors of $S_w^{-1}S_b$. Although LDA is a powerful dimensionality reduction method, it encounters a common high dimensionality problem in the classification applications. In other words, the size of the feature vectors is generally larger than the number of training samples. This makes the $S_w$ matrix singular, which is usually called small sample size problem (SSS). 

Although NN-based approaches work well under normal conditions, they are highly sensitive to corruption and noise in the features. A more sophisticated version of NN classifier is the k-NN technique which executes majority voting among $k$ nearest training samples \cite{xu2013coarse}. As shown in \cite{melgani2004classification}, its performance does not exceed those of linear and non-linear SVM classifiers. First, second, and third nearest neighbors of a test instance could be seen in Fig. \ref{knn}.    

\begin{figure*}
\centering
\includegraphics[width = 0.8\textwidth]{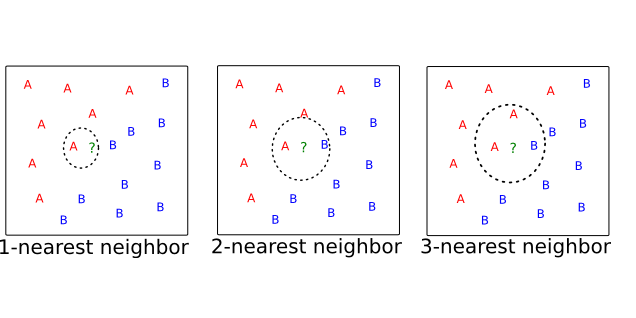}
\centering
\caption{Nearest neighbors of a test instance}
\label{knn}
\end{figure*}
      
\section{Neural Networks}
Over the last two decades, artificial neural networks (ANN) has become more important in computer vision and pattern recognition. One of the most popular ANN-based methods is the back-propagation (BP) algorithm which is a gradient based method \cite{rumelhart1988learning}. At every epoch, the algorithm adjusts the connection weights in the network such that the difference of the desired and actual output vector is minimized. BPNN was successfully applied to hand written digit recognition \cite{ciresan2010deep,le1990handwritten}. Recently, FFBPNN has been applied to face recognition using PCA \cite{kashem2011face,latha2009face}. 

\begin{figure*}
\centering
\includegraphics[width = 0.7\textwidth]{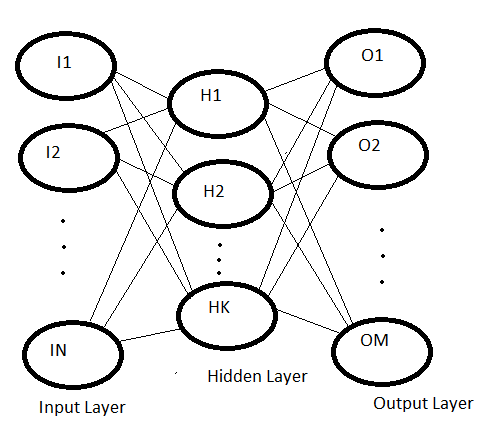}
\centering
\caption{A three-layer neural network}
\label{neuralNetwork}
\end{figure*}

Fig. \ref{neuralNetwork} shows a structure of a typical network which has three layers, namely, input, hidden, and output ones. In a classification application, generally, the number of input units is equal to the length of feature vectors. The number of hidden units could be determined experimentally, and finally the number of output units is equal to the number of classes in the application. There could be more than one hidden layers in the network. Very recently, deep convolutional neural networks (CNN) has been successfully applied to hyper-spectral image classification \cite{hu2015deep}. The technique becomes superior to the SVM approach, however, it requires more training samples than a conventional classifier. 
\section{Support Vector Machines (SVM)}
SVM is a binary classification technique in which a maximum distance decision surface is found between closest points (support vectors) of two classes \cite{cortes1995support}. The points are assumed to be linearly separable. In case they are inseparable, a penalty factor $C$ is utilized. Alternatively, a non-linear kernel (RBF, Polynomial, etc.) is used to determine a non-linear decision boundary. The hyper-plane found by the algorithm has maximum distances to the support vectors (Fig. \ref{svmHyperplane}).

\begin{figure*}
\centering
\includegraphics[width = 0.8\textwidth]{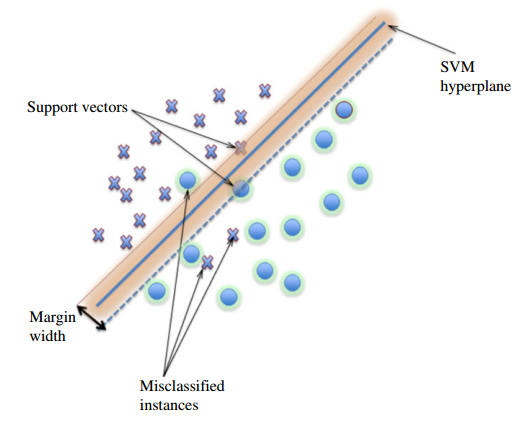}
\centering
\caption{Support vector machines}
\label{svmHyperplane}
\end{figure*}

Since the algorithm is originally a binary classification technique, it can not be directly applied to a multi-class classification problem. There are various strategies proposed in the literature to extend the binary SVMs to multi-class SVMs. The two famous of them namely one-against-all (OAA) and one-against-one (OAO) strategies are quite common for multi-class problems. In OAA approach, the number of binary SVMs is equal to the number of classes in the training set. Each SVM finds a separating hyper plane between $i$th class and the rest of the classes. There are some drawbacks of this strategy. First, the number of train pixels per class is unbalanced since the rest of the classes has more train pixels than the $i$th class. Second, the size of one SVM classifier is very large and it requires large memories. In OAO approach, for every possible pair of classes there exists a binary SVM. This approach is a symmetric one and every SVM requires less memory than the OAA approach. However, the number of classifiers is larger than the OAA strategy. Therefore, we can say that during the classification procedure, it requires more time. SVM approach has been applied to plenty of classification problems in the literature including text classification \cite{joachims1998text,joachims1999transductive,chapelle1999support,tong2002support}, face recognition \cite{guo2000face,heisele2001face,phillips1998support}, protein classification \cite{leslie2004mismatch,leslie2002spectrum,eskin2002mismatch,tsuda2005fast}, gene classification \cite{guyon2002gene,furey2000support}, hyper-spectral image classification \cite{melgani2004classification,fauvel2008spectral,bazi2006toward,tarabalka2010svm,camps2005kernel}, spam categorization \cite{drucker1999support}, etc. Although this technique is quite common, its parameters are determined using cross validation which may result in non-optimal parameter set.

\section{Sparse Representation-Based Classification (SRC)}
As we stated in the previous chapter, most of the signals in the nature are sparse and a few components of them carry information. The major challenge in sparse approximation / reconstruction is to recover the original signal $x\in\mathbb{R}^{N}$ using a few measurements $y \in\mathbb{R}^B$ and the sensing (measurement) matrix $A \in\mathbb{R}^{B\times N}$ (Fig. \ref*{sparseRecovery}). It is assumed that each column of $A$ has unit Euclidean norm and the system of equations is mostly under-determined ($B \ll N$). 

\begin{figure*}
\centering
\includegraphics[width = 1.0\textwidth]{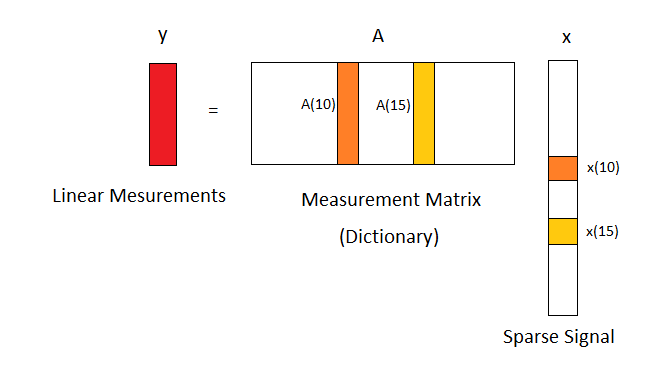}
\centering
\caption{Sparse representation}
\label{sparseRecovery}
\end{figure*}

The reconstruction task could be written as the following minimization problem:

\begin{equation}\label{sparseL0}
\hat{x}=\arg\min_x \norm{x}_0~subject~to~y=Ax~~or~~subject~to~\norm{y-Ax}_2 \le \epsilon
\end{equation}

where $\norm{x}_0$ is $l_0$ norm of $x$, which counts the number of non-zero components in $x$ ,and $\epsilon$ is small error tolerance. Unfortunately, the problem is not tractable and it is NP-hard. One intuitive solution could be found by replacing $l_0$ norm with $l_2$ norm (Euclidean norm) which results in Tikhonov regularization \cite{golub1999tikhonov}. However, this technique does not produce sparse solution. Alternative approach is to replace $l_0$ norm with $l_1$ since it is convex and close to $l_0$ function. Luckily, this method provides sparse solution. The comparison between these two approaches is demonstrated in Fig. \ref{L1L2}. 

\begin{figure*}
\centering
\includegraphics[width = 1.0\textwidth]{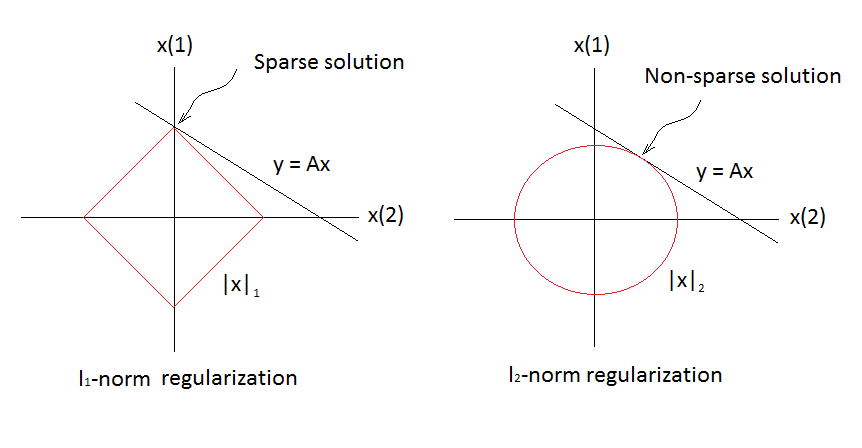}
\centering
\caption{$l_1$ versus $l_2$ regularized solutions}
\label{L1L2}
\end{figure*}

We can also compare the two methods by performing an experiment. Assume that we have a sparse signal $x\in\mathbb{R}^{512}$ having $15$ non-zero entries and a sensing matrix $A \in\mathbb{R}^{170\times 512}$ whose entries are filled with numbers drawn from Gaussian distribution. The observation vector $y$ can be formed by decoding $x$ via $A$, that is, $y=Ax$. We can then recover $x$ given $A$ and $y$ by performing $l_1$ and $l_2$ regularizations. The result could be seen in Fig. \ref{expL1L2}. As we can observe, $l_1$ regularization approach exactly recovers the original $x$, having very tiny ripples. On the other hand, the result obtained by $l_2$ regularization is highly dense and only a few peaks can be observed. 

\begin{figure*}
\centering
\includegraphics[width = 1.0\textwidth]{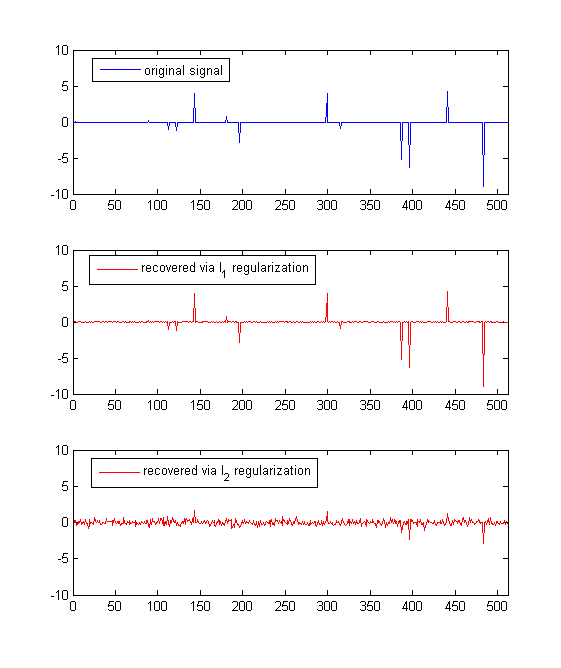}
\centering
\caption{An experiment: $l_1$ versus $l_2$ regularized solutions}
\label{expL1L2}
\end{figure*}
 
$l_1$ regularized solution could be obtained via some convex relaxation methods such as basis pursuit \cite{chen2001atomic}, least absolute shrinkage and selection operator (LASSO) \cite{tibshirani1996regression}, homotopy \cite{osborne2000new}, etc. The common problems of these techniques are heavy computational complexity. Greedy approximations such as orthogonal matching pursuit (OMP) \cite{tropp2004greed}, compressive sampling matching pursuit (CoSaMP) \cite{needell2009cosamp}, stage-wise orthogonal matching pursuit (StOMP) \cite{donoho2012sparse}  are available in the literature to reduce the computational burden in sparse approximation.   

Recently, sparse representation has been successfully applied to classification problems by exploiting the fact that identity information (class identity) of a given test sample is sparse among the other classes \cite{wright2009robust,wright2010sparse}. In this case, the matrix $A$ does not contain the basis elements but the labeled training samples. The measurement vector $y$ becomes the sample to be classified and the sparse signal $x$ is interpreted as the vector consisting of identity information. The new interpretation is shown in Fig. \ref{sparseClassification}.

\begin{figure*}
\centering
\includegraphics[width = 1.0\textwidth]{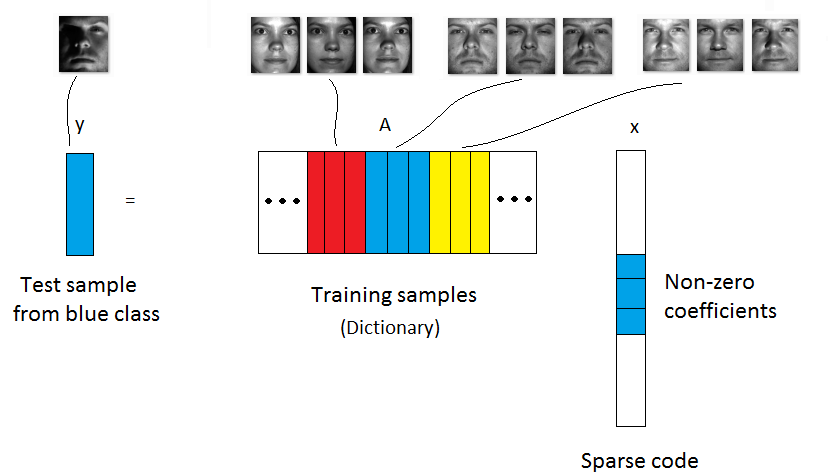}
\centering
\caption{Sparse representation in classification}
\label{sparseClassification}
\end{figure*}

After finding the sparse code using $l_1$-minimization in the new configuration, class identity of a given test sample is estimated via class-wise regression error. That is,

\begin{equation}\label{sparseRegression}
 class(y)=\arg \min_i \norm{y-A \hat{x_i}}_2 \; \forall i \in \{1,2,\ldots,C\}
\end{equation}
where $\hat{x_i}$ represents the $i$th class portion of the estimated sparse code $\hat{x}$ among $C$ many classes. The class estimation stage could also be seen in Fig. \ref{src35}.

\begin{figure*}
\centering
\includegraphics[width = 1.0\textwidth]{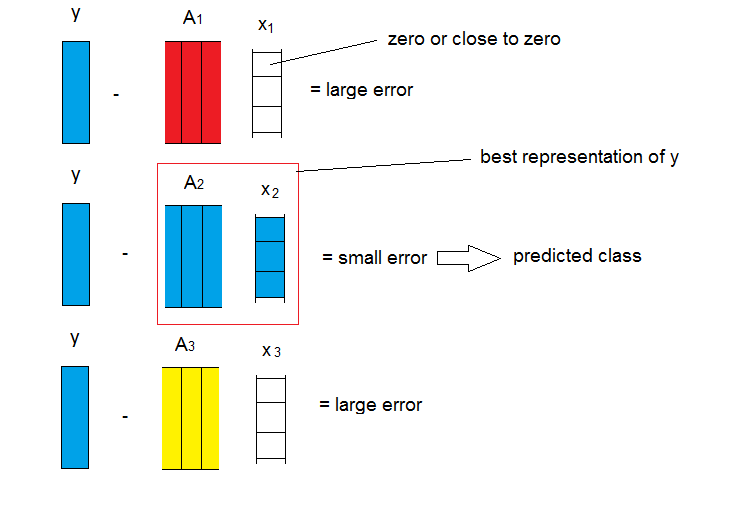}
\centering
\caption{Class estimation via class-wise regression error}
\label{src35}
\end{figure*}

The most crucial side of this approach is low sensitivity to corrupted features because of the fact that the errors due to these kinds of features are often sparse with respect to the dictionary elements \cite{wright2009robust}. Although this approach achieves state-of-the-art results, its sparse recovery part based on the $l_1$ minimization is extremely costly and infeasible to apply huge size problems such as hyper-spectral image classification. Alternative approaches such as collaborative representation-based classification (CRC) based on $l_2$ regularization have been proposed to reduce the computational cost. However, as observed in Fig. \ref{expL1L2}, the method uses non-sparse solution which only provides satisfactory results if the problem contains very large number of features. Therefore, this approach fails in most of the conventional classification problems involving a moderate number of features.


\chapter{BASIC THRESHOLDING CLASSIFICATION}
\label{chp:btc}

All the limitations with the described algorithms in the previous chapter force us to seek a method for classification problems having the following properties:
\begin{itemize}
\item It provides high classification accuracy.
\item It is robust, cost effective, and fast.
\item It is easy to implement.
\item It has no parameter tuning experimentally.
\end{itemize}
In this context, we propose the BTC algorithm for classification applications which approximately satisfies aforementioned properties. BTC is motivated by \emph{basic thresholding} (BT) algorithm which could be considered one of the simplest techniques in compressed sensing theory \cite{rauhut2008compressed, foucart2013mathematical}. Unlike BT, which is a generic sparse signal recovery algorithm, BTC is a classifier which utilizes Tikhonov's regularization \cite{golub1999tikhonov} for the overdetermined systems and performs classification using class-wise regression error. In the following section, we will develop the BTC algorithm step by step.

\section{Basic Thresholding Classifier}
In the generic sparse signal recovery problem, BT algorithm applies the following two stages given a dictionary containing orthonormal atoms:
\begin{itemize}
\item The first step consists of selecting a subset of the atoms ($D$) from the whole dictionary $A$ which are close in angle to the signal $y$.
\item The second step estimates the sparse code for $y$ with respect to the pruned dictionary $D$, which is performed by solving the following minimization problem:
\begin{equation}\label{bt_min}
\hat{x}=\arg~\min_x \norm{y-D x }_2
\end{equation}
The solution is obtained using the ordinary least squares (OLS) technique, $\hat{x}(\Lambda) = (D^TD)^{-1}D^Ty$. Note that in the expression, only the entries indexed with $\Lambda$ are updated and the others are set to zero. The index set $\Lambda$ consists of the indexes of $M$ (threshold) largest correlations.
\end{itemize}
Here, the assumption is that $y$ is a linear combination of a subset of orthonormal basis vectors included in the dictionary. In classification problems, the assumption of orthonormal basis vectors fails since the selected subset of the atoms contains extremely correlated columns possibly from the same class. Therefore, the matrix constructed using the selected atoms of the original dictionary becomes singular and the OLS produces meaningless bad sparse approximation. One solution to this problem is to utilize Tikhonov's regularization method for the overdetermined systems. In this case, the minimization problem could be written as follows:

\begin{equation}\label{btc_min}
\hat{x}=\arg~\min_x \norm{y-D x }^2_2 + \alpha\norm{x }^2_2
\end{equation}
The solution of the minimization problem could be obtained by manipulating the cost function below,
\begin{equation}\label{tikhonov_lin_btc}
J(x)= \norm{y-D x }^2_2 + \alpha\norm{x}^2_2
\end{equation}
The expanded version of it can be written as,
\begin{equation}\label{tikhonov_lin_exp}
\begin{split}
J(x) &= (y-D x )^T(y-D x ) + \alpha x^Tx\\
&=y^Ty - 2D^Tyx +x^TD^TDx +\alpha x^Tx
\end{split}
\end{equation}
In order to find $x$ which minimizes $J(x)$, we take the gradient of $J(x)$ with respect to $x$ and set it to zero, that is,
\begin{equation}\label{tikhonov_lin_grad}
\begin{split}
\nabla J(x) &= 0\\
 - 2D^Ty + 2D^TDx +2\alpha x &=0 \\
 (D^TD +\alpha) x &= D^Ty \\
  x &= (D^TD +\alpha I)^{-1} D^Ty
\end{split}
\end{equation}
The final equation which estimates the sparse code becomes as follows,
\begin{equation}\label{tikhonov_lin_final}
  \hat{x}(\Lambda) = (D^TD +\alpha I)^{-1} D^Ty
\end{equation}
Here, $\alpha$ is a small regularization constant which is generally problem dependent. Tikhonov's regularization has the following advantages:
\begin{itemize}
\item It filters out the small or zero eigenvalues of $D^TD$.
\item It is simple to implement and it requires no complex decompositions like singular value decomposition (SVD).
\item It preserves the structure of $D^TD$ matrix and the effects of it could easily be analyzed.
\end{itemize}

\begin{figure*}
\centering
\includegraphics[width = 0.9\textwidth]{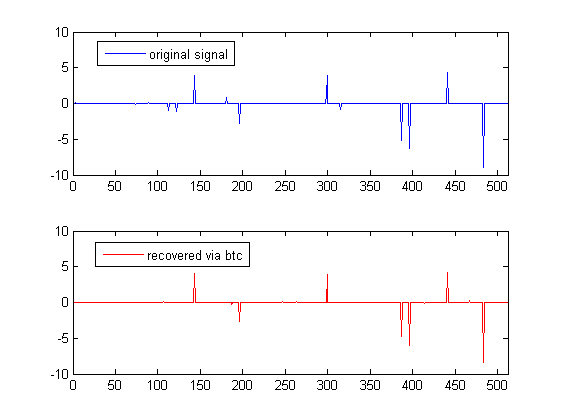}
\centering
\caption{An experiment: Sparse solution via BTC}
\label{expBtcSolution}
\end{figure*}

We can also repeat the same experiment performed in the previous chapter in order to see how BTC technique produces an efficient sparse solution. The result could be seen in Fig. \ref{expBtcSolution}. As we can observe, most of the components of sparse $x$ are perfectly recovered by BTC and only a few insignificant components of it are missed.  
  
As we stated previously, BTC not only recovers the sparse code, but also performs classification using a predetermined dictionary containing labeled training samples. It produces the class label of a test sample based on the minimal residual or equivalently class-wise regression error. The implementation is given in \textbf{Algorithm} \ref{alg:BTC}.

\alglanguage{pseudocode}
\begin{algorithm}
\caption{BTC} \label{alg:BTC}
{\normalsize \textbf{INPUT:}\\}
{\normalsize Dictionary} $A \in\mathbb{R}^{B \times N}$\\
{\normalsize Test sample} $y \in\mathbb{R}^{B}$\\
{\normalsize Threshold } $M \in \mathbb{N}$ \\
{\normalsize Regularization constant} $\alpha  \in (0,1)$ \\
{\normalsize Initial sparse coefficients with zeros} $\hat{x} \in\mathbb{R}^{N}$\\
{\normalsize \textbf{OUTPUT:}\\}
{\normalsize Class of }$y$\\
{\normalsize Residual vector} $\epsilon \in\mathbb{R}^{C}$\\
{\normalsize \textbf{PROCEDURE:} }
\begin{algorithmic}[1]
\State $v \gets A^T y$
\State $\Lambda \gets L_{M}(v)$
\State $D \gets A(\Lambda)$
\State $\hat{x}(\Lambda) \gets (D^TD+\alpha I)^{-1}D^Ty$
\State $\epsilon(j) \gets \norm{y-A_j \hat{x}_j}_2 ~~ \forall j \in \{1,2,\ldots,C\}$
\State $class(y) \gets \arg \min_j \epsilon(j)$
\end{algorithmic}
\end{algorithm}

The BTC algorithm performs the following steps:
\begin{itemize}
\item In the first step, BTC finds the correlation vector $v$ containing the linear correlations between the test sample $y$ and the samples of all training set $A$ in the original feature space via inner product.
\item In the second step, the operator $L_{M}(.)$ selects the index set $\Lambda$ consisting of the indexes of $M$ largest correlations.
\item Then, $D$ matrix is extracted from the original dictionary $A$ by means of the index set $\Lambda$.
\item The entires of the sparse code $x$ indexed with $\Lambda$ are estimated using regularized least squares technique with a small regularization constant $\alpha$.
\item Finally, it calculates the residuals for all classes, and predicts the class of $y$ based on the minimum residual.      
\end{itemize}

\section{Upper Bound for the Threshold}
Tikhonov's regularization not only helps us estimate the sparse code $x$ but also provides the upper bound for the threshold parameter $M$. Notice that this regularization technique requires the system of equations is overdetermined. Therefore, the number of columns of $D$, which is equal to the threshold parameter $M$, can not exceed the number of rows (features ($B$)) of it. We can then define the following relation,
\begin{equation}\label{MB_relation}
M < B
\end{equation}
This relation highly reduces the cost of the algorithm since $B \ll N$. A nice thing that the boundaries of $M$ provide us is that $M$ does not grow as the number of classes in the dictionary increases. For instance, assume that we have 1000 subjects and each contains 10 samples. Also suppose that we want to reduce the feature vector size to 120. In this example, $M$ will be less than 120. This will highly reduce the cost of BTC algorithm. However, for an $l_1$-minimization-based classification technique, the dictionary size will be $120\times 10^4$ which will extremely increase the convergence time. 
In the following part, we will introduce a sufficient identification condition (SIC) for BTC and based on it, we will develop a procedure to determine the best value of the parameter $M$ in the SIC sense.

\section{Sufficient Identification Condition for BTC}
In compressed sensing theory, one of the most fundamental property of a dictionary is the \emph{mutual incoherence} quantity (\ref{coherence}). 
\begin{equation}\label{coherence}
\mu \triangleq \max_{i \neq j} |<A(i),A(j)>|
\end{equation}
It simply measures how much any two elements in a dictionary look alike. Using $\mu$ one can determine under which conditions an algorithm recovers the correct sparse code. For instance, Tropp \cite{tropp2007signal} showed that if $\mu < \frac{1}{2K-1}$, then the OMP algorithm perfectly recovers any $K$-sparse vector $x$ from the measurements $y=A x$. In generic sparse signal recovery, $\mu$ is desired to be small. A dictionary having $\mu = 0.05$ will satisfy exact recovery condition for OMP to recover any sparse signal having sparsity of at most 10. Unfortunately, in classification applications, the columns of a dictionary are highly correlated. Therefore, we can not use the quantity $\mu$ to determine such conditions \cite{wright2010sparse}. \emph{Cumulative coherence} is another quantity which measures the maximum total similarity between a fixed column and the collection of other columns \cite{tropp2007signal}. Even the conditions based on this measure are not useful for classification problems because of the high correlations. 

Luckily, the BTC algorithm enables us to develop such conditions for the dictionaries in classification applications. Let us consider the following proposition:

\begin{prop}
A sufficient condition for BTC to identify a test sample $y$ belonging to the $i$th class of the dictionary $A$ is that
\begin{equation}\label{propos_btc}
\max_{j \neq i} \frac{\norm{y-A_i\hat{x}_i}_2}{\norm{y-A_j\hat{x}_j}_2} < 1
\end{equation} 
where $\hat{x_i}$ and $\hat{x_j}$ are the $i$th and $j$th class portions of $\hat{x}$, respectively.
\end{prop}

\begin{proof}
A test sample $y$ belonging to the $i$th class can be successfully identified via BTC if and only if the residual $\norm{y-A_i\hat{x}_i}_2$ is minimum. It implies that,
\begin{equation}\label{propos_proof_btc}
\norm{y-A_i\hat{x}_i}_2 < \min_{j \neq i} \norm{y-A_j\hat{x}_j}_2
\end{equation} 
Dividing both sides of the inequality with the right hand side concludes the proof. 
\end{proof}

Based on the proposition, we define a quantity namely the SIC rate as follows: We replace the testing sample $y$ with $a_i$ which is a training sample belonging to $A_i$. We also replace the $A_i$ matrix with $\overline{A_i}$ which excludes the column $a_i$. Now the training sample $a_i$ is not belonging to the dictionary $A$ anymore. Then, the quantity could be expressed as follows,

\begin{equation}\label{sic_rate}
\beta_M(a_i) \triangleq  \max_{j \neq i} \frac{\norm{a_i-\overline{A_i}\hat{x}_i}_2}{\norm{a_i-A_j\hat{x}_j}_2}
\end{equation}
Notice that we used the notation $\beta_M(a_i)$ because the expression is also a function of the parameter $M$. If we can find a threshold value which minimizes $\beta_M(a_i)$, we then state that this is the best threshold for the selected $a_i$ in the SIC sense.  One could compute the value of $\beta_M(a_i)$ using \textbf{Algorithm} \ref{alg:SIC}. In the algorithm, the operator $L_{M-1}(.)$ selects the indexes of $M$ largest correlations excluding the first one which corresponds to the index of $a_i$. 
\alglanguage{pseudocode}
\begin{algorithm}
\caption{$\beta_M(a_i)$} \label{alg:SIC}
{\normalsize \textbf{INPUT:}\\}
{\normalsize Dictionary} $A \in\mathbb{R}^{B \times N}$\\
{\normalsize Any selected sample from the $i$th class } $a_i \in\mathbb{R}^{B}$\\
{\normalsize Threshold} $M \in \mathbb{N}$ \\
{\normalsize Regularization constant} $\alpha \in (0,1)$ \\
{\normalsize Initial sparse coefficients with zeros} $\hat{x} \in\mathbb{R}^{N}$\\
{\normalsize \textbf{OUTPUT:}\\}
{\normalsize $\beta_M(a_i) \in\mathbb{R}$}\\
{\normalsize \textbf{PROCEDURE:} }
\begin{algorithmic}[1]
\State $v \gets A^T a_i$
\State $\Lambda \gets L_{M-1}(v)$
\State $D \gets A(\Lambda)$
\State $\hat{x}(\Lambda) \gets (D^TD+\alpha I)^{-1}D^Ta_i$
\State $\epsilon(j) \gets \norm{a_i-A_j \hat{x}_j}_2 ~~ \forall j \in \{1,2,\ldots,C\}$
\State $\beta_M(a_i) \gets \max_{j \neq i} \epsilon(i) / \epsilon(j)$
\end{algorithmic}
\end{algorithm}
The idea of the best threshold in the SIC sense for any $a_i$ could be extended to the all samples of $A$. Let $\overline{\beta}_M$ be the quantity which is computed by averaging the $\beta_M(a_i)$ for all samples of $A$. This is shown by the following expression.
\begin{equation}\label{avg_sic_rate}
\overline{\beta}_M \triangleq \frac{1}{N}\sum_{k = 1}^{N} \beta_M(A(k))
\end{equation}
\section{Parameter Selection}
Parameter selection is quite critical for a classifier, which highly effects the classification performance. Some algorithms use cross validation in which some portion of the training data is used for testing purposes. This method may not always provide good parameter estimation. Bad estimation of the parameters highly reduces the classification accuracy. In the following parts, we will provide some procedures to estimate the parameters of BTC.   
\subsection{Estimating the Threshold Parameter}
Using the previously defined quantity $\overline{\beta}_M$, one can estimate the threshold parameter. The best estimate of $M$ could be found in the SIC sense using the following minimization expression:
\begin{equation} \label{estimate_btc_m}
\hat{M} = \arg \min_{M} \overline{\beta}_M
\end{equation} 
By varying $M$ from $1$ to $B$, we could plot the quantity $\overline{\beta}_M$, and choose the best $M$ which minimizes it. Since $\overline{\beta}_M$ is generally convex, the procedures such as binary search and steepest descent could be utilized and the best value of $M$ could be found in a few steps. However, since the parameter determination is an off-line procedure, exhaustive search could also be used.

\subsection{Selection of the Regularization Parameter}
The best value of $\alpha$ is generally problem dependent. Once it is determined for a classification application, it is fixed for all experiments. For example, for face identification one could set it to $0.01$ for all experiments while for pixel-wise hyper-spectral image classification, it can be fixed to $0.0001$. Sometimes, the optimal choice of $\alpha$ is not critical and it is set to a quite small number just to prevent the ill-conditioned matrix operation. For instance, for spatial-spectral hyper-spectral image classification, it is set to very small number such as $10^{-10}$.      


\chapter{KERNEL BASIC THRESHOLDING CLASSIFICATION}
\label{chp:kbtc}

\section{Mapping to Hyperspace}
As we stated previously, in practice, given test samples belonging to different classes may not be distinguishable or separable in the original $B$ dimensional feature space. One solution to this problem is to map the samples in the original feature space to some arbitrarily large or possibly infinite dimensional hyperspace via a mapping function $\phi$ \cite{scholkopf1998nonlinear}, that is,
\begin{equation}\label{mapping}
\phi: \mathbb{R}^{B} \longrightarrow \mathcal{F}~~by~~y \longmapsto \phi(y)
\end{equation}

\begin{figure*}
\centering
\includegraphics[width = 1.0\textwidth]{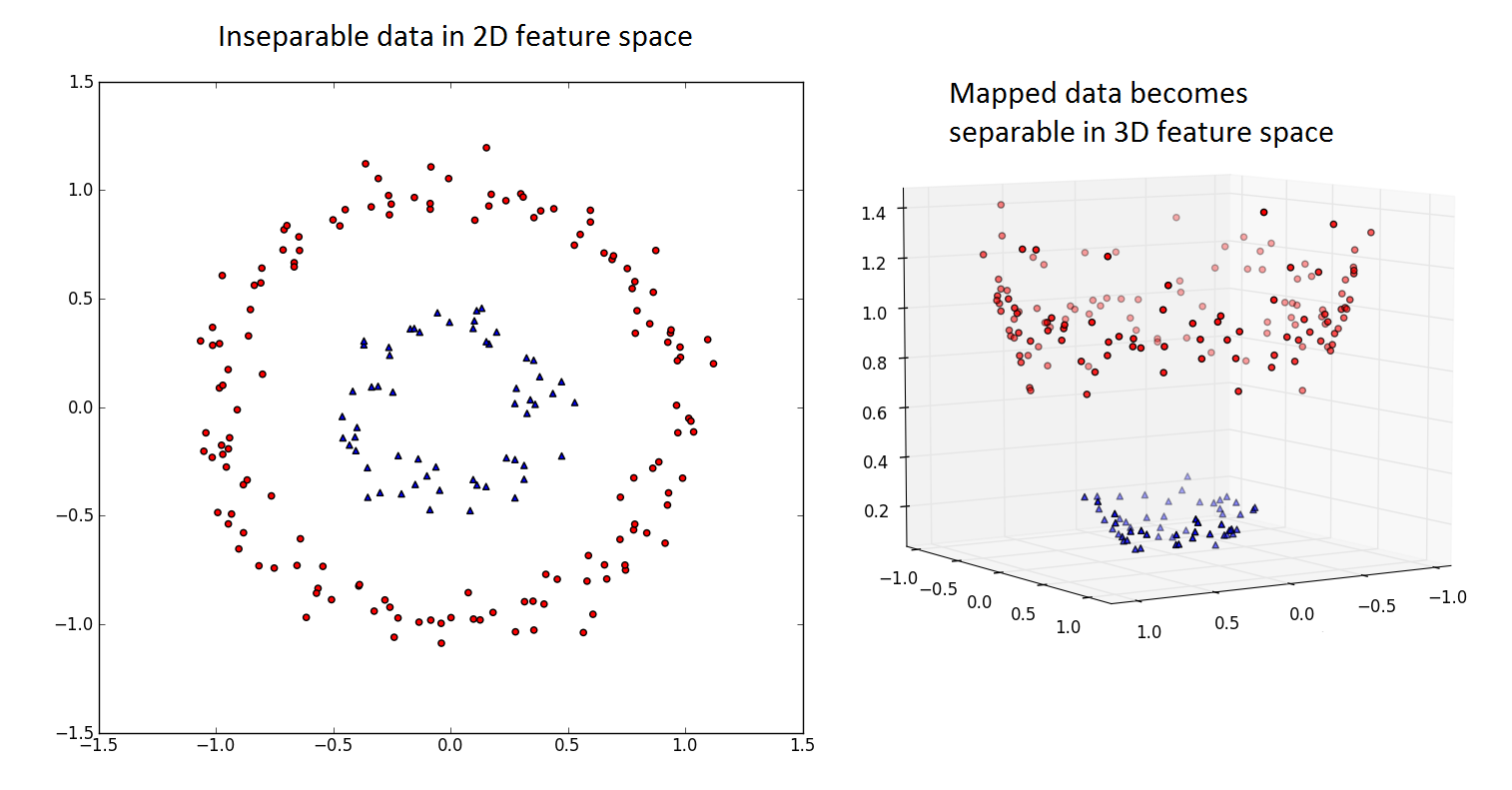}
\centering
\caption{Mapping to hyperspace}
\label{data2dto3d}
\end{figure*}
Fig. \ref{data2dto3d} shows how inseparable data in two dimensional space becomes separable in three dimensional space via a separating hyper-plane. This solution could also be applied to the BTC algorithm. In the following parts, we are going to develop kernel basic thresholding classifier by utilizing this technique.

\subsection{Mapping the Cost Function}
We know that BTC uses the cost function given below, 
\begin{equation}\label{tikhonov_lin}
J(x)= \norm{y-D x }^2_2 + \alpha\norm{x}^2_2
\end{equation}
If we map both the test sample $y$ and the selected $D$ matrix to $\mathcal{F}$ via the mapping function $\phi$, we then obtain the new cost function as follows,
\begin{equation}\label{tikhonov_nonlin}
J(x)= \norm{\phi(y)-\phi(D) x }^2_2 + \alpha\norm{x}^2_2
\end{equation}
It could also be expanded as,
\begin{equation}\label{tikhonov_nonlin_exp}
\begin{split}
J(x) &= (\phi(y)-\phi(D) x )^T(\phi(y)-\phi(D) x ) + \alpha x^Tx\\
&=\phi(y)^T\phi(y) - 2\phi(D)^T\phi(y)x+x^T\phi(D)^T\phi(D)x +\alpha x^Tx
\end{split}
\end{equation}
In order to find $x$ which minimizes $J(x)$ in $\mathcal{F}$, we take the gradient of $J(x)$ with respect to $x$ and set it to zero, that is,
\begin{equation}\label{tikhonov_nonlin_grad}
\begin{split}
\nabla J(x) &= 0\\
 - 2\phi(D)^T\phi(y) + 2\phi(D)^T\phi(D)x +2\alpha x &=0 \\
 (\phi(D)^T\phi(D) +\alpha) x &= \phi(D)^T\phi(y) \\
  x &= (\phi(D)^T\phi(D) +\alpha I)^{-1} \phi(D)^T\phi(y)
\end{split}
\end{equation}
The final equation which estimates the sparse code in $\mathcal{F}$ in terms of mapped elements becomes as follows,
\begin{equation}\label{tikhonov_nonlin_final}
  \hat{x}(\Lambda) = (\phi(D)^T\phi(D) +\alpha I)^{-1} \phi(D)^T\phi(y)
\end{equation}

\subsection{Kernel Trick}
Since we do not know the mapping function $\phi$, we can not directly calculate $\hat{x}$ using (\ref{tikhonov_nonlin_final}). However, notice that the expression contains the inner products $\langle \phi(D(i)), \phi(D(j))  \rangle$ $\forall i,j \in \{1,2,\ldots,M\}$ and $\langle \phi(D(i)), y \rangle$ $\forall i \in \{1,2,\ldots,M\}$. This enables us to use the kernel trick \cite{boser1992training} by which we can calculate the inner products of two vectors in $\mathcal{F}$ implicitly via a kernel function, that is, $K(x,y)=\langle\phi(x),\phi(y)\rangle$. Then, the kernelized version of (\ref{tikhonov_nonlin_final}) becomes as follows,
\begin{equation}\label{tikhonov_nonlin_kernel}
  \hat{x}(\Lambda) = (K(D,D) +\alpha I)^{-1} K(D,y)
\end{equation} 
where $K(D,D)$ is an $M \times M$ Gram matrix such that the entry corresponding to the $i$th row and $j$th column is equivalent to $K(D(i),D(j))=\langle \phi(D(i)), \phi(D(j))  \rangle$ and $K(D,y)$ is an $M \times 1$ vector whose entries are $K(D(i),y) = \langle \phi(D(i)), y \rangle ~~\forall i \in \{1,2,\ldots,M\}$.
\subsection{Mapping the Residuals}
After finding the sparse code estimation in $\mathcal{F}$, we also need to determine the distances or residuals in $\mathcal{F}$ between the test sample and the best representation of it for all classes, that is, $\epsilon(j) = \norm{\phi (y)-\phi(A_j) \hat{x}_j}_2 ~~ \forall j \in \{1,2,\ldots,C\}$. Since we can not calculate $\phi (y)$ and $\phi(A_j)$ directly, we are required to expand the expression as follows,

\begin{equation}\label{residual_mapped}
\begin{split}
\epsilon(j) &= \norm{\phi (y)-\phi(A_j) \hat{x}_j}_2\\
&= \sqrt{(\phi (y)-\phi(A_j) \hat{x}_j)^T(\phi (y)-\phi(A_j) \hat{x}_j)}\\ 
&= \sqrt{\phi(y)^T\phi(y) - 2\hat{x}_j^T \phi(A_j)^T\phi(y)  + \hat{x}_j^T\phi(A_j)^T\phi(A_j)\hat{x}_j} \\
&= \sqrt{K(y,y)-2\hat{x}_j^T K(A_j,y)+\hat{x}_j^TK(A_j,A_j)\hat{x}_j}
\end{split}
\end{equation} 
\section{Kernel Basic Thresholding Classifier}
After obtaining the required kernelized expressions, we can easily construct the KBTC algorithm. 

\alglanguage{pseudocode}
\begin{algorithm}
\caption{KBTC} \label{alg:KBTC}
{\normalsize \textbf{INPUT:}\\}
{\normalsize Dictionary} $A \in\mathbb{R}^{B \times N}$\\
{\normalsize Test sample} $y \in\mathbb{R}^{B}$\\
{\normalsize Threshold } $M \in \mathbb{N}$ \\
{\normalsize Regularization constant} $\alpha \in (0,1)$ \\
{\normalsize Initial sparse coefficients with zeros} $\hat{x} \in\mathbb{R}^{N}$\\
{\normalsize \textbf{OUTPUT:}\\}
{\normalsize Class of }$y$\\
{\normalsize Residual vector} $\epsilon \in\mathbb{R}^{C}$\\
{\normalsize \textbf{PROCEDURE:} }
\begin{algorithmic}[1]
\State $v \gets K(A,y) $
\State $\Lambda \gets L_{M}(v)$
\State $D \gets A(\Lambda)$
\State $\hat{x}(\Lambda) \gets (K(D,D)+\alpha I)^{-1}K(D,y)$
\State $\epsilon(j) \gets \sqrt{K(y,y)-2\hat{x}_j^T K(A_j,y)+\hat{x}_j^TK(A_j,A_j)\hat{x}_j} ~~ \forall j \in \{1,2,\ldots,C\}$
\State $class(y) \gets \arg \min_j \epsilon(j)$
\end{algorithmic}
\end{algorithm}

The implementation of the KBTC algorithm is presented in \textbf{Algorithm} \ref{alg:KBTC}. The KBTC technique performs the following steps:

\begin{itemize}
\item In the first step, KBTC finds the non-linear correlations between the test sample $y$ and the samples of all training set $A$  in $\mathcal{F}$ using the kernel function $K(.,.)$.
\item In the second step, it chooses the index set of largest $M$ correlations using the operator $L_M(.)$.
\item Then, it forms the sub matrix $D$ by using the indexes in the set $\Lambda$.
\item Using the expression $(K(D,D)+\alpha I)^{-1}K(D,y)$, it estimates the sparse code indexed with $\Lambda$.
\item Finally, it calculates the residuals for all classes, and predicts the class of $y$ based on the minimum residual.      
\end{itemize}
Note that computing $K(D,D)$ whenever a new sample is needed to be classified, is infeasible. Instead, we recommend computing $K(A,A)$ off-line and extract the required matrix as $K(D,D) \gets K(A(\Lambda),A(\Lambda))$. This way of computing $K(D,D)$ highly reduces the computational cost.
There are several kernels used in the literature such as the polynomial kernel $K(x,y)=(1+x^Ty)^d$ and the radial basis function (RBF) kernel $K(x,y)=exp(-\gamma \norm{x-y}^2_2)$. In this dissertation, we use the RBF kernel which is more commonly preferred. There are three parameters in the KBTC algorithm namely the regularization constant $\alpha$, the threshold $M$, and the $\gamma$ parameter of the kernel function. The choice of $\alpha$ is an easy one because it is used just to prevent the ill-conditioned matrix inversion. Typically, we set it to very small number such as $10^{-9}$, $10^{-10}$, etc. The most critical parameter is $\gamma$ which determines the hyperspace $\mathcal{F}$ in which the classes of a dictionary should be separate enough such that a test sample could be classified correctly. SVM technique uses cross validation to determine the $\gamma$ parameter. However, this method may not provide the best $\gamma$ for all classifiers. In the next part, we will introduce sufficient identification condition for KBTC and based on it, we will develop some procedures to determine the best values of $\gamma$ and $M$ parameters.   
\section{Sufficient Identification Condition for KBTC}

\begin{prop}
A sufficient condition for KBTC to identify a test sample $y$ belonging to the $i$th class of the dictionary $A$ is that
\begin{equation}\label{propos}
\max_{j \neq i} \frac{ \sqrt{K(y,y)-2\hat{x}_i^T K(A_i,y)+\hat{x}_i^TK(A_i,A_i)\hat{x}_i}}
{\sqrt{K(y,y)-2\hat{x}_j^T K(A_j,y)+\hat{x}_j^TK(A_j,A_j)\hat{x}_j}} < 1
\end{equation} 
where $\hat{x}$ is the sparse code estimated via KBTC, $\hat{x_i}$ and $\hat{x_j}$ are the $i$th and $j$th class portions of $\hat{x}$, respectively, and $K(.,.)$ is the kernel function.
\end{prop}

\begin{proof}
A test sample $y$ belonging to the $i$th class can be successfully identified if and only if the residual $\norm{\phi(y)-\phi(A_i)\hat{x}_i}_2$ is minimum. It implies that,
\begin{equation}\label{propos_proof}
\norm{\phi(y)-\phi(A_i)\hat{x}_i}_2 < \min_{j \neq i} \norm{\phi(y)-\phi(A_j)\hat{x}_j}_2
\end{equation} 
Expanding both sides of the inequality we obtain that,

\begin{equation}\label{propos_proof2}
\begin{split}
&\sqrt{\phi(y)^T\phi(y) - 2\hat{x}_i^T \phi(A_i)^T\phi(y)  + \hat{x}_i^T\phi(A_i)^T\phi(A_i)\hat{x}_i} \\
&< \min_{j \neq i} \sqrt{\phi(y)^T\phi(y) - 2\hat{x}_j^T \phi(A_j)^T\phi(y)  + \hat{x}_j^T\phi(A_j)^T\phi(A_j)\hat{x}_j} \\\\
&\sqrt{K(y,y)-2\hat{x}_i^T K(A_i,y)+\hat{x}_i^TK(A_i,A_i)\hat{x}_i} \\
&< \min_{j \neq i} \sqrt{K(y,y)-2\hat{x}_j^T K(A_j,y)+\hat{x}_j^TK(A_j,A_j)\hat{x}_j}\\
\end{split}
\end{equation} 

Finally, dividing both sides of the inequality with the right hand side concludes the proof. 
\end{proof}

Now, based on the proposition, we define a quantity by replacing the test sample $y$ with $a_i$ which is a training sample belonging to $A_i$ and also replacing the $A_i$ matrix with $\overline{A_i}$ which excludes the column $a_i$. It means that $a_i$ is not belonging to the dictionary $A$ anymore. Then, the quantity could be expressed as follows,

\begin{equation}\label{beta}
\beta(\gamma, M, a_i)\triangleq \max_{j \neq i} \frac{ \sqrt{K(a_i,a_i)-2\hat{x}_i^T K(\overline{A_i},a_i)+\hat{x}_i^TK(\overline{A_i},\overline{A_i})\hat{x}_i}}
{\sqrt{K(a_i,a_i)-2\hat{x}_j^T K(A_j,a_i)+\hat{x}_j^TK(A_j,A_j)\hat{x}_j}}
\end{equation}

One could easily compute $\beta(\gamma, M, a_i)$ using the \textbf{Algorithm} \ref{alg:KBTC_SIC}. It is similar to the \textbf{Algorithm} \ref{alg:KBTC}, however, this time the input is any selected training sample belonging to $i$th class and the operator $L_{M-1}(.)$ chooses the indexes of $M$ largest non-linear correlations excluding the first one which is the index of $a_i$ itself. The final output value is equivalent to the ratio of residual belonging to the $i$th class and minimum residual whose class is not equal to $i$.  

\alglanguage{pseudocode}
\begin{algorithm}
\caption{$\beta(\gamma, M, a_i)$} \label{alg:KBTC_SIC}
{\normalsize \textbf{INPUT:}\\}
{\normalsize Dictionary} $A \in\mathbb{R}^{B \times N}$\\
{\normalsize Any selected training sample from the $i$th class } $a_i \in\mathbb{R}^{B}$\\
{\normalsize Threshold } $M \in \mathbb{N}$ \\
{\normalsize Regularization constant} $\alpha \in (0,1)$ \\
{\normalsize Initial sparse coefficients with zeros} $\hat{x} \in\mathbb{R}^{N}$\\
{\normalsize \textbf{OUTPUT:}\\}
{\normalsize $\beta(\gamma, M, a_i) \in\mathbb{R}$}\\
{\normalsize \textbf{PROCEDURE:} }
\begin{algorithmic}[1]
\State $v \gets K(A,a_i) $
\State $\Lambda \gets L_{M-1}(v)$
\State $D \gets A(\Lambda)$
\State $\hat{x}(\Lambda) \gets (K(D,D)+\alpha I)^{-1}K(D,a_i)$
\State $\epsilon(j) \gets \sqrt{K(a_i,a_i)-2\hat{x}_j^T K(A_j,a_i)+\hat{x}_j^TK(A_j,A_j)\hat{x}_j} ~~ \forall j \in \{1,2,\ldots,C\}$
\State $\beta(\gamma, M, a_i) \gets \max_{j \neq i} \epsilon(i) / \epsilon(j)$
\end{algorithmic}
\end{algorithm}

As we stated previously, the quantity $\beta(\gamma, M, a_i)$ is quite important because we will utilize it in order to estimate the parameters of the KBTC algorithm.
\section{Parameter Selection}
Notice that we used the notation $\beta(\gamma, M, a_i)$ for our quantity because it is depended on $\gamma$ and the threshold $M$. In the following two parts, we develop some methodologies to estimate $\gamma$ and $M$ using the quantity $\beta(\gamma, M, a_i)$.

\subsection{Estimating the $\gamma$ Parameter}
Notice that $\beta(\gamma, M, a_i)$ is calculated only for a column of the dictionary $A$. If we repeat the procedure for all columns of $A$ and then average it, we obtain the following averaged quantity,
\begin{equation} \label{beta_avg_m}
\overline{\beta}(\gamma, M) \triangleq  \frac{1}{N}\sum\limits_{n=1}^{N} \beta(\gamma, M, A(n))
\end{equation} 
Knowing that $M$ could have values from $1$ to $B-1$, we could compute $\overline{\beta}(\gamma, M)$ for all $M$s and then average it. The final averaged quantity becomes as follows,
\begin{equation} \label{beta_avg}
\overline{\beta}(\gamma) \triangleq  \frac{1}{(B-1)}\frac{1}{N}\sum\limits_{m=1}^{B-1}\sum\limits_{n=1}^{N} \beta(\gamma, m, A(n))
\end{equation} 
Using $\overline{\beta}(\gamma) $ one can easily estimate the best $\gamma$ for KBTC by solving the following minimization problem.
\begin{equation} \label{estimate_gamma}
\hat{\gamma} = \arg \min_{\gamma} \overline{\beta}(\gamma)
\end{equation} 
As we stated previously, $\hat{\gamma}$ determines the hyperspace $\mathcal{F}$ that KBTC works in. After estimating $\gamma$, we also need to estimate $M$. In the following part, we will follow the similar procedures to estimate $M$.

\subsection{Estimating the Threshold Parameter}
Since we determined the $\gamma$, we could set it to the estimated value of it in the following function. 
\begin{equation} \label{beta_avg_m_gamma}
\overline{\beta}(\hat{\gamma}, M) \triangleq  \frac{1}{N}\sum\limits_{n=1}^{N} \beta(\hat{\gamma}, M, A(n))
\end{equation} 
Now we can easily find the best estimate of $M$ using the following minimization expression.
\begin{equation} \label{estimate_m}
\hat{M} = \arg \min_{M} \overline{\beta}(\hat{\gamma}, M)
\end{equation} 
By varying $M$ from $1$ to $B$, we could plot the quantity $\overline{\beta}(\hat{\gamma}, M)$, and choose the best $M$ which minimizes it.

\section{Alternative $5$th Step Calculation}

We could also develop some other alternative expressions for the $5$th stages of both \textbf{Algorithm} \ref{alg:KBTC_SIC} and \ref{alg:KBTC}. For this purpose, we will use the following proposition:

\begin{prop}
If the inequality $\norm{\phi(y)-\phi(A_i)x_i}_2 < \min_{j \neq i} \norm{\phi(y)-\phi(A_j)x_j}_2$ holds for a given sample $y$, then the following inequality also holds.
\begin{equation}\label{propos2}
|K(y,y)-x_i^T K(A_i,y)| < \min_{j \neq i} |K(y,y)-x_j^T K(A_j,y)|
\end{equation} 
\end{prop}

\begin{proof}
We know that both $\phi(A_i)x_i$ and $\phi(A_j)x_j$ lie in the range space of $\phi(A)$. The sparse representation assumption $\phi(y) = \phi(A)x$ implies that $\phi(y)$ also lies in $\mathcal{R}(\phi(A))$. Therefore, both the residual vectors $\phi(y)-\phi(A_i)x_i$ and $\phi(y)-\phi(A_j)x_j$ lie in $\mathcal{R}(\phi(A))$. Projecting those vectors on to the vector $\phi(y)$ via dot product with $\phi(y)/\norm{\phi(y)}_2$ does not change the direction of the inequality, that is,
\begin{equation}
\norm{\langle \frac{\phi(y)}{\norm{\phi(y)}_2},  \phi(y)-\phi(A_i)x_i \rangle}_2 
< \min_{j \neq i}  \norm{\langle \frac{\phi(y)}{\norm{\phi(y)}_2}, \phi(y)-\phi(A_j)x_j \rangle}_2
\end{equation} 
Then the expression becomes as follows,
\begin{equation}
|\phi(y)^T\phi(y) -x_i^T \phi(A_i)^T\phi(y)| < \min_{j \neq i} |\phi(y)^T\phi(y) -x_j^T \phi(A_j)^T\phi(y)|
\end{equation}
Finally, writing the expression in terms of kernel function concludes the proof.
\begin{equation}
|K(y,y)-x_i^T K(A_i,y)| < \min_{j \neq i} |K(y,y)-x_j^T K(A_j,y)|
\end{equation} 
\end{proof}
By using the proposition, we can replace the $5$th steps of \textbf{Algorithm} \ref{alg:KBTC} and \ref{alg:KBTC_SIC} with $\epsilon(j) \gets |K(y,y)-\hat{x}_j^T K(A_j,y)|  ~~ \forall j \in \{1,2,\ldots,C\}$ and with $\epsilon(j) \gets |K(a_i,a_i)-\hat{x}_j^T K(A_j,a_i)|  ~~ \forall j \in \{1,2,\ldots,C\}$, respectively.

\chapter{FACE RECOGNITION via BTC}
\label{chp:face}

\section{Introduction}

Face recognition is obviously one of the most widely investigated subjects in computer vision and pattern recognition. Law enforcement, access control systems, and several commercial applications have made face recognition an attractive research area for the last two decades. Researchers have proposed various face recognition methods to provide robust, reliable, low-cost, and high-accuracy automatic identification of frontal-view faces under various difficulties. Among those, appearance based face recognition techniques have been very popular. One of the most popular appearance methods is the Principal Component Analysis (PCA) or Eigenfaces approach \cite{turk1991eigenfaces}. In PCA, high dimensional features of both train and test samples are projected to a lower dimensional subspace. Test images are then classified using Nearest Neighbor (NN) classifier in the new feature space. Another popular subspace method is the Linear Discriminant Analysis (LDA), also known as Fisher's LDA which extracts the lower dimensional features by using both within-class and between-class information \cite{belhumeur1997eigenfaces}. Afterwards, an NN classification is performed in the identification part. All those approaches and variants focus on the feature extraction and dimension reduction stages.  

Recently, Wright \emph{et al.} \cite{wright2009robust} have proposed a robust algorithm, Sparse Representation-based Classification (SRC), for face recognition which focuses on the classification stage rather than the feature extraction. They exploit the fact that the identity information of a person is sparse among all identities in a given face database. They use a mathematical model in which any test image lying in the span of a class of a given dataset can be represented by a linear combination of all train samples of the same set. This representation is then used for classification. They showed the robustness of the algorithm under both conventional (PCA, LDA, etc.) and unconventional (Down-sampled and Random) features. In the sparse information recovery part, they use $l_{1}$-minimization which requires solving a convex optimization problem. Unfortunately, $l_{1}$-minimization has a high computational complexity especially for large scale applications.

Recently, Yang \emph{et al.} \cite{yang2010fast} have investigated the performances of several popular and fast $l_{1}$-minimizers. As stated in \cite{yang2010fast}, one of the fastest algorithm is the homotopy method which was originally proposed in \cite{osborne2000new}. When we consider real time applications with large scale datasets, even homotopy method converges very slowly.

An interesting alternative to $l_{1}$-minimizers has been proposed in \cite{shi2010rapid} for face recognition problem. The authors actually propose a hash matrix in the feature extraction part, then they use either $l_{1}$-minimization or Orthogonal Matching Pursuit (OMP) for the sparse information recovery. They recommend OMP with hashing which is extremely fast. It is indeed quite fast, however, in case of noisy sparse signal recovery, its performance reduces dramatically. Therefore, it is not recommended for face recognition especially when the train samples have illumination, expression etc. variations. Another alternative has been proposed in \cite{shi2011face}. Instead of $l_{1}$-minimization, they use $l_{2}$-minimization which could be considered fast, however, it is very sensitive to alignment variations, therefore, is not so robust. 

In this chapter, we propose the BTC algorithm for face identification problem which is capable of identifying frontal-view test samples extremely rapidly and performing high recognition rates. By exploiting rapid recognition capability, we propose a fusion scheme in which individual BTC classifiers are combined to produce better identification results especially when very small number of features is used. Finally, we propose an efficient validation technique to reject invalid test samples. Numerical results show that BTC is a tempting alternative to greedy and $l_1$-minimization-based algorithms \cite{toksoz2015btc}.

\section{Feature Extraction}
                     
In face recognition, the goal is to correctly identify the class label of any given test image using a dictionary containing labeled training samples. In this context, gray scale face samples with size of $w\times h$ are embedded into $m$ dimensional vectors where $m=wh$. Let $a_{i,j}\in \mathbb{R}^m$ with $\norm{a_{i,j}}_2=1$ be the vector containing the pixels of the $j$th sample of class $i$ and let $A_i = [a_{i,1}~a_{i,2}~\ldots ~a_{i,N_i}]\in \mathbb{R}^{m\times N_i}$ represent the matrix containing the training samples of the $ith$ class with $N_i$ many samples and $C$ denote the number of classes. Then, the final face dictionary $A$ could be constructed in such a way that $A=[A_1~A_2~\ldots~A_C]\in \mathbb{R}^{m\times N}$ where $N=\sum_{i=1}^{C}N_i$.

When we consider a face image at a typical dimension, $100 \times 100$, the final vector storing the pixels of that image will have the size of $10^4$. The computational complexity of processing with dictionaries having this number of rows is extremely high. Also most of the data in such vectors are redundant and do not carry discriminative information. Therefore, a projection $R\in \mathbb{R}^{B \times m}$, where $B\ll m$, is required to remedy aforementioned problems. After finding such a transformation, one can calculate the lower dimensional representation of the dictionary, $\Phi\in\mathbb{R}^{B\times N}$, where $\Phi = RA$. In the following part, we will briefly describe the random projections technique.

\section{Random Projections}
According to Johnson-Lindenstrauss Lemma \cite{johnson1984extensions}, any $m$ dimensional set in Euclidean space could be embedded into $O(\log m/\epsilon^2)$ dimensional Euclidean space preserving the distances between any pair of points with small distortions which are not larger than a factor of $(1+\epsilon)$, where $0<\epsilon<1$. Random projections could be considered one of the such embeddings which project the original data spherically onto a random $B$-dimensional hyperplane ($B\ll m$). There are various advantages of the use of random projections: 1-) It is a very simple and computationally efficient technique. 2-) It preserves the distances between any pair of points with small distortions. 3-) It is data independent unlike PCA or LDA. 4-) It provides classifier output variability by which we can combine the outputs of several classifiers having different random projectors.

Random projection matrices could be constructed in several ways. For instance, each entry of the matrix is selected from the zero mean and i.i.d. random variables having Normal distribution. Then, each row of the matrix is normalized to unit length. This kind of projection was successfully applied in face recognition in \cite{wright2009robust}. Another way is to select each entry from i.i.d. random variables from the Bernoulli distribution. Both ways have similar mathematical properties. Storing Bernoulli random variables (integers) obviously requires less memory than storing Normal random variables (double size numbers). Achlioptas \cite{achlioptas2001database} proposed sparse random projections in which entries are selected from $\{+\sqrt{3}, 0, -\sqrt{3}\}$ with probabilities $\{\frac{1}{6}, \frac{2}{3}, \frac{1}{6}\}$. A more generic version of this technique was proposed in \cite{li2006very}. It is called very sparse random projection in which entries are selected from $\{+\sqrt{S}, 0, -\sqrt{S}\}$ with probabilities $\{\frac{1}{2S}, 1-\frac{1}{S}, \frac{1}{2S}\}$ where $S$ is a positive integer. Note that the random projection matrices from the Bernoulli distribution and variants are normalized by $\sqrt{m}$ to produce unit length row. In this thesis, we use the last technique in which we can adjust the sparsity of the input features by varying the $S$ parameter. Also, if we are working on a device having limited storage capacity, aggressive choice of $S$ (e.g. 100) will highly reduce the size of the random projection matrix (We can only store the locations of $\{+\sqrt{S}\}$s and $\{-\sqrt{S}\}$s) without losing significant discriminative features. 

\section{Recognition with a Single Classifier}
Assume that after random projection we obtained the lower dimensional representation of the dictionary as $\Phi \in\mathbb{R}^{B\times N}$ and the test face image as $y \in\mathbb{R}^{B}$. Then, one can use the following expression in order to identify $y$:

\begin{equation}\label{btc_face}
 Id(y) \gets BTC(\Phi, y, M, \alpha)
\end{equation}

The details of the algorithm could be found in Chapter \ref{chp:btc}. Using the described quantity (\ref{avg_sic_rate}), the threshold value $M$ can be estimated for any dataset. For instance, we computed $\overline{\beta}_M$ for Extended Yale-B dataset for 504 and 120 features. In Fig. \ref{betaYale}, we plotted $\overline{\beta}_M$ by simply ranging $M$ over the intervals (1, 120) and (1, 504). By observing the plotted curves, the optimum value of $M$ in SIC sense could be determined approximately by picking the value which makes $\overline{\beta}_M$ minimum. Note that the detailed descriptions of the Extended Yale-B dataset is given in Section \ref{chp5:exp}.

\begin{figure}[!t]
\centering
\includegraphics[width = 0.8\textwidth]{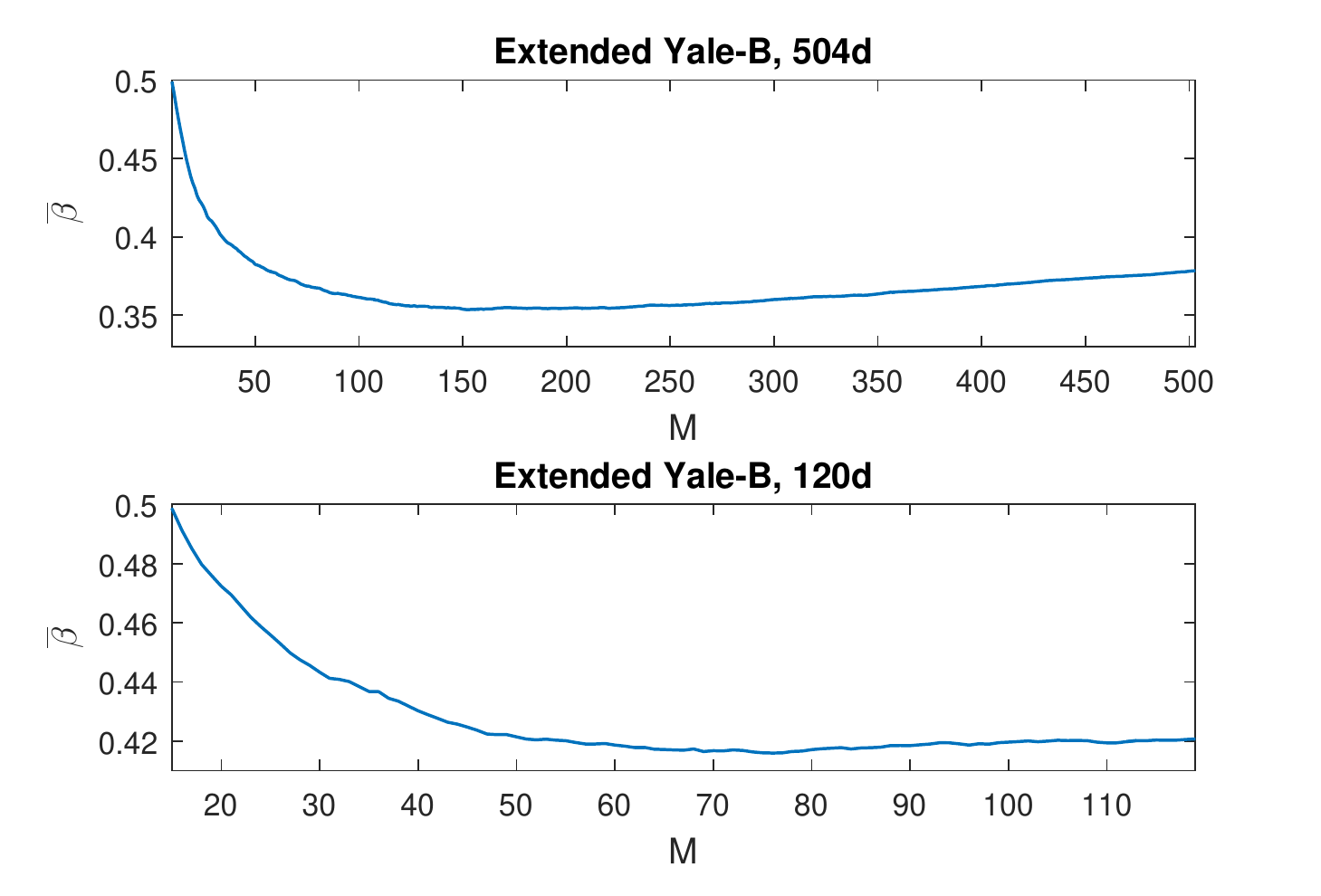}
\caption{ $\overline{\beta}_M$ vs threshold $M$ on Extended Yale-B dataset for dimensions 120 and 512}
\label{betaYale}
\end{figure} 

In generic classification problems, it is always desired to reduce the dimensionality of the input features. Here, the goal is not only to increase the discriminative properties of the features, but also to reduce the memory usage and computational burden. It is valid for face recognition as well. Consider the portable devices like cell phones and near feature intelligent robots for which energy requirement is always problematic. Therefore, the designers have to keep the processing capability of such intelligent devices in low levels. In this context, we need to construct reliable algorithms which are able to provide highly accurate results using only a small number of features. On the other hand, keeping the number of features small does not always provide good results especially for a single classifier. Immediate remedy of this problem is to combine the outputs of several classifiers. In this sense, we need high accuracy individual classifiers having enough level of output diversity. Those requirements are exactly met by individual BTC classifiers with different random projections. In the following section, we will develop an efficient fusion scheme by combining the outputs of individual BTC classifiers.  

\section{Classifier Ensembles}
In this part, we present a classifier fusion scheme to increase the classification accuracy further. In Fig. \ref{classifierEnsemble}, we see how we can combine the outputs of $n$ individual BTCs using parallel topology. This architecture is commonly used in the literature \cite{wozniak2014survey}, \cite{krawczyk2014clustering}, \cite{krawczyk2015usefulness}. Assume that the random projector for the $i$th BTC is $R^i \in\mathbb{R}^{B \times m}$ and the corresponding residual vector containing the errors for each class is $\epsilon^i \in\mathbb{R}^{C}$. We can then combine the individual residuals to obtain one final residual vector. Here, the sample mean of the residuals, as indicated in \cite{cyganek2012one} and successfully applied in \cite{yang2007feature}, could be considered as a good combiner, that is, $\epsilon = \frac{1}{n}\sum_{i=1}^{n}\epsilon^i$. The technique used in here combining the intermediate function output instead of direct classification results is called support function fusion \cite{wozniak2014survey}. After obtaining the overall residual vector, final classification is done by selecting the class which has the minimum residual. \textbf{Algorithm} \ref{alg:fusion} describes the practical implementation of BTC-n. Note that the fusion technique given here is similar to the method used in \cite{yang2007feature}. Here, we use the outputs of individual BTCs instead of the outputs of $l_1$ minimizers.
\begin{figure*}
\centering
\includegraphics[width = 1.0\textwidth]{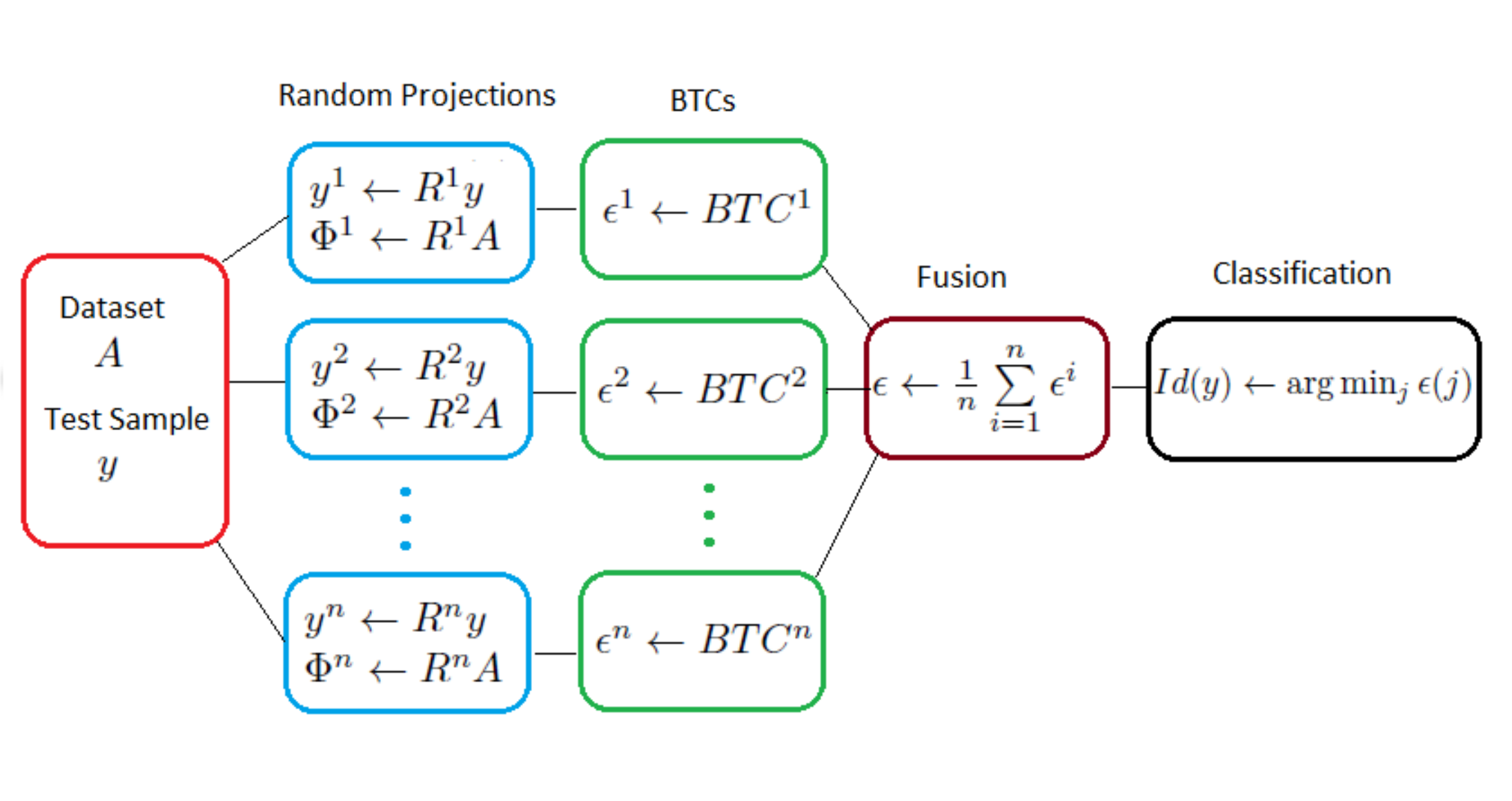}
\caption{Classifier ensembles}
\label{classifierEnsemble}
\end{figure*}
\alglanguage{pseudocode}
\begin{algorithm}
\caption{BTC-n} \label{alg:fusion}
{\normalsize \textbf{INPUT:}\\}
{\small Dictionary} $A \in\mathbb{R}^{m\times N}$\\
{\small Test sample} $y \in\mathbb{R}^{m}$\\
{\small Threshold parameter $M \in \mathbb{N}$ }\\
{\small Regularization parameter $\alpha \in (0,1)$ }\\
{\small Random Projectors} $R^i \in\mathbb{R}^{B \times m}~i \in \{1,2,...,n\}$\\
{\normalsize \textbf{OUTPUT:}\\}
{\small Identity of $y$}\\
{\normalsize \textbf{PROCEDURE:} }
\begin{algorithmic}[1]
\State $\epsilon(j) \gets 0 ~~ \forall j \in \{1,2,...,C\}$
\For{$i\gets 1, n$}
\State $\Phi^i \gets R^iA$
\State $y^i \gets R^iy$ 
\State $\epsilon^i \gets BTC(\Phi^i, y^i, M, \alpha)$
\State $\epsilon \gets \epsilon + \epsilon^i$
\EndFor
\State $\epsilon \gets \epsilon / n$
\State $Id(y) \gets \arg \min_j \epsilon(j)$
\end{algorithmic}
\end{algorithm}

\section{Experimental Verification and Performance Comparison}
\label{chp5:exp}
In this section, we will compare the classification performances of SRC, OMP, COMP, and BTC as well as their ensembles, SRC-n, OMP-n, COMP-n, and BTC-n, on publicly available datasets namely Extended Yale-B, Faces 94, 95, 96, and ORL in face identification domain. We tested the performances using 30, 56, 120, and 504 features as used in \cite{wright2009robust} and \cite{yang2007feature}. Note that in the beginning of the experiments we computed the threshold values for each dataset and feature vector size according to properties of the datasets. Also note that whenever the dictionary changes, that is, some classes and samples are added or removed from it, the $M$ value for that dictionary must be recalculated. The computed $M$ values are shown in Table \ref{thresholdValues}.

\begin{table*}
\renewcommand{\arraystretch}{1}
\caption{ Computed $M$ values for each dataset and dimension}
\label{thresholdValues}
\centering
\begin{tabular}{l c c c c}
\hline
 Dimension(d)    & 30 & 56 & 120 & 504\\
\hline
 Extended Yale-B & 29 & 48 & 88 & 172\\ 
 Faces 94        & 11 & 11 & 33 & 21\\ 
 Faces 95        & 11 & 26 & 36 & 38\\  
 Faces 96        & 11 & 13 & 35 & 43\\ 
 ORL             & 6  & 6  & 6  & 6\\ \hline 
\end{tabular}
\end{table*}

\begin{figure*}
\centering
\includegraphics[width = 0.8\textwidth]{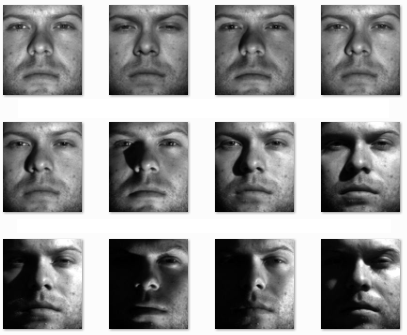}
\centering
\caption{Sample faces from Extended Yale-B dataset}
\label{sampleFaces}
\end{figure*}

\subsection{Results on Extended Yale-B}
The Extended Yale-B face dataset contains 38 classes (subjects) each having about 64 frontal face samples with resolution $192$ x $168$. The samples have perfect alignment and different illuminations per individual (Fig. \ref{sampleFaces}). As in \cite{wright2009robust} and \cite{yang2007feature}, we randomly selected half of the samples (about 32 per class) for training and remaining half for testing to make sure that the results do not depend on any special configuration of the dictionary. We used the same random instance for all algorithms. For SRC method, we chose Homotopy $l_1$ minimizer during the experiments with tolerance 1e-12. We set the sparsity level of OMP algorithm to the average number of samples per class ($\overline{N}$) for all experiments. For BTC algorithm, we used the threshold values given in Table \ref{thresholdValues} and for all experiments we set the regularization constant $\alpha$ to 0.01. For face identification, this is the value that we propose based on the experiment we provide on Extended Yale-B. As shown in Fig. \ref{fig:alpha}, the classification accuracies are maximized at 0.01 for BTC and BTC-n techniques. We repeated the experiments $50$ times with $50$ different random projectors (Bernoulli random projections) for each method to make sure that results do not depend on any special choice of the projection matrix. The classification accuracies of each method are shown in Table \ref{resultsYale}. In multi-class problems, considering only classification accuracy may not be enough to show the real performance of a classifier as indicated in \cite{galar2015drcw}. Therefore, we also added kappa ($\kappa$) statistic in Table \ref{resultsYale} which measures pairwise degree of agreement among set of raters \cite{scales1960ment}.    
\begin{table*}
\renewcommand{\arraystretch}{1}
\caption{ Recognition rates and $\kappa$ statistics on Extended Yale-B dataset}
\label{resultsYale}
\centering
\begin{tabular}{|c |c |c |c| c | c| c| c| c|}
\cline{2-9}
\multicolumn{1}{ c| }{}& \multicolumn{4}{ c |}{Accuracy} & \multicolumn{4}{ |c |}{$\kappa$ statistics} \\ 
\hline
Dimension(d) & 30 & 56 & 120 & 504 & 30& 56 & 120 & 504\\
\hline 
SRC [\%]     & 83.47 & \textbf{91.82} & \textbf{95.77} & 97.40 & 83.02 & \textbf{91.60} & \textbf{95.65} & 97.33 \\ 
OMP [\%]     & 66.59 & 76.66 & 89.91 & 96.52 & 65.69 & 76.03 & 89.64 & 96.43 \\  
COMP [\%]     & 72.23 & 86.75 & 93.24 & 97.03 & 71.48 & 86.40 & 93.05 & 96.95 \\ 
BTC [\%]    & \textbf{83.98} & 91.25 & 95.44 & \textbf{98.14} & \textbf{83.54} & 91.01 & 95.32 & \textbf{98.09} \\  \hline 

SRC-5 [\%]  & 90.87 & 94.98 & 96.87 & 97.45 & 90.62 & 94.84 & 96.79 & 97.38 \\  
OMP-5 [\%]  & 78.20 & 81.90 & 93.58 & 96.87 & 77.61 & 81.41 & 93.41 & 96.79 \\ 
COMP-5 [\%]  & 83.96 & 90.46 & 94.90 & 97.37 & 83.53 & 90.20 & 94.76 & 97.29 \\ 
BTC-5 [\%]  & \textbf{92.51} & \textbf{95.97} & \textbf{97.45} & \textbf{98.84} & \textbf{92.31} & \textbf{95.86} & \textbf{97.38} & \textbf{98.81} \\  \hline 
\end{tabular}
\end{table*}

In the single classifier case, we see that the rates of SRC and BTC are very close to each other. In case of 30 and 504 features, BTC slightly outperforms the SRC method. On the other hand, SRC slightly performs better than BTC at 56d and 120d. The performance of OMP algorithm is quite low at dimensions 30d, 56d, and 120d as expected. We also added Cholesky-based OMP (COMP) which performs better than ordinary OMP. In the case of classifier ensembles using 5 individual classifiers, we see that the accuracies highly increase. For this case, BTC-5 outperforms the other techniques at all dimensions. The ensemble technique namely \emph{E-Random} or SRC-5 proposed in \cite{yang2007feature} achieves {90.72, 94.12, 96.35, 98.26} percent at the same dimensions, respectively, using 5 classifiers. we see that BTC-5 outperforms \emph{E-Random} which is computationally quite expensive.   

\begin{figure*}
\centering
\includegraphics[width=0.8\textwidth]{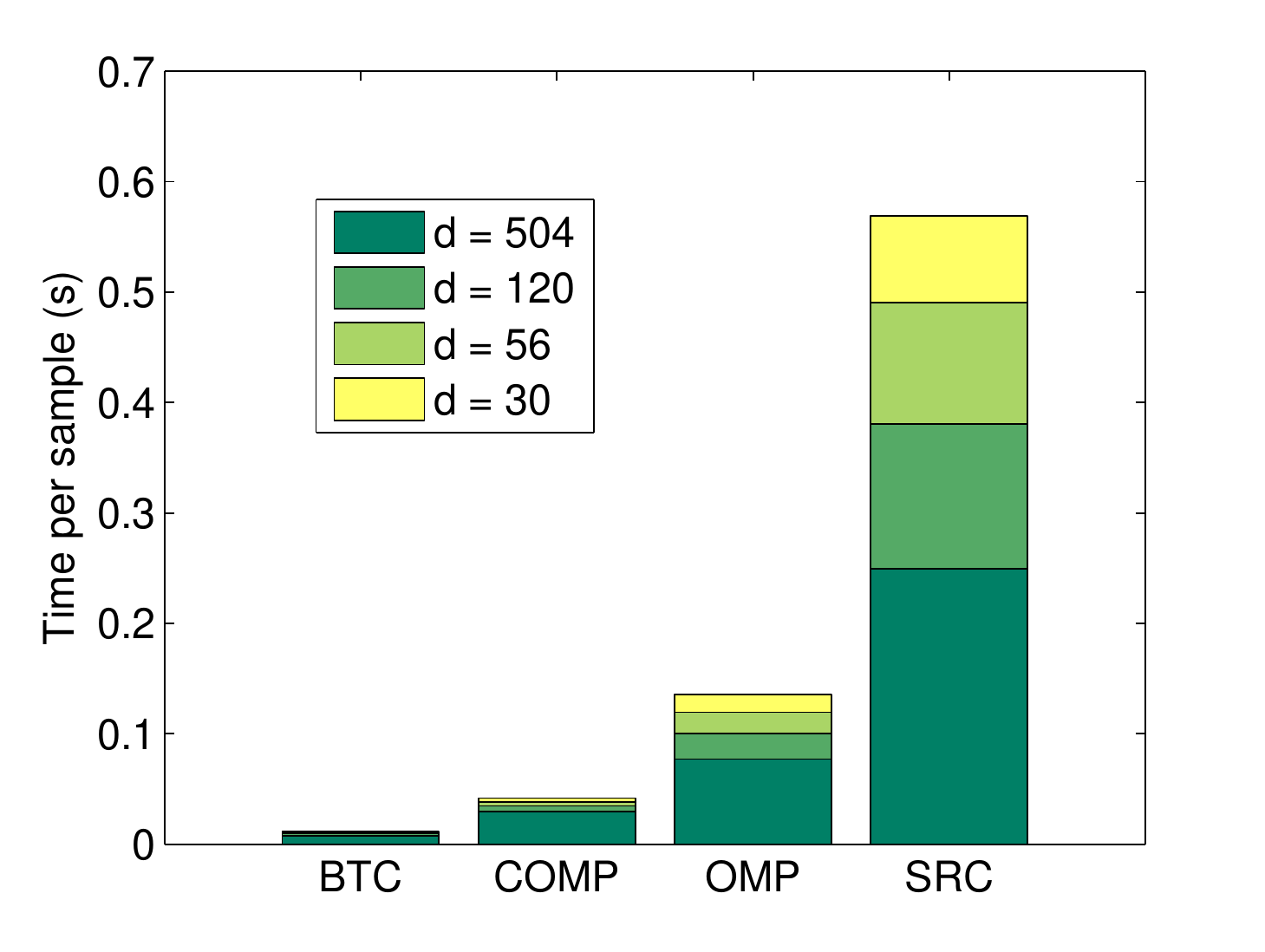}
\caption{Classification times on Extended Yale-B}
\label{fig:yaleTimes}
\end{figure*}

\begin{figure*}
\centering
\includegraphics[width=0.8\textwidth]{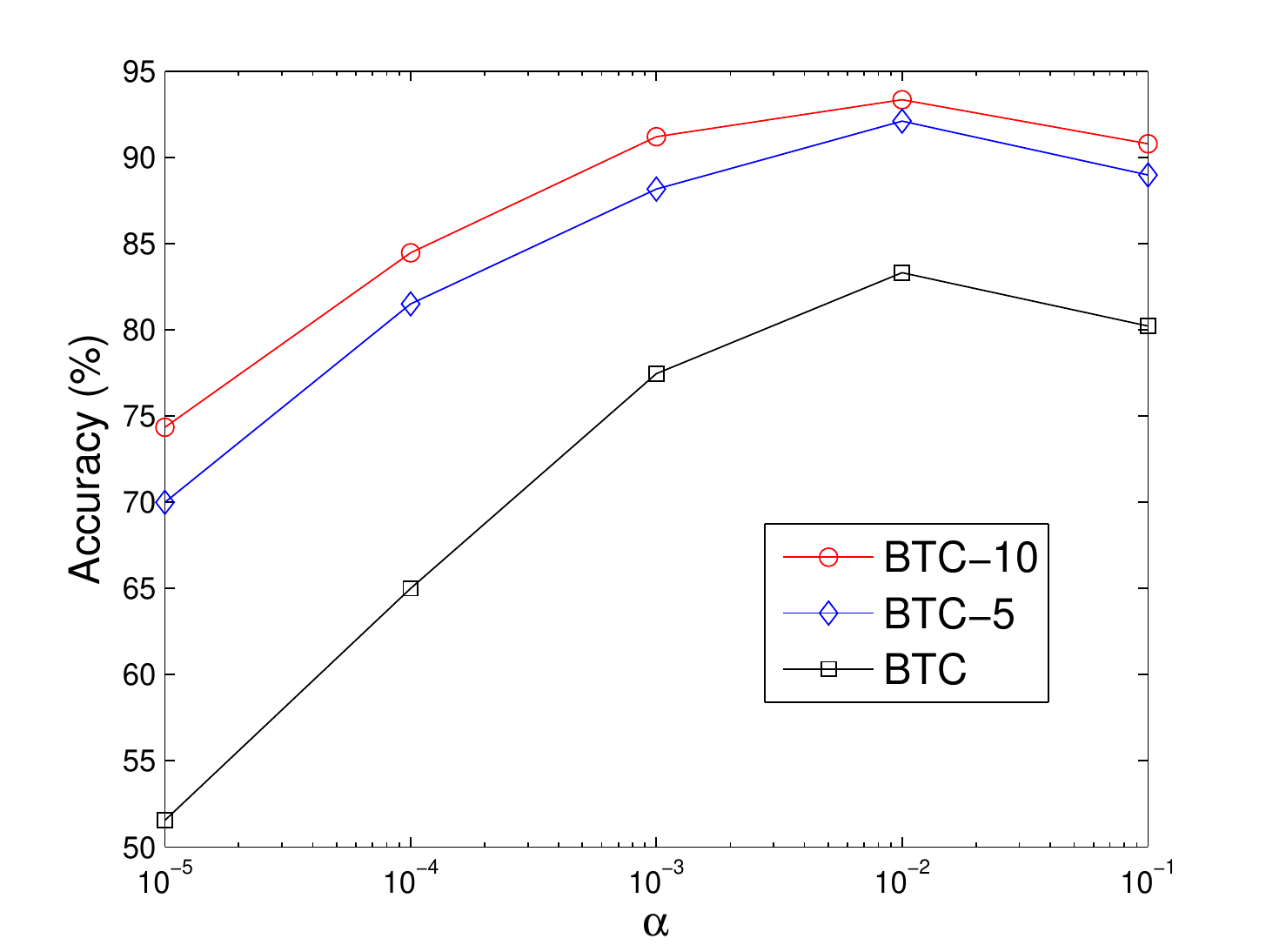}
\caption{Classification accuracies with respect to regularization parameter ($\alpha$)}
\label{fig:alpha}
\end{figure*}

\begin{figure*}
\centering
 \includegraphics[width=0.8\textwidth]{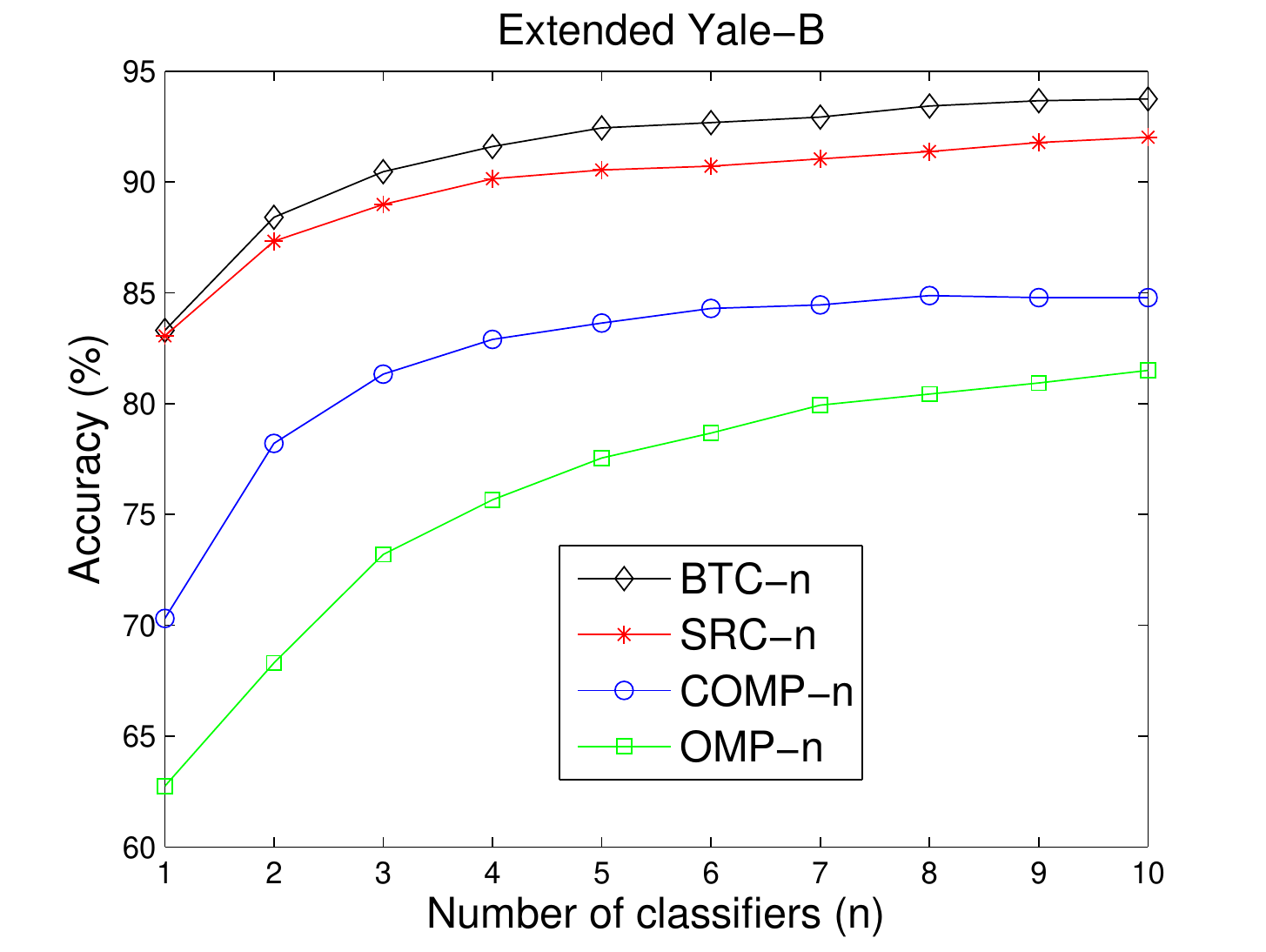}
 \caption{Classification accuracies on Extended Yale-B using ensemble techniques}
 \label{fig:yaleEnsemble}
\end{figure*}

\begin{figure*}
\centering
 \includegraphics[width=0.8\textwidth]{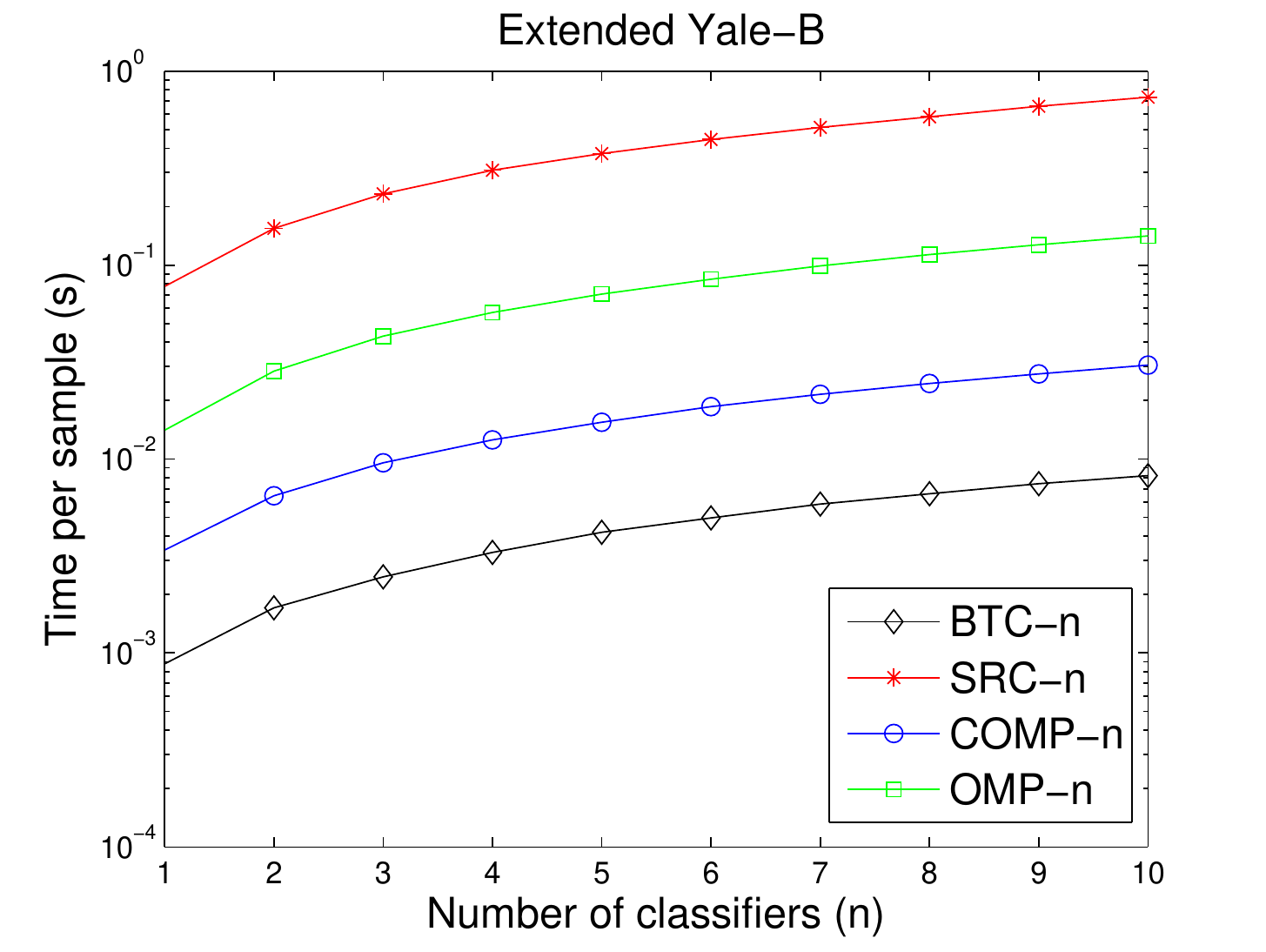}
 \caption{Classification times on Extended Yale-B  using ensemble techniques}
 \label{fig:yaleEnsembleTime}
\end{figure*}

We also compared the computation times per sample for each method in order to further investigate the feasibility of each proposal. Fig. \ref{fig:yaleTimes} shows the classification times per individual for each method at each dimension. Note that all experiments were performed in MATLAB on a workstation with dual quad-core 2.67 GHz Intel processors and 4GB of memory. As expected SRC is slow and OMP and COMP are fast. On the other hand, we see that single BTC is extremely fast as compared to the others. The speed of the proposal enables us efficiently use the ensemble technique which is superior to the single classifiers. The SRC-n approach is highly inefficient in terms of computational cost although its classification accuracy is high. We also compared the performances of the ensemble techniques with respect to the number of classifiers using only 30 features in Fig. \ref{fig:yaleEnsemble}. We observe that for all cases BTC-n is superior to the other methods. Up to 5 classifiers, the classification accuracies immediately increase. Based on this observation, combining 5 or 10 classifiers could be good choice depending on the speed of the algorithm. The corresponding computation times for the ensemble techniques are given in Fig. \ref{fig:yaleEnsembleTime}. The results show that the computation performance differences between BTC-n and the other approaches are significant.  
\subsection{Results on Faces 94, 95, 96, and ORL}
The goal of the experiments on Faces 94, 95, and 96 is to compare the classification performances under automatic object detection framework. Note that we used \emph{Viola-Jones} detector to capture the faces \cite{viola2001rapid}. Since the detector that we use is not able to perfectly capture and align the test samples, the maximum identification rates of the classifiers are decreased. Because of this reason, we focus on the performance differences rather than the maximum rates achieved. Brief description of the datasets could be seen in Table \ref{dbDescription}. The detailed description of each dataset (Faces 94, 95, 96) could be found in \cite{spacek2007collection}. Notice that this time we have more realistic scenarios which contain misalignments, head scale, expression, and illumination variations in the cropped faces. We have also chance to compare the performances under ORL which has a few (5) training samples per subject. Configurations for these experiments were similar to the previous ones. The only differences were the number of classes and samples for each dataset.

\begin{table}[!t]
\renewcommand{\arraystretch}{1}
\caption{ Description of each dataset}
\label{dbDescription}
\centering
\begin{tabular}{|c|c|c| c| c |c| c|}
\hline
Dataset & Difficulty & Resolution  & Captured & Num. of & Tr. smp.& Te. smp. \\
  &  &  & resolution & classes & per class & per class\\
\hline 
Faces 94 & Easy & $180\times 200$ & $123\times 123$ & 152 & 10 & 10\\
Faces 95 & Medium & $180\times 200$ & $75\times 75$ & 72 & 10 & 10\\
Faces 96 & Hard & $196\times 196$ & $98\times 98$ & 151 & 10 & 10\\
ORL      & Medium & $9\times 112$ & $92\times 112$ & 40 & 5 & 5\\ \hline 
\end{tabular}
\end{table}

\begin{figure*}
\centering
\includegraphics[width=0.8\textwidth]{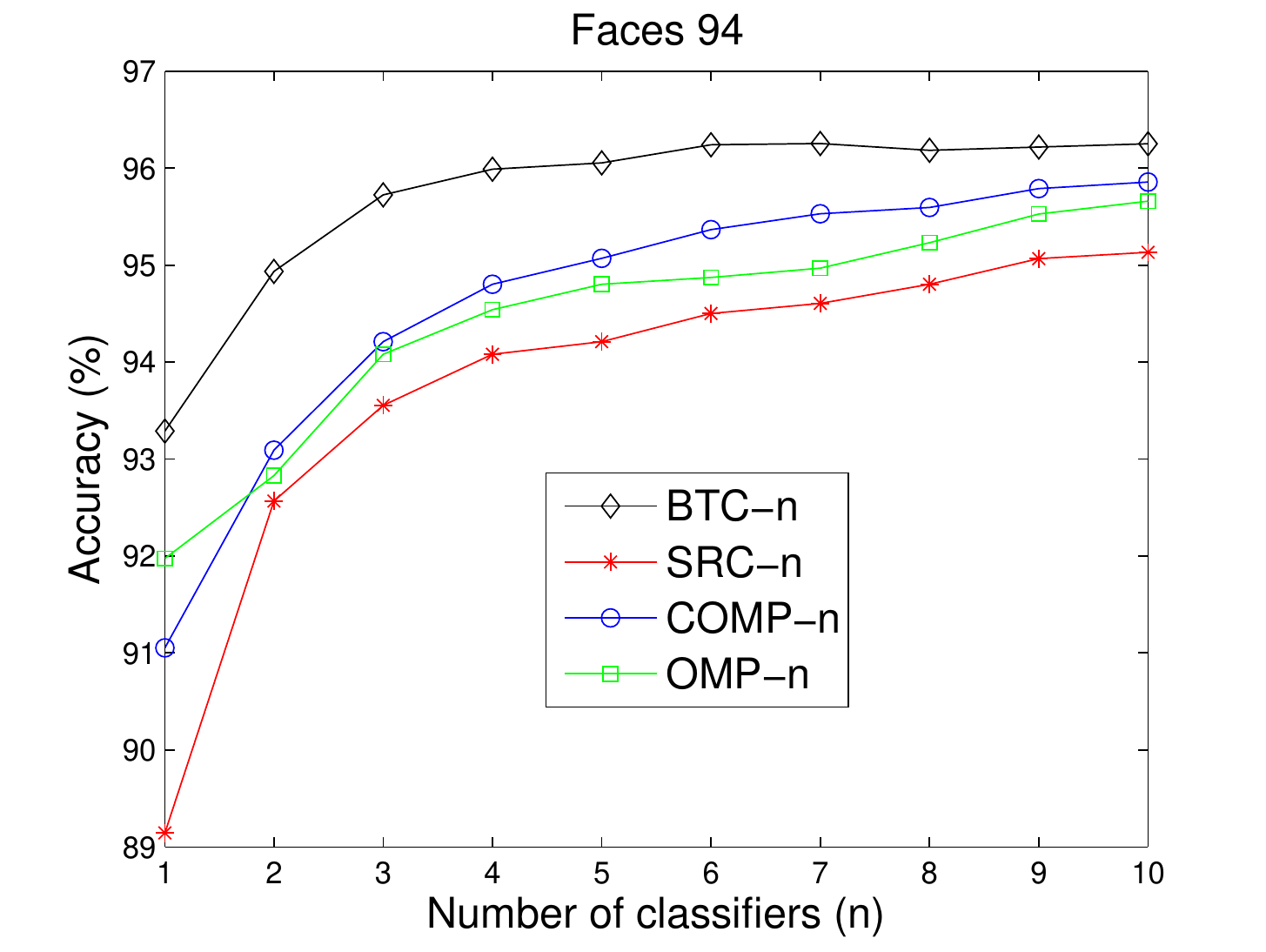}
\caption{Recognition rates on Faces 94}
\label{fig:faces94}
\end{figure*}

\begin{figure*}
\centering
\includegraphics[width=0.8\textwidth]{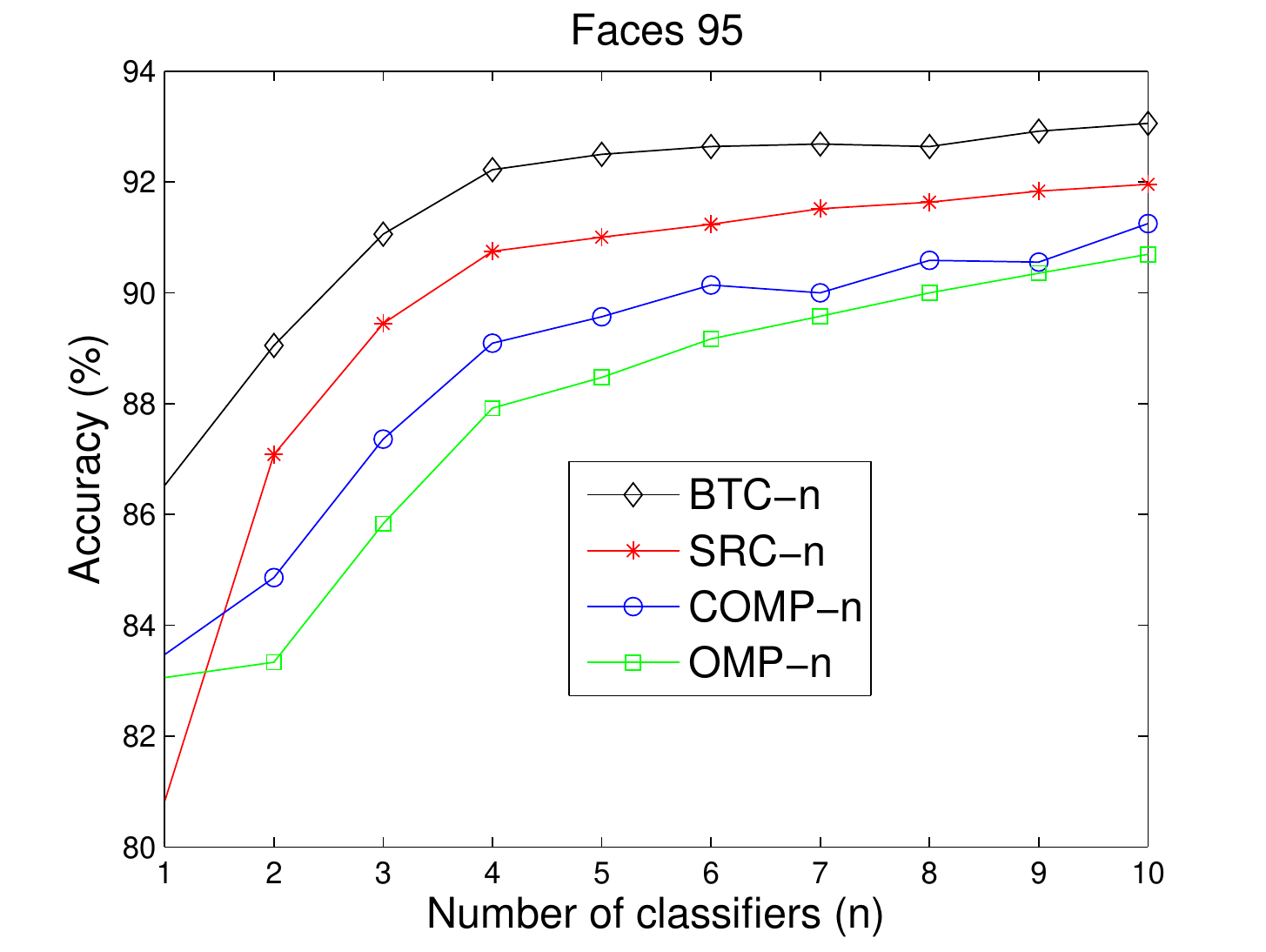}
\caption{Recognition rates on Faces 95}
\label{fig:faces95}
\end{figure*}

\begin{figure*}
\centering
 \includegraphics[width=0.8\textwidth]{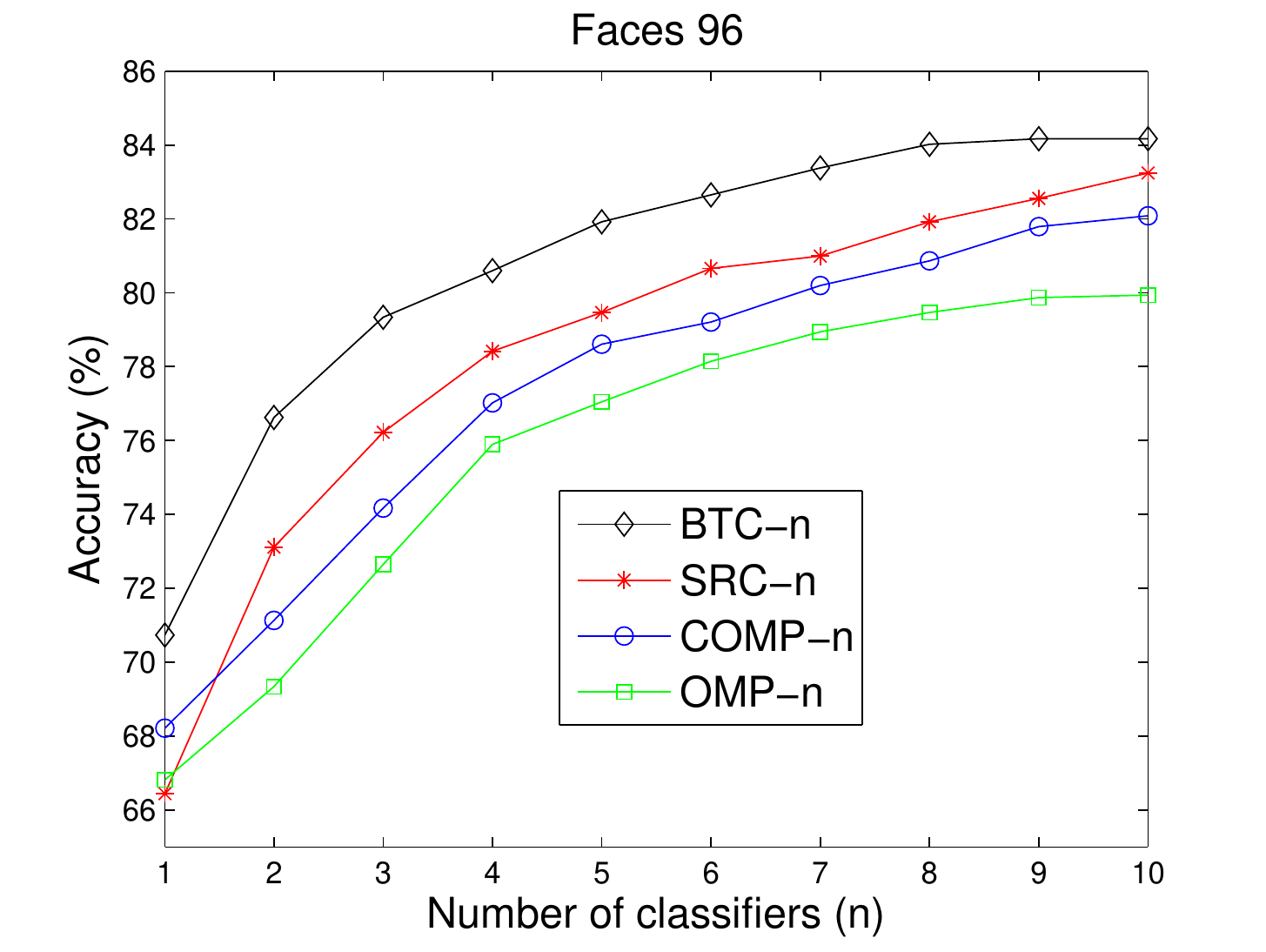}
 \caption{Recognition rates on Faces 96}
 \label{fig:faces96}
\end{figure*}

\begin{figure*}
\centering
 \includegraphics[width=0.8\textwidth]{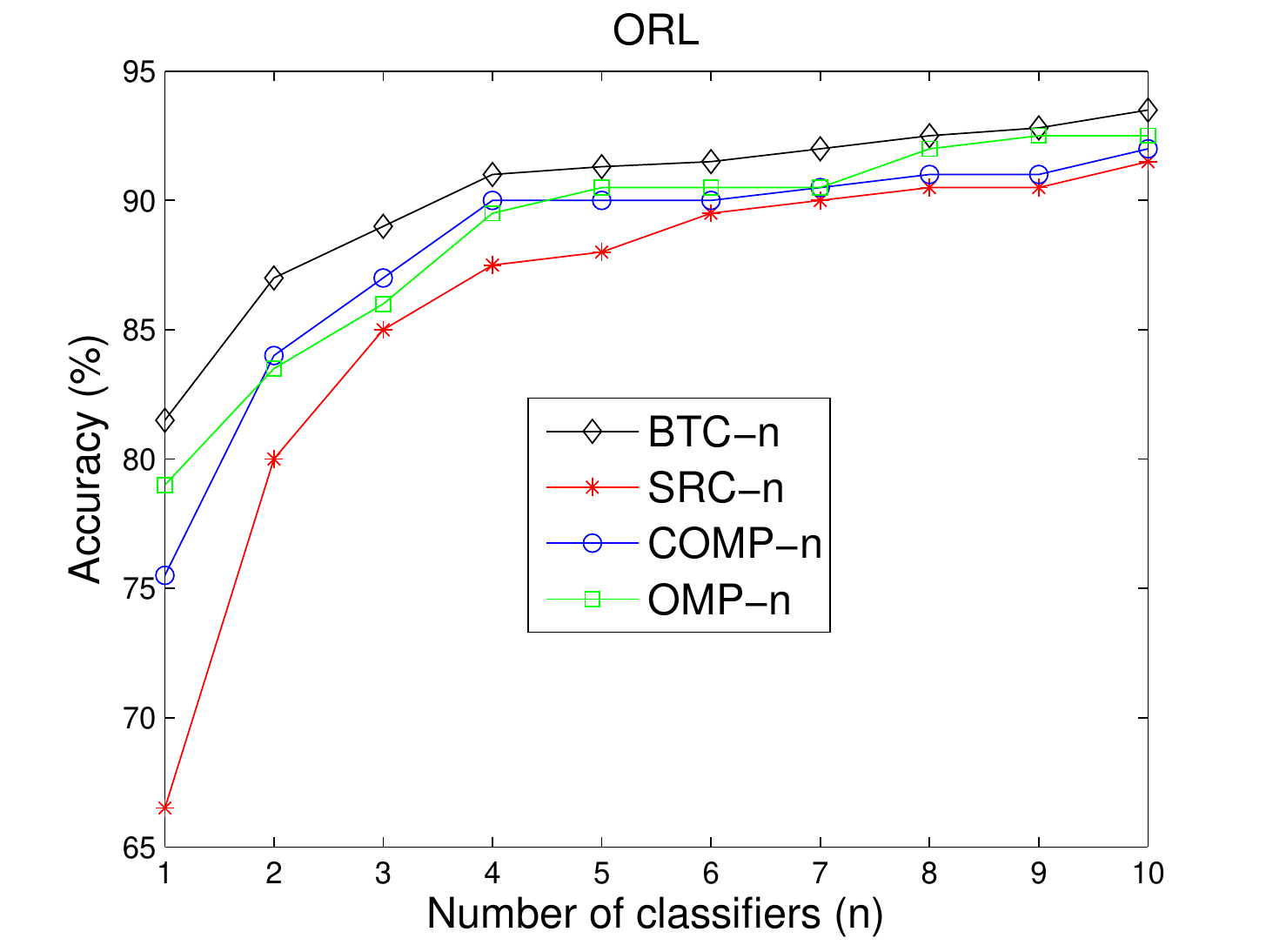}
 \caption{Recognition rates on ORL}
 \label{fig:orl}
\end{figure*}

%
%

Fig. \ref{fig:faces94} shows the recognition rates on Faces 94 dataset for SRC-n, OMP-n, COMP-n, and BTC-n methods with respect to the number of classifiers using 30 features. As we see in the figure, BTC-n outperforms the other techniques as expected. Notice that this time, OMP-n's and COMP-n's performances are acceptable as compared to the results on Extended Yale-B. This is because the illumination variations are not significant which corresponds to less-noisy sparse recovery. They even outperform the SRC-n method. However, previous results on Extended Yale-B show that those approaches are not as robust as the SRC-n and BTC-n techniques.

In Fig. \ref{fig:faces95} we see the classification performances on Faces 95 dataset. This dataset is more difficult than the previous one. Therefore, the performances of all methods slightly decreased as compared to the previous case. This time also BTC-n technique outperforms the others. We observe that SRC-n is superior to the OMP-n and COMP-n techniques because of the difficulty of the dataset. 

In Fig. \ref{fig:faces96} we see the classification accuracies on Faces 96 dataset which could be considered the most difficult one. This time the performances of the all methods decreased. The performance differences are similar to those of the previous experiment. The best results also were obtained by BTC-n. 

Fig. \ref{fig:orl} shows the classification performances on ORL dataset. As expected BTC-n achieves best results as in the previous cases. This time the performances of BTC-n, OMP-n, and COMP-n are very close. However, BTC-n slightly outperforms the OMP-n and COMP-n methods. SRC-n achieves the lowest rates. This shows that SRC-n is vulnerable to the number of samples per subject especially when the dictionary has very small number of samples per class. 
\subsection{Comparison with the Correlation Classifier}
One could wonder what happens if we directly use simple correlations instead of sparse representation. For this purpose, we designed an algorithm namely simple correlation classifier (CORR) which performs the following steps:

\begin{itemize}
\item Find the correlation vector $v$ containing the $M$ largest linear correlations between the test sample $y$ and the samples of all training set $A \in\mathbb{R}^{B \times N}$. 
\item Set the remaining $N-M$ entries in $v$ to zero.
\item Perform classification using the sum of the linear correlations within each class, that is, 
\begin{equation}\label{CORR}
class(y) = \arg \max_j \sum_iv_j(i) ~~ \forall j \in \{1,2,\ldots,C\}
\end{equation}
where $i$ represents the correlation index within a class and $j$ shows the class index.
\end{itemize}
Using this method, we repeated the same experiment previously performed on Extended Yale-B dataset. The recognition rates for different feature vector sizes and threshold values ($M$) could be seen in Fig. \ref{fig:corr}. We can observe that the best results are achieved when $M$ is set to $1$. This means that the best policy is to find the class label of the dictionary element having the highest correlation with the test sample. Note that we could have used the majority voting (MV) technique instead of the sum operation. However, the sum operation is superior to the MV technique.  
 
\begin{figure*}
\centering
 \includegraphics[width=0.8\textwidth]{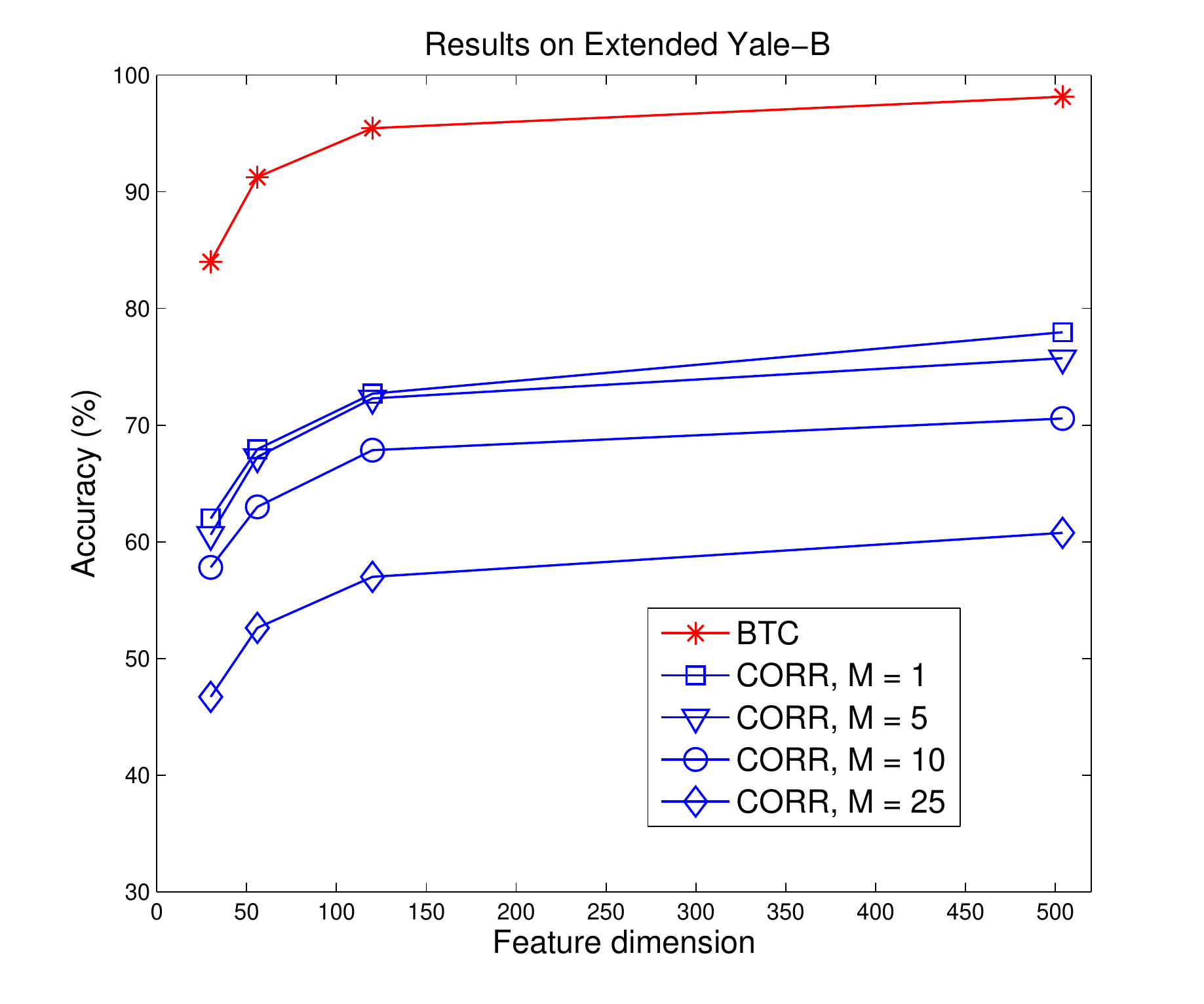}
 \caption{Comparison with correlation classifier (Extended Yale-B)}
 \label{fig:corr}
\end{figure*}

By observing the results, we can also compare the performance of this intuitive method with that of the proposed algorithm. We see that the accuracies obtained by BTC are far beyond the CORR approach. This experiment not only shows the superiority of BTC but also the power of the sparse representation.

\subsection{Rejecting Invalid Test Samples}
Evaluating a classifier considering only classification performance and computational cost is not enough in real world applications. A classifier must also be able to reject invalid test samples and correctly classify the valid ones at the same time. Up to now we only considered the cases where the test sample belongs to one of the classes in the dictionary. This time we consider a test sample which does not belong to any of the classes. Rejection could be performed by using a predefined threshold value. In this thesis, we propose a validation mechanism based on the residual vector which is produced when the test sample $y$ is applied to the BTC algorithm. As we stated previously, the residual vector contains entries for all classes in the dictionary. In this context, we define the following measure for any test sample $y$.      
\begin{equation}\label{gamma}
\gamma(y) \triangleq 1 - \frac{\epsilon(i)}{\epsilon(j)}
\end{equation}
where $i=\arg \min_k \epsilon(k)$ and $j=\arg \min_{k \neq i} \epsilon(k)$. Here, $i$ and $j$ are simply the class indexes which give the smallest and the second smallest residuals, respectively. Notice that the ratio in (\ref{gamma}) is the natural result of $\overline{\beta}_M$ that we mentioned previously. Assume that we apply $y$ to the BTC algorithm and obtain $\gamma(y)$ being close to 1. Then, we say that with high probability it belongs to one of the classes. If the result is close to 0, then we say that it probably does not belong to any of the classes. Let us define $\tau \in (0,1)$ as the rejection threshold. If the following condition is not satisfied, then the test sample is rejected.
\begin{equation}\label{rejectionTau}
\gamma(y) \geq \tau
\end{equation}

\begin{figure*}
\centering
\includegraphics[width=0.8\textwidth]{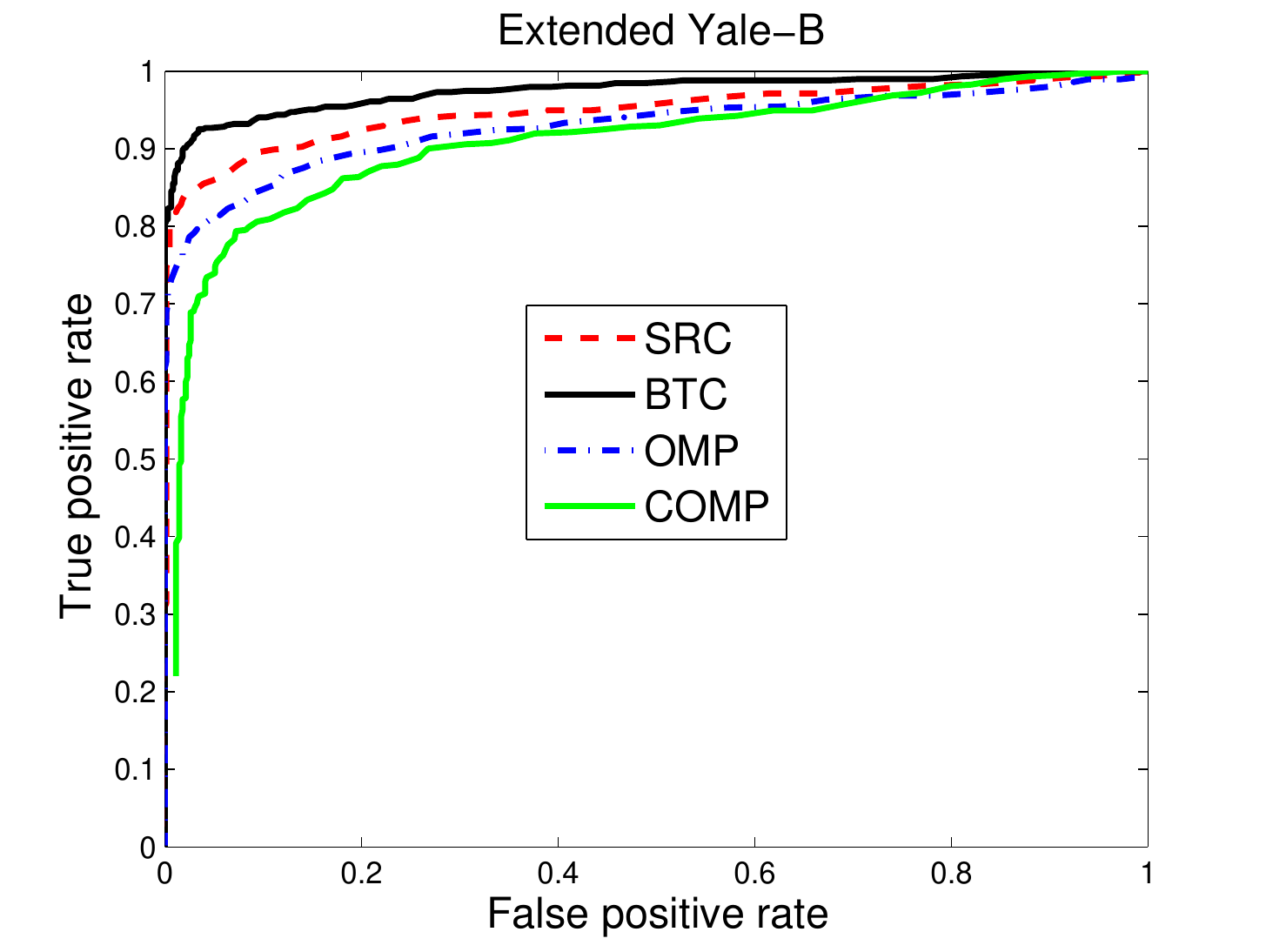}
\caption{ROC curves on Extended Yale-B}
\label{fig:yaleRoc}
\end{figure*}

\begin{figure*}
\centering
\includegraphics[width=0.8\textwidth]{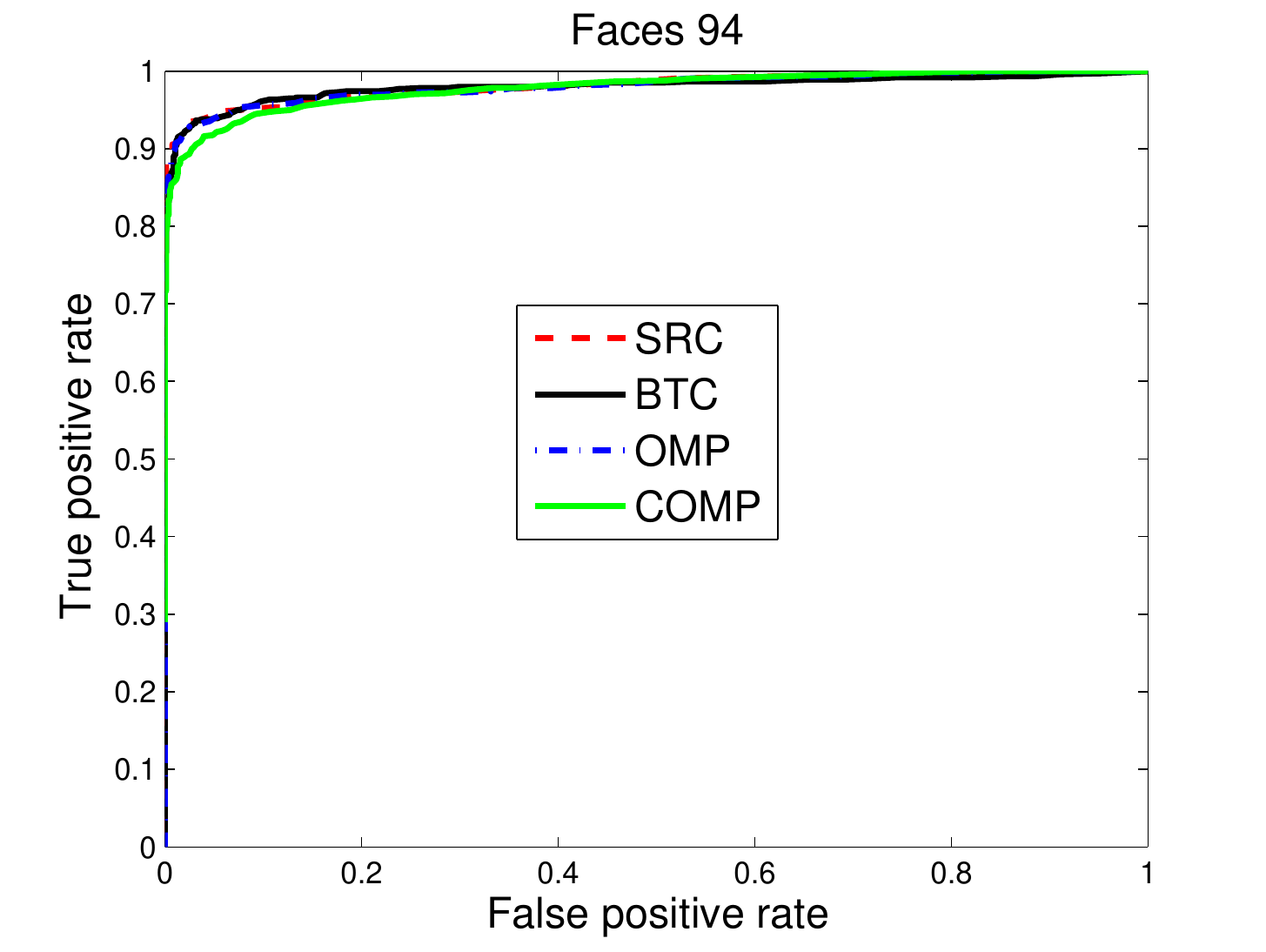}
\caption{ROC curves on Faces 94}
\label{fig:faces94Roc}
\end{figure*}

\begin{figure*}
\centering
 \includegraphics[width=0.8\textwidth]{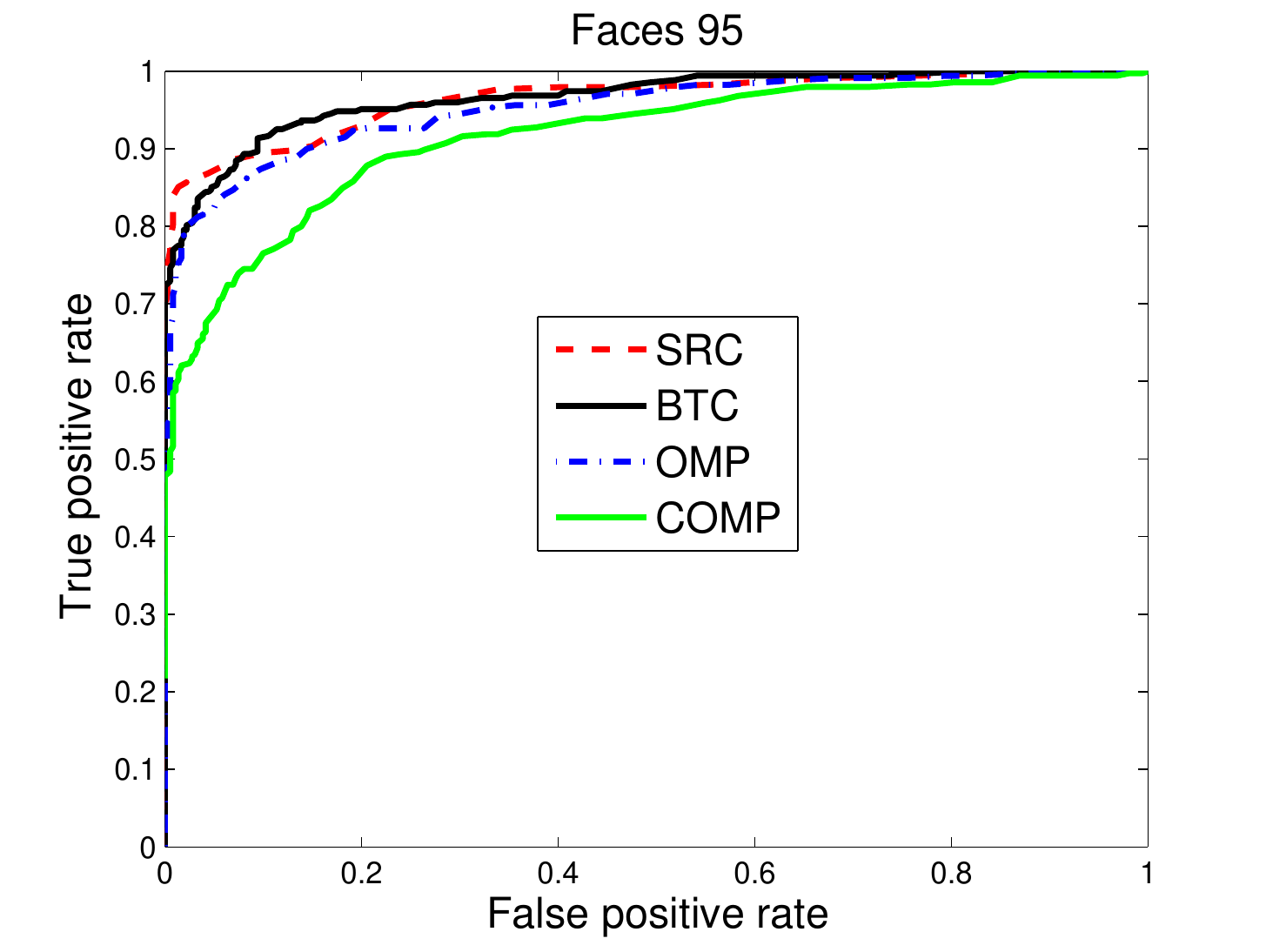}
 \caption{ROC curves on Faces 95}
 \label{fig:faces95Roc}
\end{figure*}

\begin{figure*}
\centering
 \includegraphics[width=0.8\textwidth]{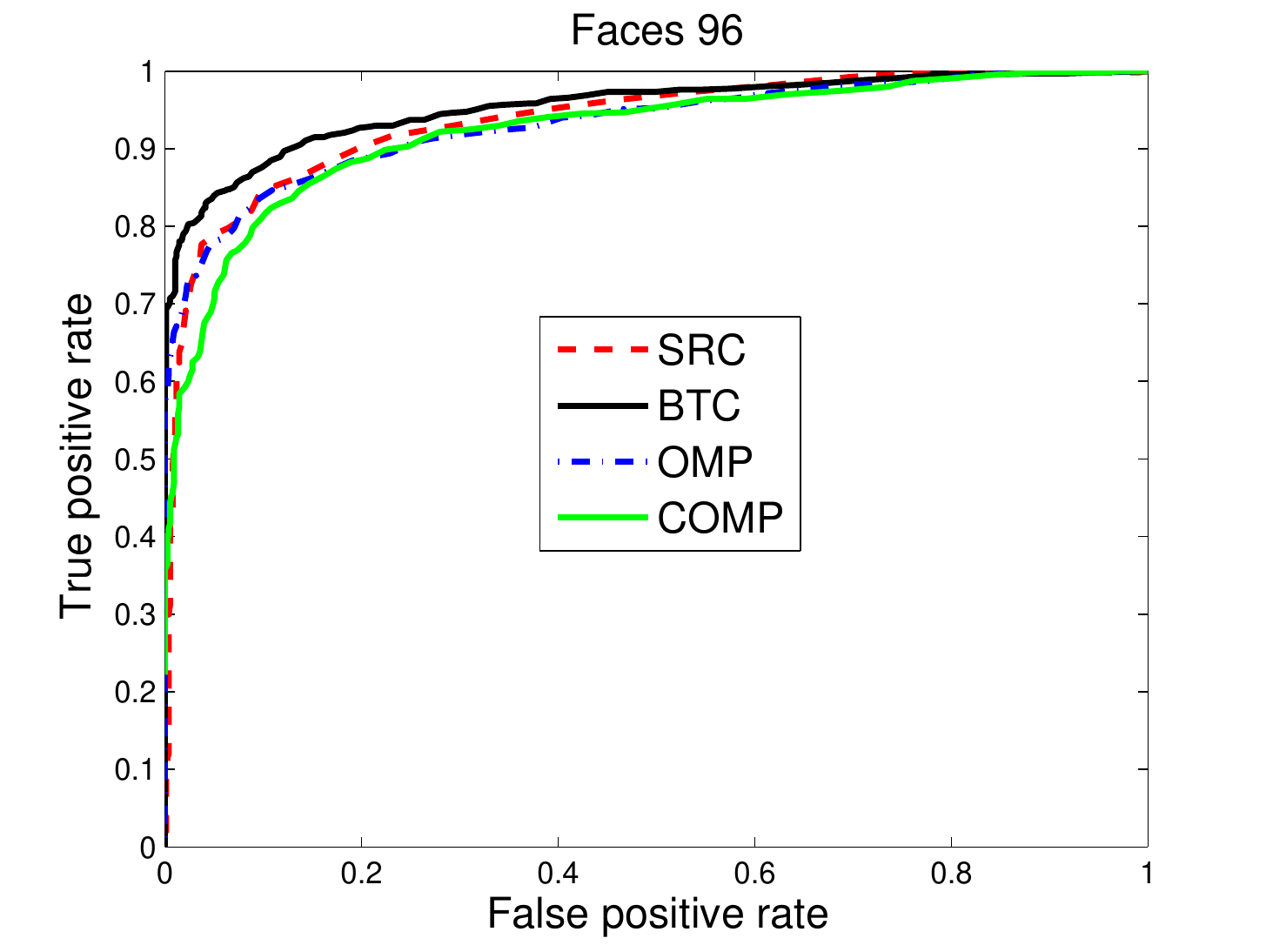}
 \caption{ROC curves on Faces 96}
 \label{fig:faces96Roc}
\end{figure*}

%
%

Unlike the rejection rule defined here, SRC algorithm uses the \emph{Sparsity Concentration Index} (SCI) to reject invalid test samples \cite{wright2009robust}. It is based on the estimated sparse code $x$. The details of the rule could be found in \cite{wright2009robust}. In order to compare the rejection performances of BTC, SRC, OMP, and COMP we designed a new experimental configuration. This time in the experiments, we included Extended Yale-B, Faces 94, 95, and 96 at the dimension 120d. We also included only half of the classes in the training sets. However, the test sets remained the same. Therefore, half of the classes and their samples in the test sets were invalid for the new dictionaries. Since the dictionaries were redesigned, we recalculated the $M$ values which are 72, 11, 33, 24 for Extended Yale-B, Faces 94, 95, and 96, respectively. We then generated the normalized Receiver Operating Characteristics (ROC) curves for each dataset and algorithm by simply sweeping $\tau$ over the range (0,1). Note that for SRC, OMP, and COMP we used the SCI rule. 

Fig. \ref{fig:yaleRoc} shows the ROC curves for Extended Yale-B dataset. In this set, BTC significantly outperforms the other algorithms. Since the dataset Faces 94 is an easy one, all algorithms generate nearly the same curves in Fig. \ref{fig:faces94Roc}. In Fig. \ref{fig:faces95Roc}, we see the results for Faces 95. In this case, up to 0.1 false positive rate, SRC slightly outperforms BTC. After that point, BTC performs better than SRC up to 0.3. After this rate, the performances are nearly the same. Finally, Fig. \ref{fig:faces96Roc} shows the ROC curves for Faces 96, which was the most difficult one. This time again BTC algorithm achieves the best results. 

\section{Conclusions}
In this chapter, we introduced basic thresholding classification for face recognition which has significant speed, accuracy, and  rejection performance improvements over $l_1$-minimization-based and greedy approaches. One of the main contributions is that unlike most of the classification algorithms, the computation performance of the proposed algorithm does not depend on the number of classes in the dictionary. It depends only on the feature vector size which is already intended to be small. By exploiting the output diversity property of the random projections and speed of the proposed algorithm, we developed a classifier ensemble mechanism which efficiently combines the outputs of the individual classifiers to further increase the classification accuracy especially in the case of small feature vector size. The ensemble technique also enables us to run each individual classifier in parallel manner. Finally, in order to reject invalid samples, we presented a rejection methodology which was actually developed by using the natural results of the sufficient identification condition rate. We demonstrated the performance and robustness of the algorithm under various well-known publicly available datasets. Simulation results showed that basic thresholding classification is a quite reasonable alternative to some state-of-the-art $l_1$-minimization-based and greedy classifiers.


\chapter{HYPER-SPECTRAL IMAGE CLASSIFICATION via BTC}
\label{chp:hyper1}

\section{Introduction}
Different materials on the surface of earth have different electromagnetic spectral signatures. In a given scene, those signatures could be captured by remote sensors with a small spatial and spectral resolution. Each pixel of a captured hyper-spectral image (HSI) or data cube contains very useful spectral measurements or features which could be used to distinguish different materials and objects. With the advancement of the sensor technology, current sensors are able to capture hundreds of spectral measurements. However, increasing the number of spectral features of a HSI pixel does not always help to increase the correct classification rate of the pixel. For instance, in the supervised classification techniques, at some point increasing the input feature vector size further may reduce the classification accuracy. This is known as high dimensionality problem or Hughes phenomenon \cite{hughes1968mean}.   

Several dimension reduction techniques have been proposed to eliminate the effects of Hughes phenomenon \cite{harsanyi1994hyperspectral}, \cite{wang2006independent}, \cite{li2012locality}. Besides, to increase the classification accuracy, various approaches have been proposed. Among those, support vector machines (SVM) technique outperformed classical methods such as K-nearest neighbor (K-NN) and radial basis function (RBF) networks \cite{melgani2004classification}. It has been shown that SVM distinguishes in terms of classification accuracy, computational cost, low vulnerability to Hughes phenomenon, and it requires a few training samples \cite{melgani2004classification}, \cite{camps2005kernel}. On the other hand, SVM approach has some limitations. First, it has parameter tuning ($C$, $\gamma$, error tolerance) and kernel (linear, RBF, etc.) selection steps which are done using k-fold cross validation in which some portion of the training data is used for testing purposes. Those procedures are cumbersome and the resulting parameter set may not be optimum for the test sets \cite{mountrakis2011support}. Second, since SVM is a binary classifier, a conversion strategy to multi-class case is required. An easy one is the one-against-all (OAA) strategy in which a test sample may result as unclassified which causes low classification accuracies \cite{mountrakis2011support}. Another strategy is the one-against-one approach in which number of binary classifiers ($K(K-1)/2$) increases dramatically as the number of classes ($K$) increases. The final limitation is that the probability outputs of the SVM classifier can not be directly provided and an estimation procedure is required such as logistic sigmoid \cite{platt1999probabilistic}. Therefore, the SVM-based methods using probability outputs must rely on those estimates.           

The SVM approach described above is in the class of pixel-wise algorithms since it uses only the spectral features. It is well known that the performance of a pixel-wise classifier could be improved by incorporating spatial information based on the fact that neighboring pixels in the homogeneous regions of a HSI have similar spectral signatures. Therefore, various approaches have been proposed to combine the spectral and spatial information. For instance, a composite kernels approach has been proposed in \cite{camps2006composite} which successfully enhances the classification accuracy of the SVM. A segmentation-based technique proposed in \cite{tarabalka2009spectral} combines the segmentation maps obtained via clustering and pixel-wise classification results of the SVM technique. Final decisions are made by majority voting in the adaptively defined windows. A similar framework has been proposed in \cite{tarabalka2010segmentation} which utilizes segmentation maps using watershed transformation. All these methods share common limitations since they are based on the SVM approach. 

One of the recent spatial-spectral frameworks, which utilizes an edge-preserving filter, has been proposed in \cite{kang2014spectral}. The method uses the SVM classifier as the pixel-wise classification step. For each class, the probability maps, which are the posterior probability outputs of the SVM classifier, are smoothed by an edge-preserving filter with a gray scale or rgb guidance image. Final decision for each pixel is then made based on the maximum probabilities. As an edge-preserving filter, they use one of the recent state-of-the-art techniques namely guided image filtering \cite{he2010guided}. Since the proposed framework is based on SVM, it has also common problems with the SVM-based classifiers. One alternative to SVM classifier is multinomial logistic regression (MLR)\cite{bohning1992multinomial} method in which class posterior probability distributions are learned using Bayesian framework. MLR has been successfully applied to HSI classification in \cite{borges2007evaluation}, \cite{li2012spectral}, and \cite{li2010semisupervised}. One of the recent techniques based on MLR has been proposed in \cite{li2011hyperspectral}. The method uses logistic regression via splitting and augmented Lagrangian (LORSAL) \cite{bioucas2009logistic} algorithm with active learning in order to estimate the posterior distributions. In the segmentation stage, it utilizes a multilevel logistic (MLL) prior to encode the spatial information. LORSAL-MLL (L-MLL) technique achieves promising results as compared to classical segmentation methods.

Recently, sparsity-based methods, sparse representation-based classification (SRC) and joint SRC (J-SRC), alternative to SVM-based frameworks have been successfully applied to HSI classification \cite{chen2011hyperspectral,chen2011hyperspectral,chen2013hyperspectral,zhang2016spectral,zhang2016weighted,bo2016hyperspectral}. SRC originally was proposed for face identification in \cite{wright2009robust}. Since SRC is based on $l_1$ minimization which includes solving costly convex optimization problem, greedy algorithms like orthogonal matching pursuit (OMP) \cite{tropp2007signal} and simultaneous OMP (SOMP) have been preferred for HSI classification in \cite{chen2011hyperspectral}. SOMP was originally proposed in \cite{tropp2006algorithms} for generic simultaneous sparse information recovery. HSI version of SOMP is based on a joint sparsity model assuming that the pixels in the small neighborhood of a test pixel share a common sparsity pattern. Those pixels are simultaneously represented by the linear combinations of the training samples of a predetermined dictionary. Kernelized versions of SOMP have been developed in \cite{chen2013hyperspectral}, \cite{liu2013spatial}, and \cite{li2014column} to exploit the non-linear nature of the kernels for a better class separability. Another version of SOMP namely weighted joint sparsity (W-JSM) or WSOMP has been proposed in \cite{zhang2014nonlocal}. The method calculates a non-local weight matrix for neighboring pixels of each test pixel. It then executes the standard SOMP with the calculated weight matrix. Multi-scale adaptive sparse representation (MASR), which utilizes spatial information at multiple scales, has been proposed in \cite{fang2014spectral}. A final one is the adaptive SOMP (ASOMP) which adaptively selects the neighborhood of the test pixel according to a predetermined segmentation map \cite{zou2014classification}. All SOMP-based approaches have several common drawbacks. The most important one is the extensive computational cost due to simultaneous sparse recovery of the surrounding pixels of the test sample. Another limitation is the parameter tuning step in which the sparsity level $K_0$, error tolerance $\epsilon$, maximum iterations, and the weight thresholds are needed to be tuned experimentally. Once the parameters are determined for a dataset, they may not be optimum for some other datasets. Therefore, there is no guidance for the parameter selection.  

In this chapter, we propose the basic thresholding classifier (BTC) for HSI classification. It is a pixel-wise light-weight method which classifies every pixel of an HSI image using only spectral features. During the classification it uses a predetermined dictionary containing labeled training pixels. For each pixel, it produces two outputs which are the error vector consisting of the residuals and the class label selected based on the minimal residual. To improve the classification accuracy of BTC, we extend our framework to a three-step spatial-spectral procedure. First, we run pixel-wise classification step using BTC for each pixel of a given HSI. The output residual vectors form a cube which is also interpreted as a stack of images. Every image is also called as residual map. Secondly, we smooth every residual map using an averaging filter. In the final step, we determine the class label of each pixel based on the minimal residual. The contribution of this chapter is threefold:
\begin{itemize}
\item We introduce a new sparsity-based algorithm for HSI classification which is light-weight, cost effective, easy to implement, and provides high classification accuracy. Unlike classical approaches such as minimum distance, K-NN, and SVM techniques, our method has low vulnerability to the corrupted, noisy, and partial features in the test samples since it is based on sparse representation which exploits the fact that the errors due to these kinds of features are often sparse with respect to the standard basis \cite{wright2009robust}. It also distinguishes from previous sparsity-based techniques in its ability to classify test pixels extremely rapidly.   
\item The proposed approach eliminates the limitations of well-known SVM technique. First, unlike SVM, it does not have training and cross-validation stages. We give the full guidance of threshold and regularization parameter selection of the BTC method. On the other hand, the parameters of SVM ($C$, $\gamma$) have to be determined using cross-validation which may result in a non-optimal set. Second, the computational cost of BTC does not significantly increase as the number of classes ($K$) increases. However, since the number of binary classifiers is dependent on the square of $K$ in one of the common conversion strategies (OAO) of SVM to multi-class case, the cost of SVM dramatically increases as K increases. Finally, SVM does not provide residuals which might be used for intermediate processing such as smoothing.
\item Our proposal can easily be extended to spatial-spectral case by smoothing the residual maps. This procedure eliminates high computational cost of joint sparsity model or SOMP-based techniques in which simultaneous sparse code recovery is essential. This low cost intermediate process extremely increases the classification accuracy of the proposed method. It is even able to outperform non-linear SVM-based techniques which use direct classification output maps of the spectral-only SVM.
\end{itemize}  

\section{HSI Classification}
In the context of HSI classification, spectral measurements are embedded into $B$ dimensional feature vectors. Let $a_{i,j}\in \mathbb{R}^B$ with $\norm{a_{i,j}}_2=1$ be the vector consisting of the spectral features of the $j$th sample of the class $i$ and let $A_i = [a_{i,1}~a_{i,2}~\ldots ~a_{i,N_i}]\in \mathbb{R}^{B\times N_i}$ denote the matrix containing the training pixels of the $i$th class with $N_i$ many pixels. Then, one can construct the dictionary $A$ with $C$ many classes in a way that $A=[A_1~A_2~\ldots~A_C]\in \mathbb{R}^{B\times N}$ where $N=\sum_{i=1}^{C}N_i$. 

Suppose that we constructed an HSI dictionary $A \in\mathbb{R}^{B\times N}$ and we are given a test pixel $y \in\mathbb{R}^B$ to be classified. In the sparse representation model, the assumption is that there exists a minimum $l_1$-norm sparse code $x\in\mathbb{R}^{N}$ such that $y=Ax$ or $y=Ax+\sigma$ where $\sigma$ is a small error \cite{wright2009robust}. The problem is equivalent to 

\begin{equation}\label{sparse}
\hat{x}=\arg\min_x \norm{x}_1~subject~to~y=Ax
\end{equation}
or alternatively subject to $\norm{y-A x }_2 \leq \sigma$. The class of $y$ is then found using the following expression.
\begin{equation}\label{sparse_class}
 class(y)=\arg \min_i \norm{y-A \hat{x_i}}_2 \; \forall i \in \{1,2,\ldots,C\}
\end{equation}
where $\hat{x_i}$ represents the $i$th class portion of the estimated sparse code $\hat{x}$. As we stated previously, solving (\ref{sparse}) is not an easy problem. $l_1$-minimization or convex relaxation-based techniques such as homotopy method \cite{yang2010fast} could be used. However, those techniques are quite expensive for the applications such as HSI classification. A faster way of solving (\ref{sparse}) is the OMP technique which has been successfully applied in HSI classification \cite{chen2011hyperspectral}. When we use OMP, the expression in (\ref{sparse}) is replaced with (\ref{omp}). 
\begin{equation}\label{omp}
\hat{x}=OMP(A, y)
\end{equation} 
The details of the OMP algorithm could be found in \cite{tropp2007signal}. In order to incorporate spatial information, Chen \emph{et al.} proposed the joint sparsity model (JSM) in \cite{chen2013hyperspectral}. In this case, not only the test pixel $y$ is used but also the surrounding $n-1$ pixels $y^1,y^2,\ldots,y^{n-1}$ in a given window $T$ are used in the minimization problem where the corresponding sparse codes are $x^1,x^2,\ldots,x^{n-1}$. This time expression (\ref{sparse}) is replaced with the following:
\begin{equation}\label{jsm}
\hat{X}=\arg\min_X \norm{X}_{row,0}~subject~to~Y=AX
\end{equation} or alternatively subject to $\norm{Y-A X }_F \leq \sigma$ where $X=[x~x^1~x^2~\ldots~x^{n-1}]$, $Y=[y~y^1~y^2~\ldots~y^{n-1}]$, and $\norm{X}_{row,0}$ is the number of nonzero rows. The class of $y$ is then found using the following expression:
\begin{equation}\label{jsm_class}
 class(y)=\arg \min_i \norm{Y-A \hat{X_i}}_F \; \forall i \in \{1,2,\ldots,C\}
\end{equation}
where $\hat{X_i}$ represents the $i$th class portion of the estimated sparse code matrix $\hat{X}$. For this case, greedy SOMP algorithm is preferred to solve (\ref{jsm}) in \cite{chen2011hyperspectral}. The details of SOMP could be found both in \cite{chen2011hyperspectral} and \cite{tropp2006algorithms}. As we stated previously, although SOMP technique is a greedy approach, it is computationally expensive since it simultaneously recovers the sparse codes of test and surrounding pixels.

All the limitations with the SVM and SOMP algorithms force us to use BTC for HSI classification \cite{toksoz2016btc}. One can use the following expression in order to classify the given pixel $y$:

\begin{equation}\label{btc_hyper1}
 Class(y) \gets BTC(A, y, M, \alpha)
\end{equation}
Note that the details of proposed technique could be found in Chapter \ref{chp:btc}. Before using the BTC algorithm, first, we need to determine the parameters $M$ and $\alpha$. In the following section, we will show how the parameters are selected using the quantity $\overline{\beta}_M$.
 
\section{Parameter Selection}
We plotted $\overline{\beta}_M$s for the dictionaries constructed using some publicly available well-known hyper-spectral images. The $\overline{\beta}_M$s for the dictionaries of Indian Pines, Salinas and Pavia University images are given in Fig. \ref{avgSicValues4}. The detailed description of each dataset and the corresponding dictionaries are given in Section \ref{sec:exp}. We see that in all plots of Fig. \ref{avgSicValues4}, $\overline{\beta}_M$ decays and reaches approximately to some minimum value. As we stated previously, any $M$ at which $\overline{\beta}_M$ is close to the minimum value is an acceptable choice. However, we need to consider that increasing $M$ will also increase the computational cost. For the given plots, we selected the regularization constant $\alpha$ as $10^{-4}$ which is a good choice for pixel-wise HSI classification. The effects of $\alpha$ could also be seen in the figures. Without $\alpha$, the decaying $\overline{\beta}_M$ would start to increase at some point due to the noisy eigenvalues which reduce the classification performance. One could think that the optimum choice of $\alpha$ is quite critical for BTC. It is important in spectral-only classification, however, in the following part, we will show that the effects of it will be compensated by the post processing smoothing in the spatial extension case. Therefore, using the optimal choice of $\alpha$ will not be critical anymore. Instead of optimal choice, we will prefer a quite small $\alpha$ in order to avoid an ill-conditioned matrix operation.

\begin{figure*}
\centering
\includegraphics[width = 0.9\textwidth]{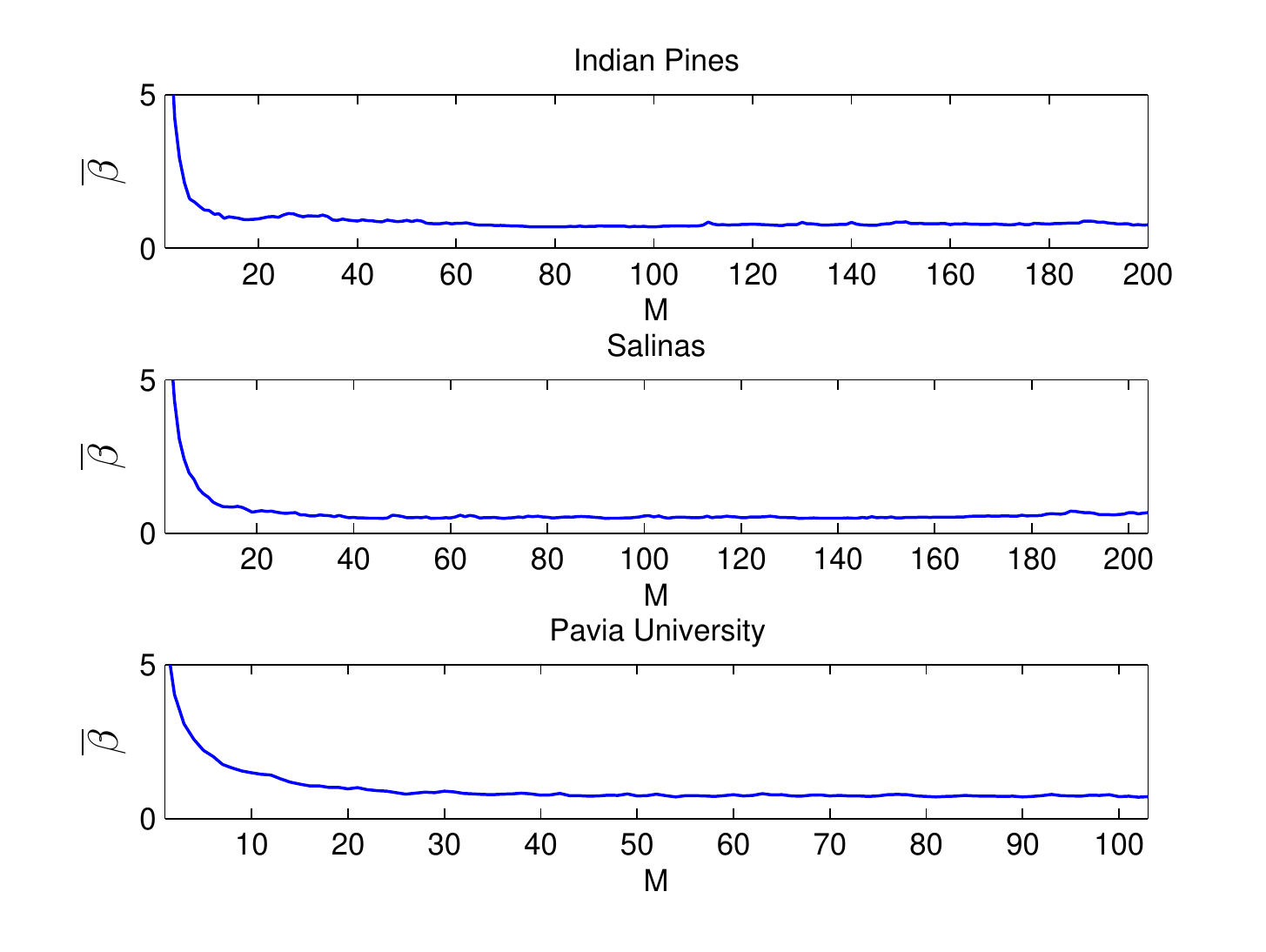}
\caption{$\overline{\beta}_M$ vs threshold for Indian Pines, Salinas, and Pavia University ($\alpha=10^{-4}$)}
\label{avgSicValues4}
\end{figure*} 

\section{Extension to Spatial-Spectral BTC}
In this section, we extend our pixel-wise proposal to a three-step spatial-spectral framework in order to incorporate spatial information. In the first step, each pixel $y \in\mathbb{R}^{B}$ in a given HSI, $H \in\mathbb{R}^{n_1 \times n_2 \times B}$ containing $n_1 \times n_2$ pixels, is classified using BTC. The produced outputs are not only the class labels but also the residual vectors, $\epsilon \in\mathbb{R}^{C}$, for all pixels. The resulting residual vectors form a residual cube $R \in\mathbb{R}^{n_1 \times n_2 \times C}$ which could also be interpreted as a stack of images representing residual maps ($map_i \in\mathbb{R}^{n_1 \times n_2}$ for all $i \in \{1,2,\ldots,C\}$). In the second step, each $map_i$ for all $i \in \{1,2,\ldots,C\}$ is smoothed via an averaging filter producing $\overline{map}_i $ for all $i \in \{1,2,\ldots,C\}$. Before smoothing operation, the values in the residual cube $R$ are normalized between 0 and 1. Also note that by using the intermediate classification map of spectral-only BTC, we simply set the residual values to the maximum value 1 in $map_i$ for the entries whose labels are not equal to $i$. This improves the classification performance. The smoothed maps form a smoothed residual cube $\overline{R}\in\mathbb{R}^{n_1 \times n_2 \times C}$. In the final step, class label of each pixel is determined based on minimal smoothed residuals, that is, $class(y) =\arg \min_i \overline{\epsilon}(i)$ where $i \in \{1,2,\ldots,C\}$. The overall framework is shown in Fig. \ref{SBTC}. For the spatial-spectral extension case, in which $\alpha$ is set to very small number ($10^{-10}$), we also plotted $\overline{\beta}_M$s in Fig. \ref{avgSicValues10} for the dictionaries of the images used in this work. As we see in the figure, the decaying $\overline{\beta}_M$s start to increase at some point because of the noisy eigenvalues. As we stated in the previous part, the optimal choice of $\alpha$ for BTC in the spatial extension is not critical and a quite small value ($10^{-8}, 10^{-9}, 10^{-10},$ etc.) could be used only to prevent the singular matrix inverse. Also note that using a bit larger value such as $10^{-4}$ will reduce the classification accuracy of the spatial-spectral classifier since it eliminates some discriminative small eigenvalues. Therefore, in the spatial extension, we always prefer to use a very small regularization constant.

\begin{figure*}
\centering
\includegraphics[width = 1.0\textwidth]{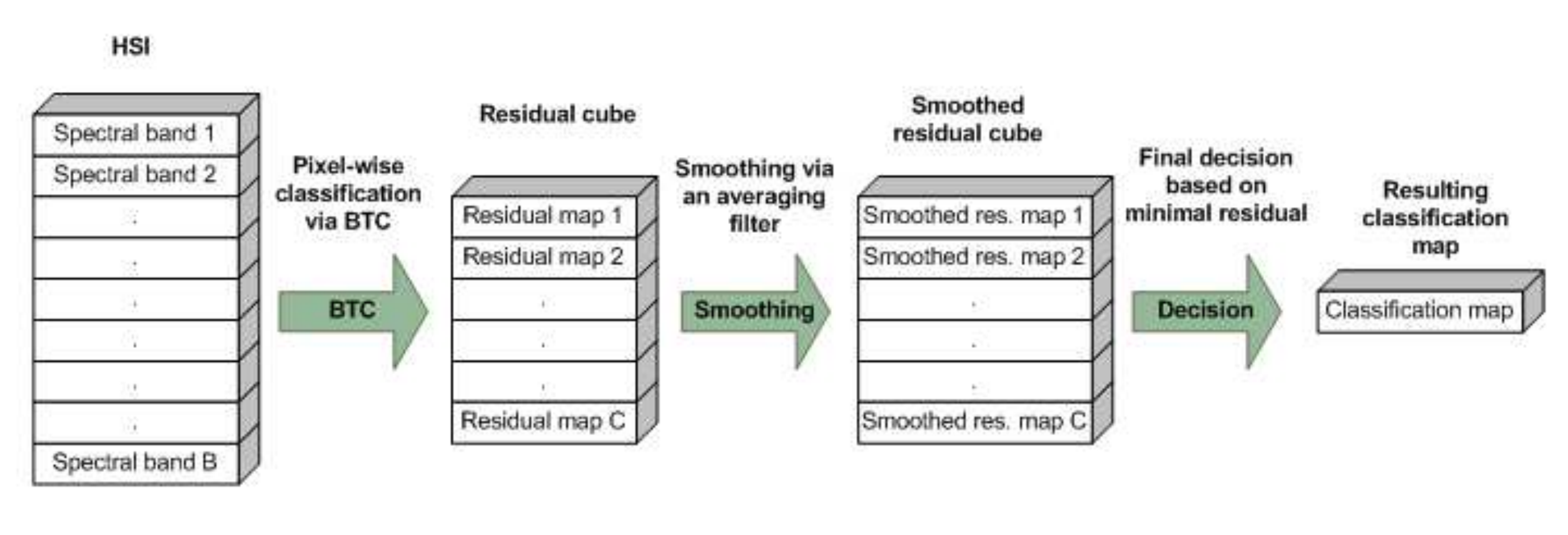}
\centering
\caption{Spatial-Spectral BTC}
\label{SBTC}
\end{figure*}

\begin{figure*}
\centering
\includegraphics[width = 0.9\textwidth]{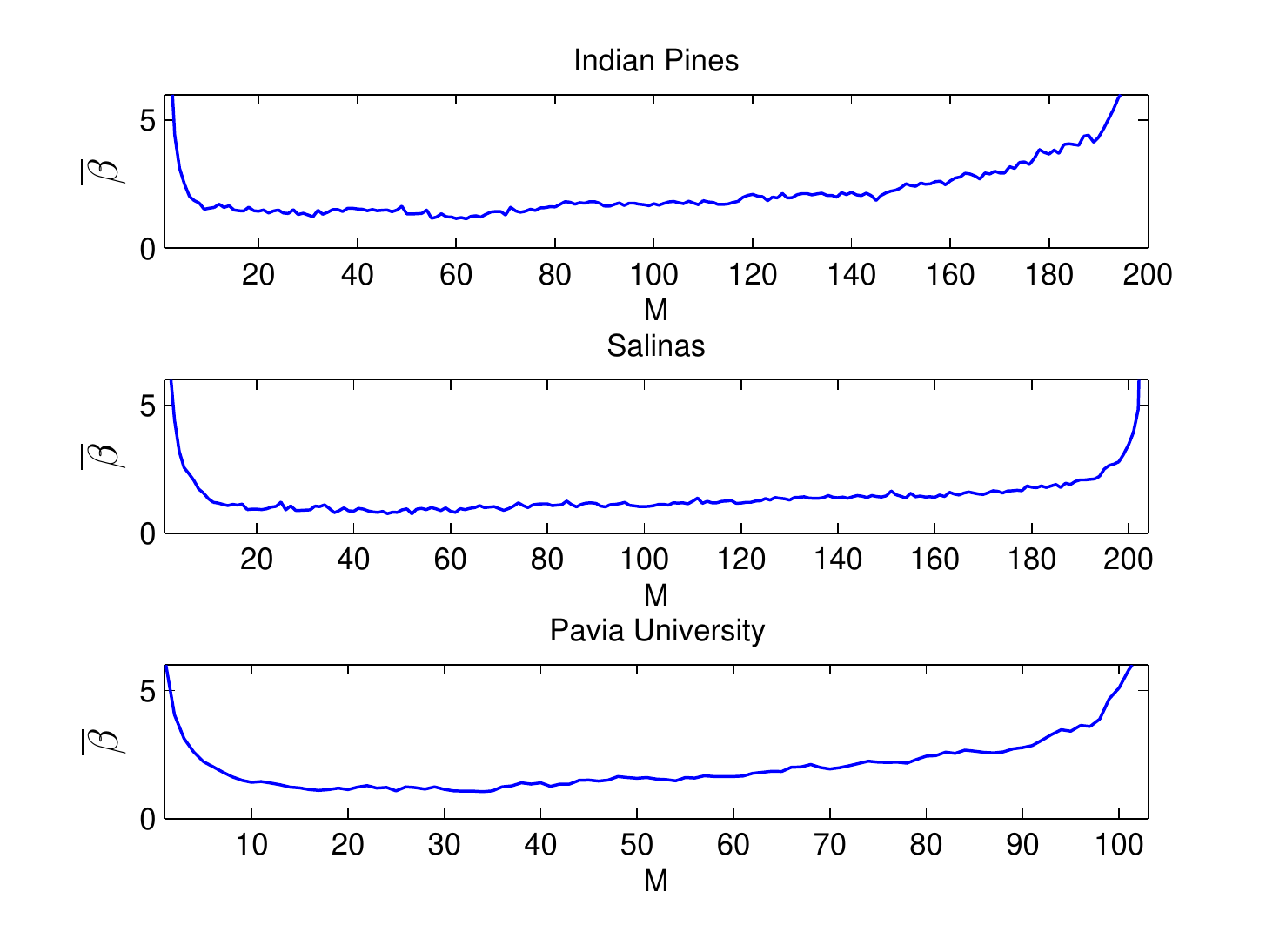}
\caption{$\overline{\beta}_M$ vs threshold for Indian Pines, Salinas, and Pavia University ($\alpha=10^{-10}$)}
\label{avgSicValues10}
\end{figure*} 
 
\section{Experimental Results}
\label{sec:exp}
\subsection{Datasets}
As in \cite{kang2014spectral}, we also used three well-known publicly available datasets namely Indian Pines, Salinas, and Pavia University in this work. We present the brief description of each dataset in Table \ref{description}. Both the Indian Pines and Salinas images were captured by Airborne Visible/Infrared Imaging Spectrometer (AVIRIS) sensor. The Pavia University image was captured by Reflective Optics System Imaging
Spectrometer (ROSIS) sensor. Before the experiments, some noisy water absorption bands (20 bands for Indian Pines and Salinas, 12 bands for Pavia University) were discarded. The 3-Band color image, ground truth, each class and the corresponding number of training and test pixels are given for Indian Pines, Salinas, and Pavia University in Fig. \ref{indianPines}, Fig. \ref{salinas}, and Fig. \ref{paviaUniversity}, respectively.     
 
\begin{table*}
\footnotesize
 \caption{Description of each dataset}
 \label{description}
 \centering
 \begin{tabular}{|c |c | c| c| c| c| c|}
 \hline
 Dataset & Size & Spatial  & Spectral  & Num. of & Sensor & Num. of  \\ 
         &      & resolution & coverage & classes &  & bands \\ \hline
 Indian Pines &	145 $\times$  145 $\times$  220 &	20 m &	0.4 $-$ 2.5 $\mu$m & 16 & AVIRIS &	200 \\ \hline
 Salinas	& 512 $\times$  217$\times$  224 &	3.7 m &	0.4 $-$ 2.5 $\mu$m &	16 & AVIRIS & 204 \\ \hline
 Pavia University &	610 $\times$ 340 $\times$  115 &	1.5 m &	0.43 $-$ 0.86 $\mu$m & 9 &	ROSIS &	103 \\
 \hline
 \end{tabular}
 \end{table*}
 
 \begin{figure*}
 \centering
 \includegraphics[width = 0.65\textwidth]{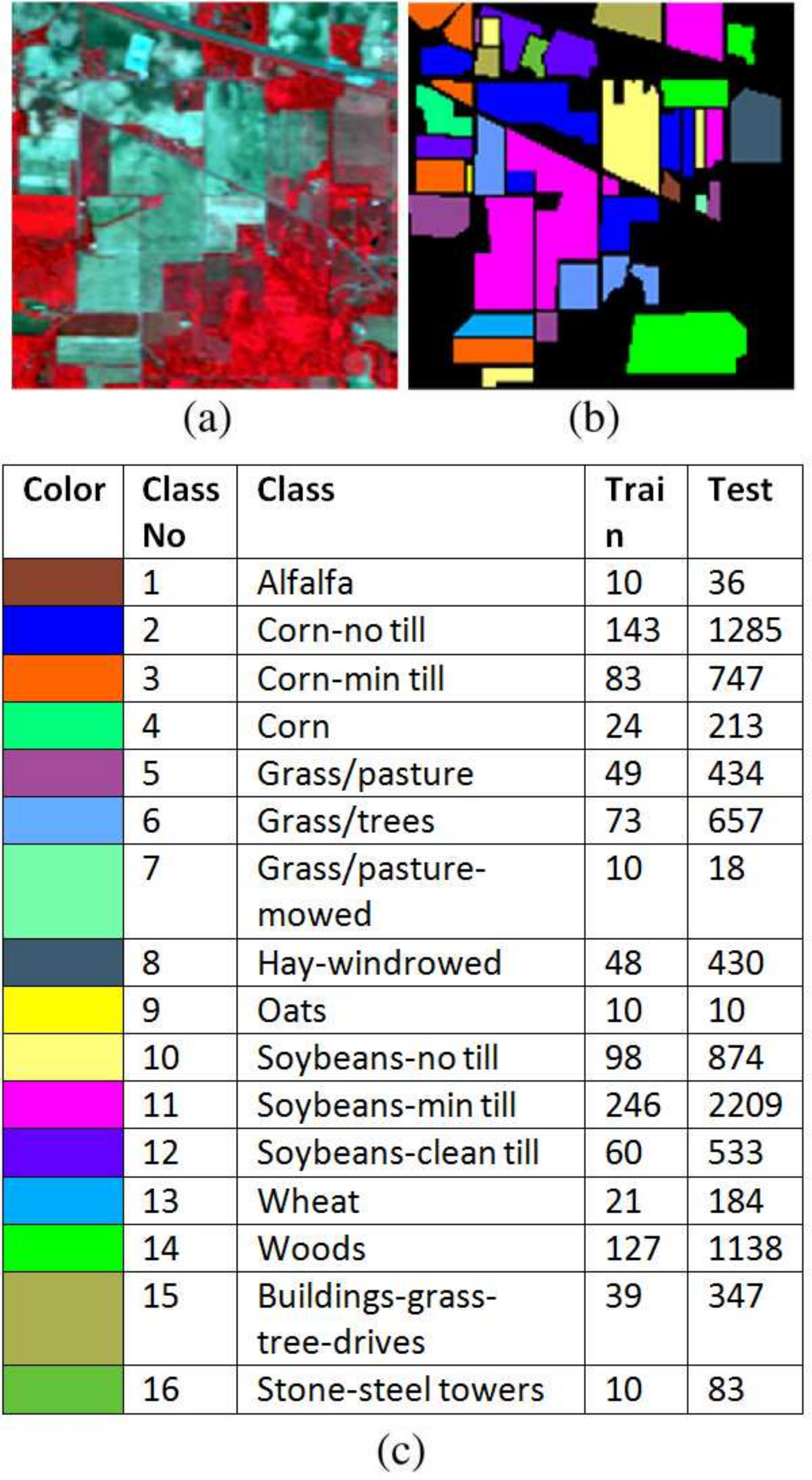}
 \centering
 \caption{ a-) 3-Band color image b-) ground truth image, and c-) each class and the corresponding number of training and test pixels of Indian Pines dataset }
 \label{indianPines}
 \end{figure*} 
 
 \begin{figure*}
 \centering
 \includegraphics[width = 0.65\textwidth]{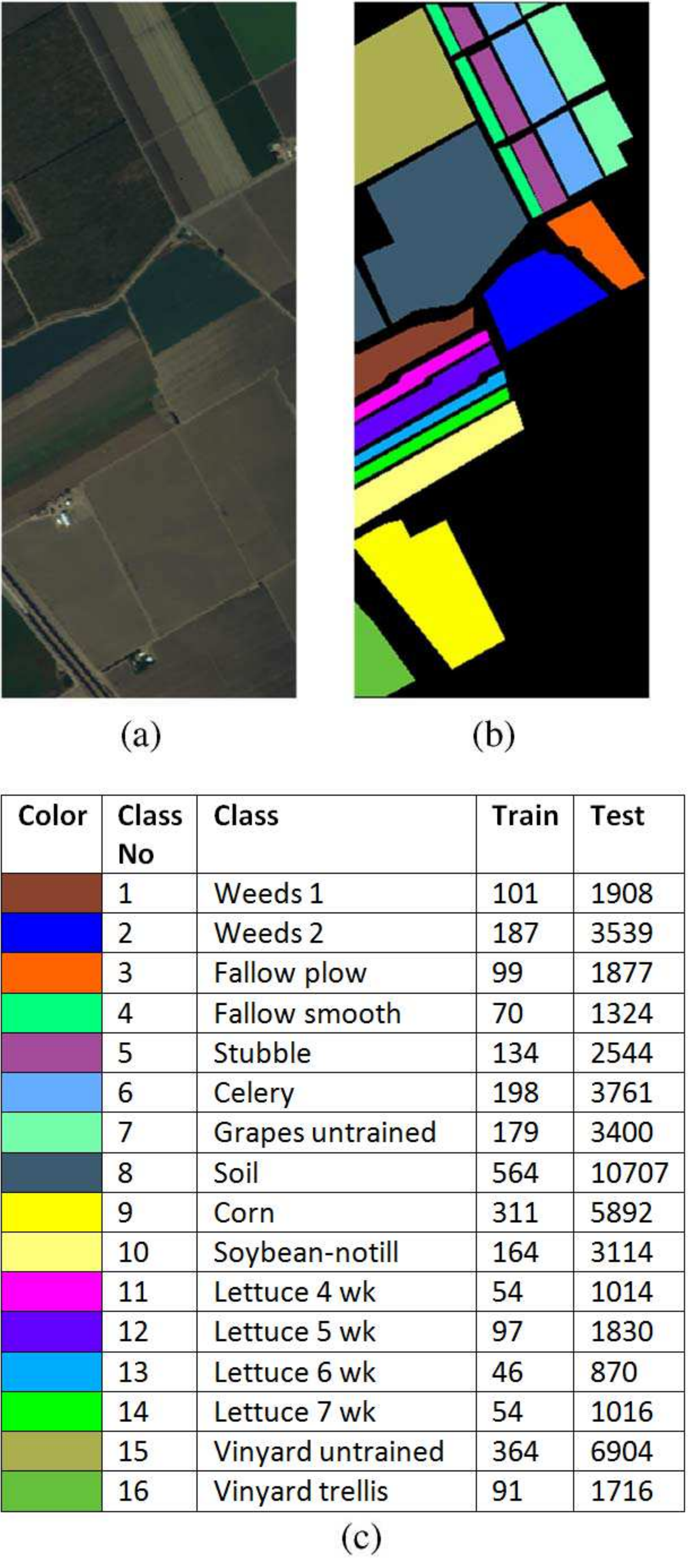}
 \centering
 \caption{ a-) 3-Band color image b-) ground truth image, and c-) each class and the corresponding number of training and test pixels of Salinas dataset }
 \label{salinas}
 \end{figure*} 
 
 \begin{figure*}
 \centering
 \includegraphics[width = 0.65\textwidth]{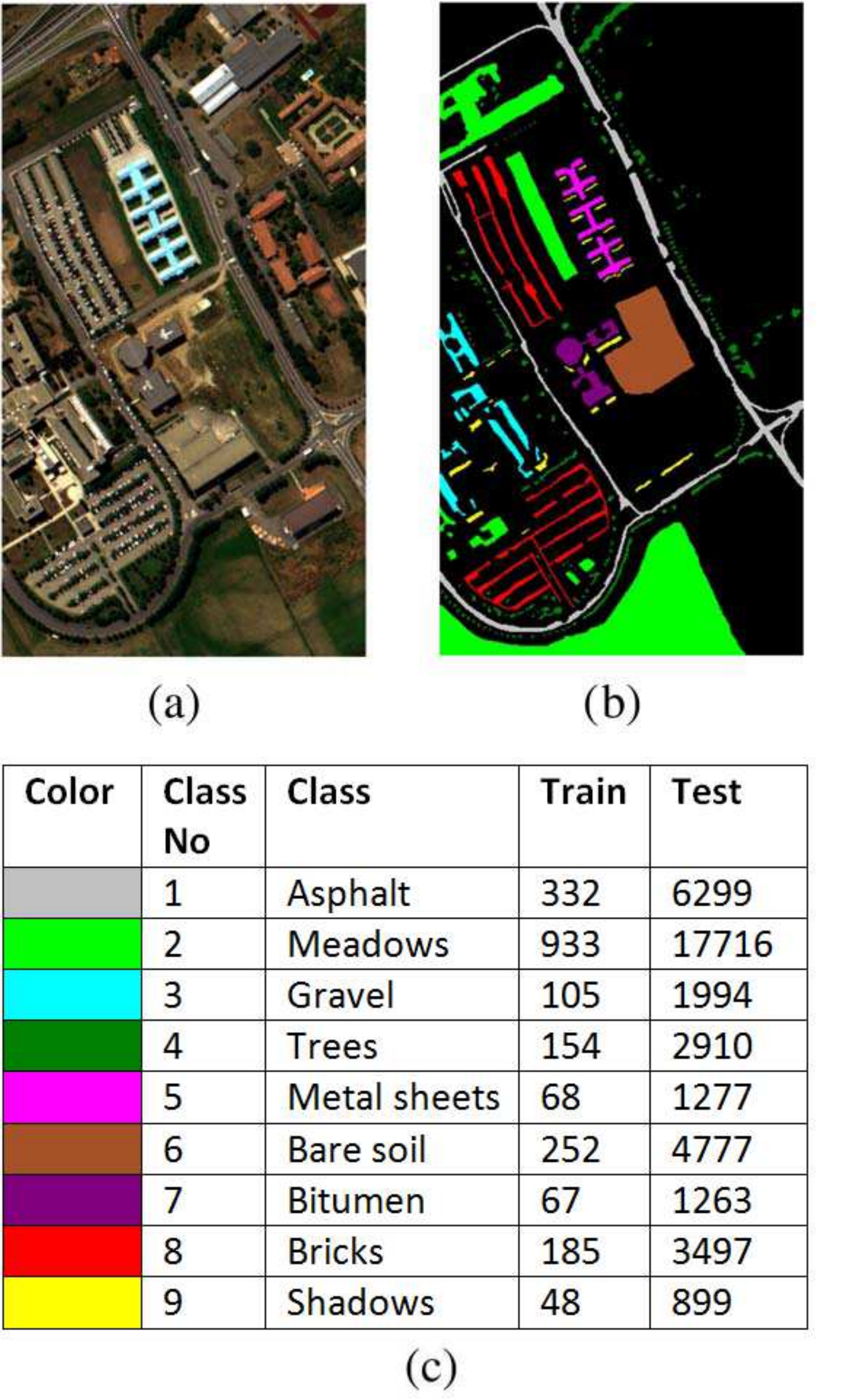}
 \centering
 \caption{ a-) 3-Band color image b-) ground truth image, and c-) each class and the corresponding number of training and test pixels of Pavia University dataset }
 \label{paviaUniversity}
 \end{figure*} 
 
 \subsection{Performance Indexes}
 During the experiments, we used three commonly preferred performance indexes namely overall accuracy (OA), average accuracy (AA), and the kappa coefficient ($\kappa$). In addition to them, we also included the computation time. OA shows the percentage of the correctly classified samples and AA gives the average of the percentages of the correctly classified samples in each class. The $\kappa$ coefficient is used to measure the degree of consistency \cite{cohen1960coefficient}. The computation time is also an important measure which determines whether the classifier is feasible for the real time applications or not.  
   
 \subsection{Experimental Setup}
 In the experiments, we included various spectral-only algorithms such as SVM \cite{melgani2004classification}, OMP \cite{chen2011hyperspectral}, and BTC as well as some state-of-the-art spatial-spectral methods such as EPF-G-g \cite{kang2014spectral}, L-MLL \cite{li2011hyperspectral}, SOMP \cite{chen2011hyperspectral}, and spatial-spectral BTC. EPF-G-g is based on SVM and guided image filtering. We used two versions of this method namely SVM-GF (based on guided filter \cite{he2010guided}) and SVM-WLS (based on weighted least squares filtering \cite{farbman2008edge}). For our spatial-spectral proposal, we also used the same filtering techniques in order to smooth the resulting residual maps and called the two versions of it as BTC-GF and BTC-WLS. For SVM-based classifiers, we used well-known and fast LIBSVM library which was written in C++ \cite{chang2011libsvm}. The parameters ($C$, $\gamma$) of SVM were chosen by 5-fold cross validation by varying $C$ from $10^{-2}$ to $10^4$ and $\gamma$ from $2^{-3}$ to $2^4$. Since the non-linear RBF kernel is superior to the linear kernel (dot product), we used the RBF kernel for the SVM-based methods. In all experiments, for SVM-GF and BTC-GF, we set the filtering parameters namely filtering size ($r$) and blur degree ($\epsilon$) to $3$ and $0.01$, respectively. Those values are proposed in \cite{kang2014spectral} for HSI classification. Since both GF and WLS filters require a gray-scale guidance image, we obtained it by extracting the first principal component of the given HSI using the principal component analysis (PCA) \cite{van2009dimensionality} procedure. Since the L-MLL approach requires an active learning stage, we set the initial training set to the half of the all training set for all experiments. We then incremented the samples by 50 using random selection (RS) method up to the whole training set. The classification and segmentation stages were performed after the learning stage. For OMP and SOMP classifiers, we used the SPArsa toolbox provided by Julien Mairal \cite{mairal2010online}. There is no guidance for the selection of sparsity parameter ($L$) of the OMP and SOMP techniques. However, based on the experiments in \cite{chen2011hyperspectral} and \cite{zhang2014nonlocal}, we set $L$ to 25 for OMP and 30 for SOMP. Similarly, we set the spatial window size (T) to 25 for SOMP technique. Note that we used the same values for all experiments.
 
 \begin{figure*}
 \centering
 \includegraphics[width = 1.0\textwidth]{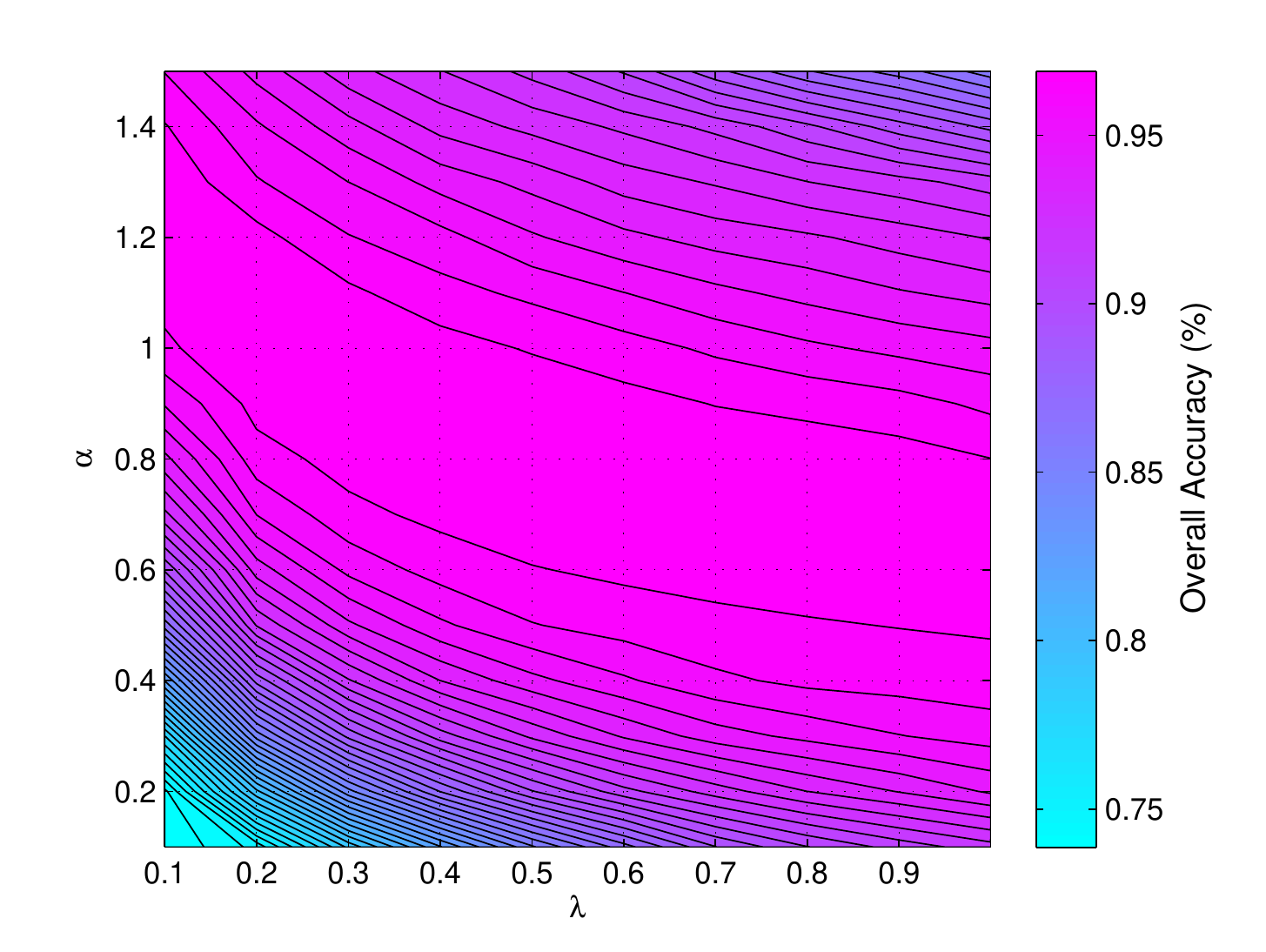}
 \centering
 \caption{Overall accuracy ($\%$) grid by varying the $\lambda$ and $\alpha$ pair of WLS filter on Indian Pines dataset using BTC-WLS method}
 \label{wlsParams}
 \end{figure*} 
 
 The reason that we include the WLS filter is because it does not cause halo artifacts at the edges of an image as the degree of smoothing increases \cite{farbman2008edge}. However, the guided filter and bilateral filter \cite{tomasi1998bilateral} techniques tend to blur over the edges. Similar to the filtering techniques mentioned here, WLS filtering also has two input parameters namely the degree of smoothing ($\lambda$) and the degree of sharpening ($\alpha$) over the preserved edges. Note that the reader should not confuse the parameter $\alpha$ here with the regularization parameter of BTC method. Since there is no guidance for the selection of WLS filtering parameters in HSI classification, in this thesis, our proposal is 0.4 for $\lambda$ and 0.9 for $\alpha$ based on the experiment we performed on the Indian Pines dataset. Note that the details of the experiment are presented in the following subsection. We obtained the metric OA by varying $\lambda$ from 0.1 to 1.0 and $\alpha$ from 0.1 to 1.5 using the BTC-WLS method. The resulting OA grid is shown in Fig. \ref{wlsParams}. The wide contour in the middle of the figure shows the highest OA region. The coordinates of the points in this region are also acceptable choices for the WLS filtering in our case. Note that the WLS filter parameters are fixed to the proposed values ($\lambda=0.4,\alpha=0.9$) for all experiments.     
 
 Regarding the regularization parameter ($\alpha$) of our proposal, as we pointed out in the previous section, we set it to $10^{-4}$ for the spectral-only BTC and $10^{-10}$ for BTC-GF and BTC-WLS techniques. Note that like spectral-only SVM and OMP methods, spectral-only BTC is also not a practical HSI classifier. We used those methods as the reference for our experiments. Therefore, the choice of proposed regularization parameter in the spectral-only case is not critical for real-world applications. However, for BTC-GF and BTC-WLS approaches, we should select a quite small number in order to only prevent ill-conditioned matrix operation. Since any quite small number is acceptable, there is no guidance required for the selection of this parameter.  
      
 For the choice of threshold parameter (M) of the proposed BTC-based techniques, once the dictionary is determined, one can immediately plot the $\overline{\beta}_M$ by varying M from $1$ to the number of bands ($B$) available. Then, the best value of it in the described sense could be easily determined by looking at the resulting plot. Notice that the procedure does not require any cross validation or experiment and it is totally based on the predetermined dictionary. In the following subsections, we will provide the estimated $M$ values for each experiment based on the $\overline{\beta}_M$ plots given in the previous section. A final note about the setup is that we performed all experiments on a PC with a quad-core 2.67 GHz processor and 4GB of memory. 
 \begin{table*}
 \footnotesize
 \caption{The results (accuracy per class ($\%$), OA ($\%$), AA ($\%$), $\kappa$ ($\%$), Time (s) of twenty Monte Carlo runs) for spectral-only and spatial-spectral methods on Indian Pines dataset}
 \label{resultIndianPines}
 \centering
 \begin{tabular}{|c|| c| c| c || c| c| c| c| c| c|}
 \cline{2-10}
 \multicolumn{1}{ c }{}& \multicolumn{3}{ |c|| }{Spectral-Only} & \multicolumn{6}{ |c| }{Spatial-Spectral} \\ \hline
 Class & SVM & OMP & BTC & SOMP & L-MLL & SVM-GF & SVM-WLS & BTC-GF & BTC-WLS\\ \hline 
  1 & 60.28 & 57.50 & \textbf{76.11} & 81.11 & 87.22 & 98.33 & 96.67 & \textbf{100.00} & 98.89 \\
  2 & \textbf{75.86} & 61.63 & 72.82 & 89.06 & 92.61 & 91.59 & 90.88 & \textbf{95.67} & 94.27\\
  3 & \textbf{65.19} & 53.00 & 62.24 & 84.75 & 89.18 & 85.74& 86.79 & 95.88 & \textbf{97.51} \\
  4 & \textbf{47.70}& 31.60& 43.71 & 69.81 & 79.90 & 87.65 & 90.94 & \textbf{95.40} & 94.41 \\
  5 & 87.12 & 85.51 & \textbf{88.23} & 93.02& 90.69 & \textbf{95.46} & 95.23 & 95.18 & 94.63 \\
  6 & 95.40& 93.15& \textbf{96.26} & 99.21& 99.54 &  \textbf{100.00} & 99.74 & \textbf{100.00}& \textbf{100.00} \\
  7 & \textbf{88.89} & 81.67& 85.56 & 89.44 & 93.33& \textbf{97.22} & \textbf{97.22} & \textbf{97.22} & 96.11 \\
  8 & 95.19 & 92.47 & \textbf{97.33} & 99.86& 99.69 & \textbf{100.00} & \textbf{100.00} & \textbf{100.00} & \textbf{100.00} \\
  9 & \textbf{91.00} & 63.00 & 84.00 & 78.00 & 96.00 & 98.00 & 86.00 & 97.00 & \textbf{100.00} \\
  10 & \textbf{74.45}& 57.25 & 70.35 & 85.27 & 89.01 & \textbf{93.83} & 93.36 & 93.27 & 93.59 \\
  11 & \textbf{83.73} & 72.87 & 81.98& 90.38 & 95.10 & 98.07 & 98.58 & 98.81 & \textbf{99.41} \\
  12 & 65.35 & 50.99 & \textbf{65.63} & 87.90 & 93.65 & 95.03 & 96.81 & 98.59 & \textbf{99.25} \\
  13 & 94.62 & 95.71 & \textbf{98.91} & 98.15 & 99.18 & 99.67 & \textbf{100.00} & \textbf{100.00}& \textbf{100.00}\\
  14 & 95.02 & 92.55 & \textbf{96.26} & 98.91& 97.78 & \textbf{100.00} & \textbf{100.00} & 99.93 & 99.83\\
  15 & \textbf{52.31}& 44.84 & 51.84 & 81.67 & 83.97 & 84.06 & 84.15 & 91.87 & \textbf{93.29}\\
  16 & 78.19 & 81.93 & \textbf{85.42} & \textbf{99.40} & 75.42 & 95.06 & 91.93 & 98.43 & 96.99\\ \hline \hline
  OA & \textbf{80.18} & 70.80& 79.17 & 90.74 & 93.36 & 95.16 & 95.34 & 97.38& \textbf{97.51}\\
  AA & \textbf{78.14} & 69.73 & 78.04 & 89.12 & 91.39 & 94.98& 94.27 & 97.33 & \textbf{97.39}\\
  $\kappa$ & \textbf{77.28} & 66.53 & 76.11& 89.43 & 92.42 & 94.46 & 94.66& 97.01 & \textbf{97.15} \\
  Time & 11.44 & 7.33 & \textbf{4.62}& 86.98 &\textbf{3.23}& 26.48 & 27.47 & 7.89& 8.77\\ \hline
 \end{tabular}
 \end{table*}
 \begin{figure*}
 \centering
 \includegraphics[width = 1.0\textwidth]{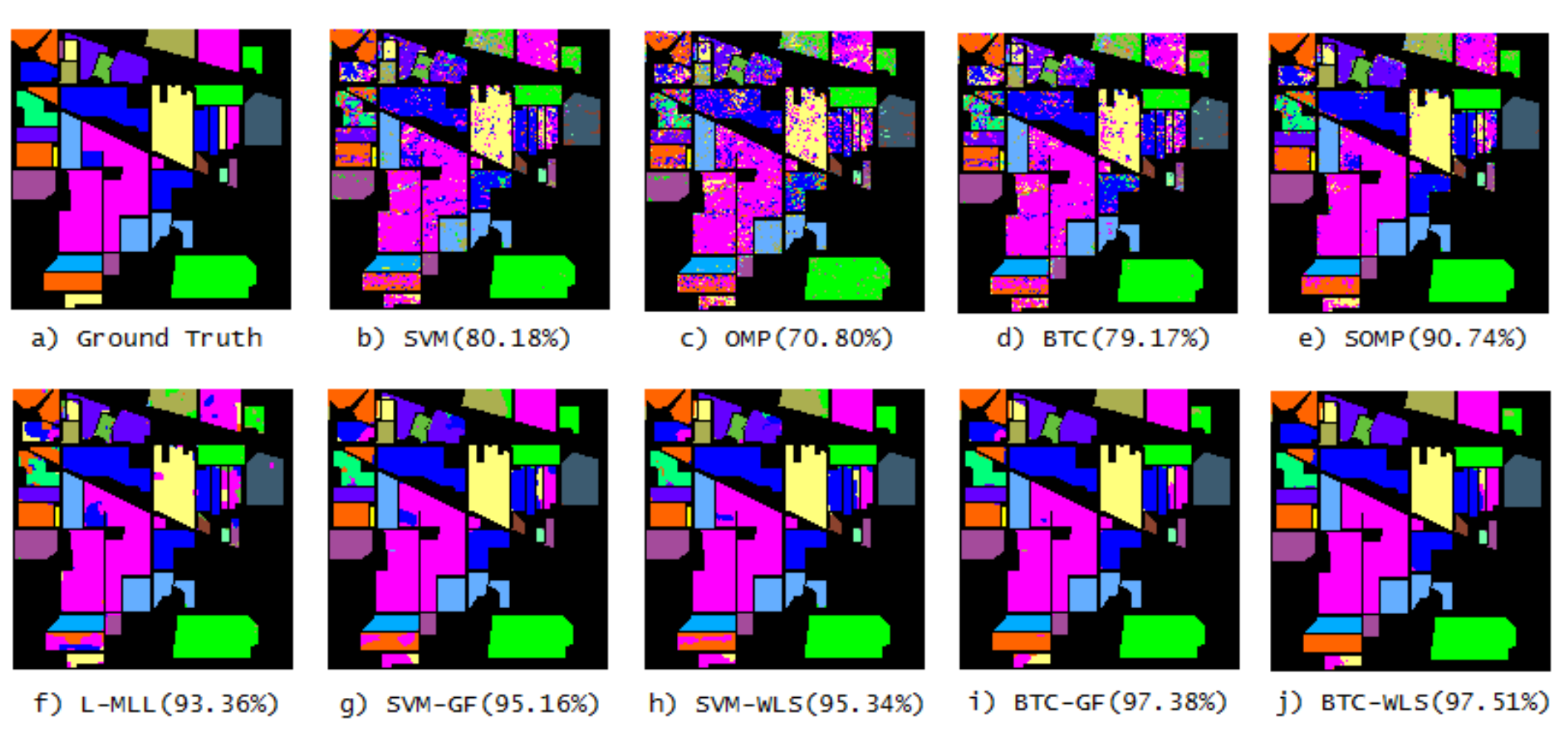}
 \centering
 \caption{ Classification Maps on Indian Pines Dataset with Overall Accuracies}
 \label{indianClassMap}
 \end{figure*}
 \subsection{Results on Indian Pines Dataset}
 We performed the first experiment on the Indian Pines dataset. We randomly selected $10\%$ of the ground truth pixels for training set (dictionary) and the remaining $90\%$ for testing. We limited the minimum number of training pixels to 10 for each class. Each class and the corresponding number of test and training pixels are given in Fig. \ref{indianPines}c. For sparsity-based algorithms, each training pixel (column) of the constructed dictionary was $l_2$ normalized. For SVM-based approaches, both the training and testing pixels were normalized between $-1$ and $1$. For BTC method, based on the $\overline{\beta}_M$ plot given in Fig. \ref{avgSicValues4}, we set the threshold $M$ to $80$ at which $\overline{\beta}_M$ approaches approximately to the minimum. Similarly, for BTC-GF and BTC-WLS methods, by observing the plot given in Fig. \ref{avgSicValues10}, we set the $M$ parameter to $60$. Based-on this configurations, we repeated the experiments for twenty Monte Carlo runs both for spectral-only methods (SVM, OMP, BTC) and for spatial-spectral approaches (SOMP, L-MLL, SVM-GF, SVM-WLS, BTC-GF, BTC-WLS). The classification results are shown in Table \ref{resultIndianPines}. We have also provided the classification maps with overall accuracies (\%) in Fig. \ref{indianClassMap}. In spectral-only case, as expected the SVM method having non-linear kernel (RBF) achieves best results in terms of OA, AA, and $\kappa$. This is because unlike SVM approach both OMP and BTC methods use linear kernel (dot product). On the other hand, classification results of the BTC method are very close to those of SVM. In terms of computation time, the best result is achieved by the BTC method. In spatial-spectral case, both BTC-GF and BTC-WLS approaches achieve best results in terms of all metrics except the computation time. The OA and AA differences between BTC-WLS and SVM-WLS are about $2\%$ and $3\%$, respectively. When we compare the BTC-GF and BTC-WLS methods with the SOMP method, the performance differences are significant. This results show that smoothing residual maps is quite effective way of improving the classification accuracy. The L-MLL method achieves better than the SOMP technique and it is the fastest algorithm in this case. However, the performance differences in terms of classification accuracies between the proposed BTC-based methods and L-MLL are significant. The SVM-based approaches perform quite slower than the proposed algorithms and the L-MLL method. The results also show that the SOMP method is computationally very expensive. Note that the time metric in the table include the classification time as well as the smoothing time. When implementing the WLS filter, we applied the preconditioned conjugate gradient method with incomplete Cholesky decomposition which highly reduces the computational cost \cite{saad2003iterative}. A final note about this experiment is that WLS-based approaches are generally superior to the GF-based methods. 
 \begin{table*}
  \footnotesize
 \caption{The results (accuracy per class ($\%$), OA ($\%$), AA ($\%$), $\kappa$ ($\%$), Time (s) of twenty Monte Carlo runs) for spectral-only and spatial-spectral methods on Salinas dataset}
 \label{resultSalinas}
 \centering
 \begin{tabular}{|c|| c| c| c || c| c| c| c| c| c|}
 \cline{2-10}
 \multicolumn{1}{ c }{}& \multicolumn{3}{ |c|| }{Spectral-Only} & \multicolumn{6}{ |c| }{Spatial-Spectral} \\ \hline
 Class & SVM & OMP & BTC & SOMP & L-MLL & SVM-GF & SVM-WLS & BTC-GF & BTC-WLS \\ \hline
 1 & 99.15 & 99.28 & \textbf{99.44} & 90.68 & 99.88 & \textbf{100.00} & \textbf{100.00} & \textbf{100.00} & \textbf{100.00}\\
 2 & 99.76 & \textbf{99.85} & 99.22 & 92.11 & 99.97 & \textbf{100.00} & \textbf{100.00} & 99.99 & \textbf{100.00}\\
 3 & \textbf{99.03} & 97.19 & 97.38 & 91.10 & 99.87 & \textbf{100.00} & \textbf{100.00} & \textbf{100.00} & \textbf{100.00}\\
 4 & \textbf{99.34} & 98.56 & 99.33 & 86.71 &  99.05 & \textbf{100.00} & \textbf{100.00} & \textbf{100.00} & 99.95\\
 5 & 98.26 & 98.16 & \textbf{98.74} & 88.77 & 99.04 & 99.58 & 99.45 & \textbf{99.83} & \textbf{99.83}\\
 6 & 99.77 & \textbf{99.88} & 99.77 & 88.52 & 99.87 & \textbf{100.00} & 99.97 & \textbf{100.00} & \textbf{100.00}\\
 7 & 99.66 & \textbf{99.84} & 99.64 & 93.53 & 99.80 & \textbf{100.00} & \textbf{100.00} & 99.94 & \textbf{100.00}\\
 8 & 87.30 & 78.07 & \textbf{88.83} & 88.54 & 93.79& 96.03 & 98.35 & 98.04 & \textbf{99.43}\\
 9 & 99.56 & \textbf{99.77} & 99.42 & 91.09 & 99.98& \textbf{100.00} & \textbf{100.00} & \textbf{100.00} & \textbf{100.00}\\
 10 & 94.65& \textbf{96.18} & 94.38 & 84.43& 96.48 & 99.51 & \textbf{100.00} & 99.86 & \textbf{100.00}\\
 11 & 96.98 & \textbf{98.08} & 97.93 & 86.16 & 97.82 & \textbf{100.00} & \textbf{100.00} & \textbf{100.00} & \textbf{100.00}\\
 12 & 99.49 & 99.32 & \textbf{99.98} & 86.49 & \textbf{100.00} & \textbf{100.00} & \textbf{100.00} & \textbf{100.00} & \textbf{100.00}\\
 13 & 97.79 & \textbf{97.86} & 97.51 & 85.94 & 97.94 & \textbf{100.00} & 99.49 & 99.90 & 99.34\\
 14 & 95.70 & 95.30 & \textbf{96.85} & 90.20 & 97.76& 99.84 & 98.93 & \textbf{99.85} & 99.64\\
 15 & \textbf{71.51} & 64.93 & 65.89 & 72.43 & 72.91 & 82.15 & 86.21 & 85.61 & \textbf{89.36}\\
 16 & 98.42 & \textbf{98.50} & 98.38 & 89.11 &  98.87 & 99.99 & \textbf{100.00} & 99.92 & 99.99\\ \hline \hline
 OA & \textbf{92.68} & 89.94 & 92.20 & 87.02 & 94.59 & 96.72 & 97.74 & 97.63 & \textbf{98.42}\\
 AA & \textbf{96.02} & 95.05 & 95.79 & 87.86 & 97.06 & 98.57 & 98.90 & 98.93 & \textbf{99.22}\\
 $\kappa$ & \textbf{91.84} & 88.80 & 91.30 & 85.61 & 93.96 & 96.34 & 97.48 & 97.36 & \textbf{98.24}\\
 Time & 61.81 & 113.51 & \textbf{24.37} & 988.99 &\textbf{30.16} & 135.85 & 138.94 & 42.80 & 45.55\\
 \hline
 \end{tabular}
 \end{table*}
 \begin{figure*}
 \centering
 \includegraphics[width = 1.0\textwidth]{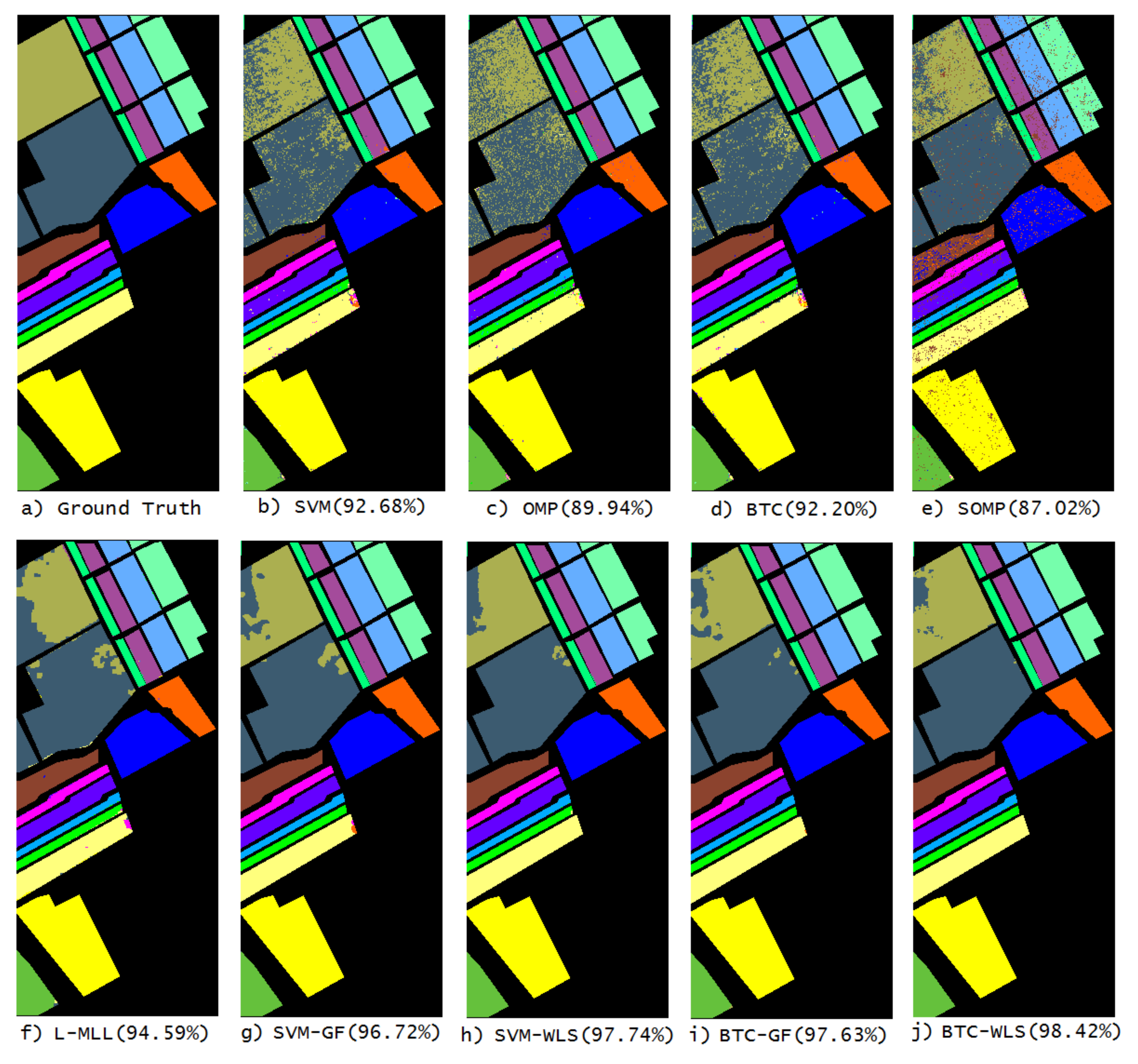}
 \centering
 \caption{ Classification Maps on Salinas Dataset with Overall Accuracies}
 \label{salinasClassMap}
 \end{figure*}
 \subsection{Results on Salinas Dataset}
 The second experiment was performed on the Salinas dataset. Since the number of ground truth pixels is large as compared to the first dataset, this time we selected $5\%$ of the ground truth pixels for training set (dictionary) and the remaining $95\%$ for testing. Similarly, each class, the number of training, and test pixels are given in Fig. \ref{salinas}c. The normalization process for the SVM and sparsity-based approaches was performed as described in the first experiment. For spectral-only and spatial-spectral BTC methods, we set $M$ to 50, and to 20, respectively based on the $\overline{\beta}_M$ plots given in Fig. \ref{avgSicValues4} and in Fig. \ref{avgSicValues10}. We repeated the experiments for twenty Monte Carlo runs for all methods. The classification results are shown in Table \ref{resultSalinas}. We have also included the classification maps with overall accuracies (\%) in Fig. \ref{salinasClassMap}. Since the HSI contains large homogeneous areas as compared to the previous case, all spectral-only approaches perform quite similarly. The SVM method again achieves best results except the classification time due to the reason we explained in the previous subsection. In terms of classification time, again the fastest one is the BTC method. Since the dictionary size is larger in this case, the OMP method performs quite slowly as compared to the others. In the spatial-spectral case, similar to the previous experiment, BTC-WLS technique again achieves the best results in terms of all metrics except the computation time. The performance differences between the filtering-based methods and the SOMP method are significant. In terms of OA, BTC-WLS performs approximately 4\% better than the L-MLL method. In terms of computational cost, although L-MLL approach slightly outperforms the BTC-based techniques, the computational performance differences between these methods and the SVM-based approaches are significant.
 \begin{table*}
 \footnotesize
 \caption{The results (accuracy per class ($\%$), OA ($\%$), AA ($\%$), $\kappa$ ($\%$), Time (s) of twenty Monte Carlo runs) for spectral-only and spatial-spectral methods on Pavia University dataset}
 \label{resultPaviaUni}
 \centering
 \begin{tabular}{|c|| c| c| c || c| c| c| c| c| c|}
 \cline{2-10}
 \multicolumn{1}{ c }{}& \multicolumn{3}{ |c|| }{Spectral-Only} & \multicolumn{6}{ |c| }{Spatial-Spectral} \\ \hline
 Class& SVM & OMP & BTC & SOMP & L-MLL & SVM-GF & SVM-WLS & BTC-GF & BTC-WLS \\ \hline
 1 & \textbf{93.07} & 66.58 & 87.74 & 87.30 & 98.66 & 98.74 & 98.69 & \textbf{98.88} & 98.79\\
 2 & \textbf{97.93} & 88.51 & 93.47 & 99.69 &  99.70 & 99.92 & \textbf{100.00} & 99.83 & \textbf{100.00}\\
 3 & \textbf{75.02} & 56.61 & 72.84 & 79.80 & 80.98 & 85.40 & 85.88 & 90.71 & \textbf{93.79}\\
 4 & \textbf{93.44} & 84.41 & 91.87 & 92.58 & 95.94 & 96.41 & 90.89 & \textbf{96.53} & 90.71\\
 5 & 99.21 & \textbf{99.80} & 99.50 & \textbf{100.00}& 99.53 & \textbf{100.00} & \textbf{100.00}& \textbf{100.00} & \textbf{100.00}\\
 6 & \textbf{86.50}& 58.13 & 73.35& 65.62 & 98.97 & 98.55 & \textbf{99.97} & 94.55 & 99.27\\
 7 & \textbf{84.74} & 55.31 & 73.29 & 80.59 & 90.25 & 99.81 & \textbf{100.00} & 97.27 & 99.46\\
 8 & \textbf{90.09} & 60.33& 73.32 & 67.78 & 94.64 & 98.78 & \textbf{99.45} & 96.79 & 98.87\\
 9 & \textbf{99.86} & 79.78 & 88.55& 63.87& \textbf{99.87} & 99.86& 98.77 & 99.35 & 97.58\\ \hline \hline
 OA & \textbf{93.39} & 76.38 & 86.81 & 88.16 & 97.54& 98.51 & 98.37 & 98.03 & \textbf{98.59}\\
 AA & \textbf{91.10} & 72.16 & 83.77 & 81.92 & 95.39 & 97.50 & 97.07 & 97.10 & \textbf{97.61}\\
 $\kappa$ & \textbf{91.19} & 68.67 & 82.46 & 83.98 & 96.73 & 98.02 & 97.84 & 97.38 & \textbf{98.12}\\
 Time & 17.58 & 47.19 & \textbf{13.48} & 503.57 & \textbf{37.37} & 92.21 & 96.65 & 63.80 & 67.94\\
 \hline
 \end{tabular}
 \end{table*}
 \begin{figure*}
 \centering
 \includegraphics[width = 1.0\textwidth]{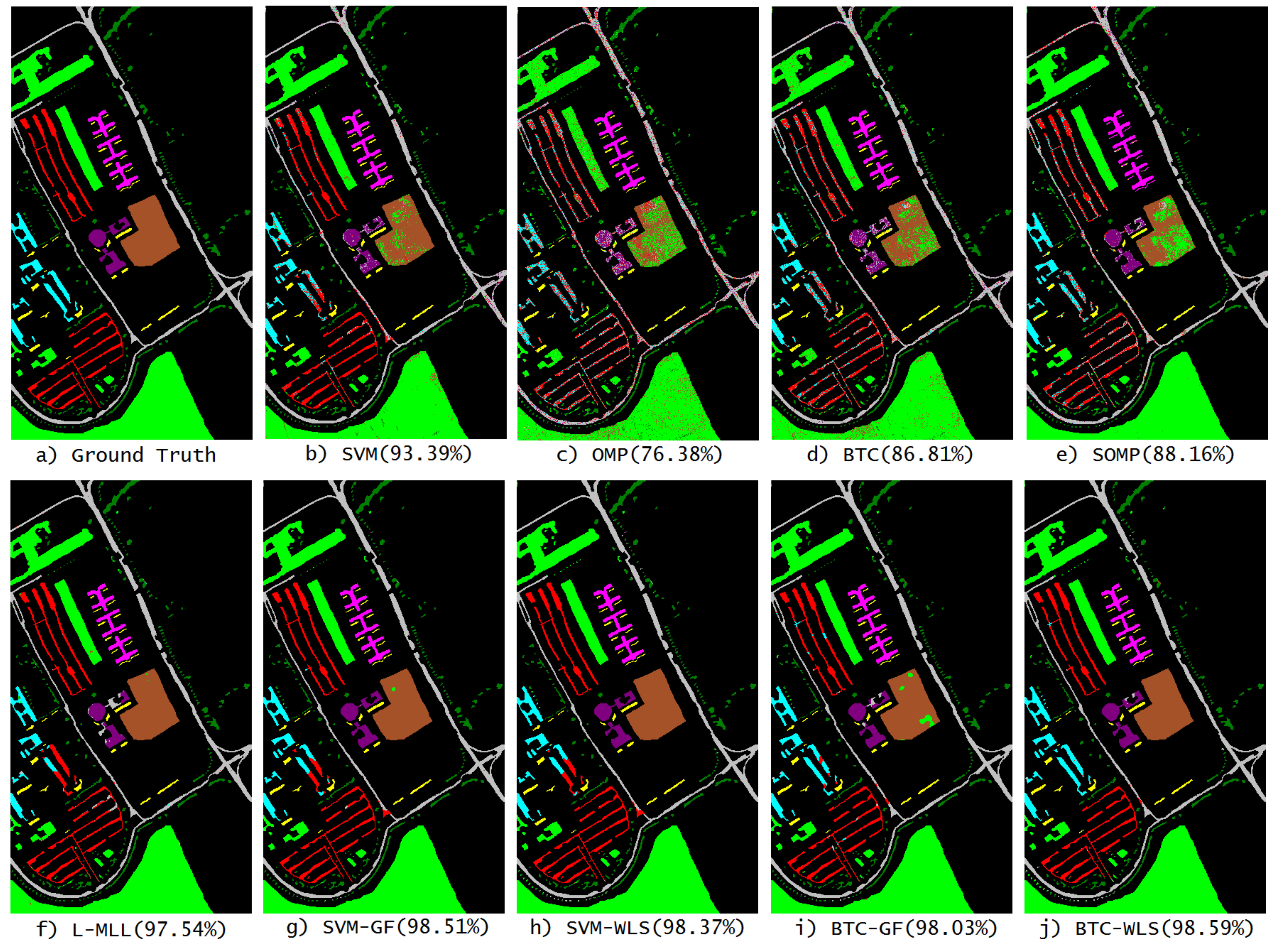}
 \centering
 \caption{ Classification Maps on Pavia University Dataset with Overall Accuracies}
 \label{paviaUniClassMap}
 \end{figure*}
 \subsection{Results on Pavia University Dataset}
 The third experiment was performed on the Pavia University dataset. Similar to the previous experiment, we chose the $5\%$ of the ground truth pixels for training and $95\%$ for testing. Each class of this dataset, the number of training and testing pixels for each class are shown in Fig. \ref{paviaUniversity}c. The threshold parameter ($M$) was set to 40 for BTC and 35 for BTC-GF and BTC-WLS techniques based on the plots given in Fig. \ref{avgSicValues4} and in Fig. \ref{avgSicValues10}, respectively. The experiments were repeated for twenty Monte Carlo runs for all methods using the same configurations given in the previous experiments. The classification results are given in Table \ref{resultPaviaUni}. Similarly, we have also provided the classification maps with overall accuracies (\%) for each method in Fig. \ref{paviaUniClassMap}. This time in spectral-only case the performance differences between the SVM method and the sparsity-based methods are significant in terms of OA, AA, and $\kappa$ metrics. It even outperforms the SOMP method. This is because the classes of Pavia University dataset are highly non-linearly separable. On the other hand, in spatial-spectral case, BTC and SVM-based methods perform quite similarly. The best results are achieved by the BTC-WLS method. This time L-MLL technique achieves 1\% worse than BTC-WLS in terms of OA. For this case, the fastest method is the L-MLL approach as well. Classification accuracy difference on the third class is significant between the BTC and SVM-based approaches. SVM-GF method outperforms the SVM-WLS technique. We believe this is the reason that the SVM-GF method has been proposed since it seems more robust for HSI classification. 
 \subsection {Results using Fixed Training Set}
 We performed the last experiment on the Pavia University dataset using fixed training set which is available on Dr. Li's web page\footnote{http://www.lx.it.pt/~jun/}. The original set contains 3921 training samples, however, we noticed that only 2777 of them are included in the ground truth (Fig. \ref{paviaUniversityFixed}a). Therefore, for this experiment, we used 2777 of the training samples, which are shown in Fig. \ref{paviaUniversityFixed}b. The number of training and testing pixels for each class are shown in Table \ref{fixedResultPaviaUni}. Fig. \ref{paviaUniversityFixed}b also shows that most of the pixels are grouped together which means that the samples belonging to the same class are highly similar to each other. This reduces the diversity of the input features causing lower classification accuracies as compared to the previous experiments. For this experiment, we evaluated the performances of the BTC-WLS technique and the final spatial-spectral techniques (SOMP, L-MLL, SVM-GF) proposed in \cite{chen2011hyperspectral}, \cite{li2011hyperspectral}, and \cite{kang2014spectral}, respectively. The threshold parameter ($M$) was set to 5 for the BTC-WLS technique based on the $\overline{\beta}_M$ plot using fixed training set. The classification results are given in Table \ref{fixedResultPaviaUni}. Again BTC-WLS achieves best results in terms of OA, AA, and $\kappa$. This time the performance differences between the proposed method and the other techniques are significant. Although L-MLL outperforms the other approaches in terms of classification time, the performance difference is not significant between our proposal and the L-MLL technique in terms of this metric. On the other hand, SOMP and SVM-GF performs quite slowly as compared to L-MLL and BTC-WLS.               
 \begin{figure*}
 \centering
 \includegraphics[width = 0.6\textwidth]{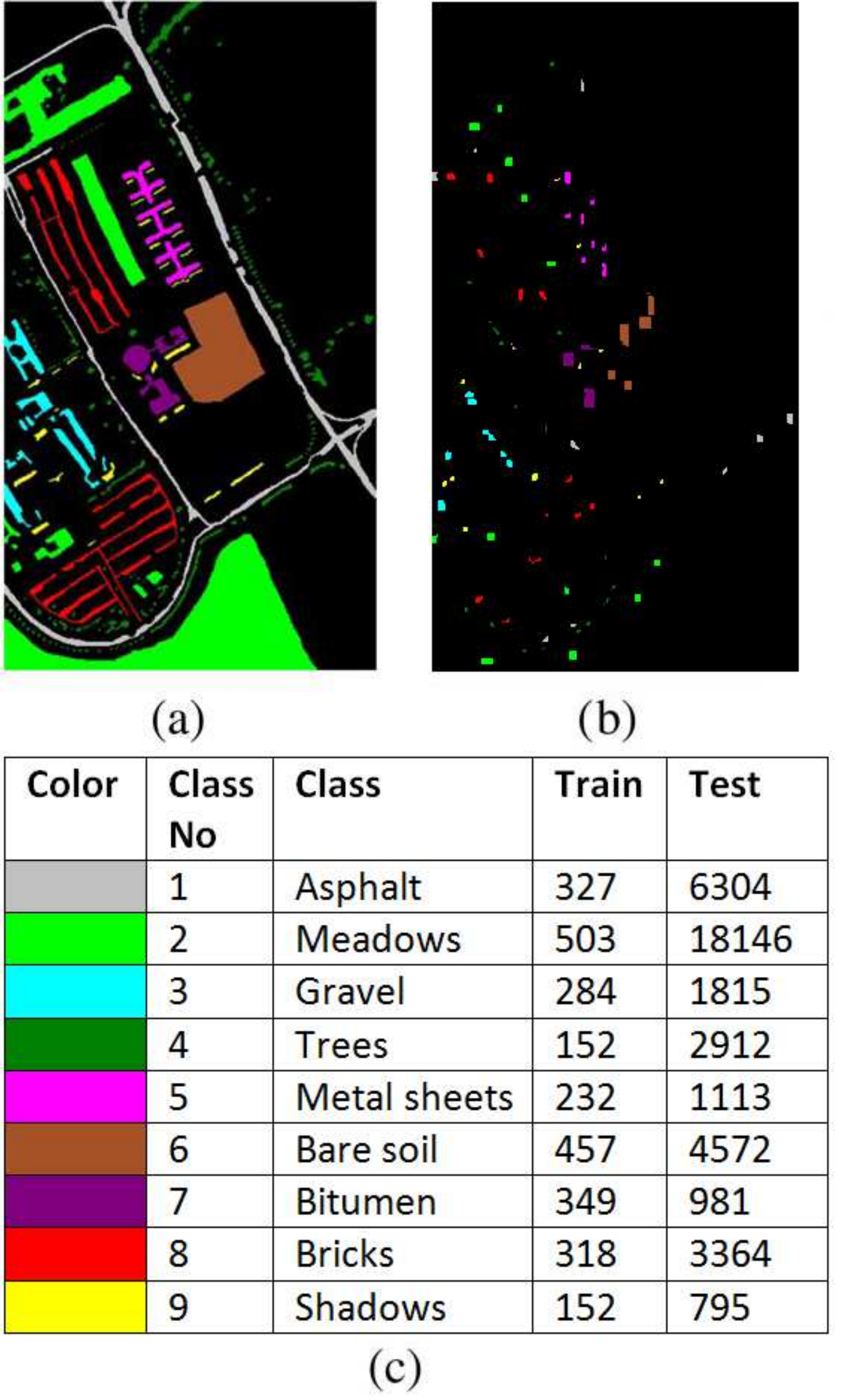}
 \centering
 \caption{ a-) Ground truth image b) fixed training samples, and c-) class numbers and the corresponding classes of Pavia University dataset }
 \label{paviaUniversityFixed}
 \end{figure*} 
 \begin{table*}
 \footnotesize
 \caption{The results (accuracy per class ($\%$), OA ($\%$), AA ($\%$), $\kappa$ ($\%$), Time (s) using fixed training set) for spatial-spectral methods on Pavia University dataset}
 \label{fixedResultPaviaUni}
 \centering
 \begin{tabular}{|c| c| c|| c| c| c | c|}
 \hline      
 Class  & Train & Test & SOMP \cite{chen2011hyperspectral} & L-MLL \cite{li2011hyperspectral} & SVM-GF \cite{kang2014spectral} & BTC-WLS \\ \hline
 1 & 327 & 6304 & 62.72 & 92.61 & \textbf{96.19} & 91.89 \\
 2 & 503 & 18146 & 71.49 & 70.80 & 79.00 & \textbf{86.75} \\
 3 & 284 & 1815 & 81.60 & 76.80 & 62.98 & \textbf{82.26} \\
 4 & 152 & 2912 & 89.66 & 91.79 & \textbf{96.81} & 92.41 \\
 5 & 232 & 1113 & \textbf{100.00} & 99.82 & \textbf{100.00} & \textbf{100.00} \\
 6 & 457 & 4572 & 94.03 & \textbf{99.02} & 98.60 & 98.43 \\
 7 & 349 & 981 & 95.51 & 98.37 & \textbf{100.00} & \textbf{100.00} \\
 8 & 318 & 3364 & 65.67 & 98.69 & 99.05 & \textbf{99.35} \\
 9 & 152 & 795 & 57.99 & \textbf{99.87} & 99.75 & 97.74 \\ \hline \hline
 OA & - & - & 75.09 & 83.67 & 87.71 & \textbf{91.07} \\
 AA & - & - & 79.85 & 91.98 & 92.48 & \textbf{94.31} \\
 $\kappa$ & - & - & 68.29 & 79.21 & 84.12 & \textbf{88.27} \\
 Time & -& - & 660.18 & \textbf{51.18} &  100.14 & 63.36 \\
 \hline
 \end{tabular}
 \end{table*}
 
 \section{Conclusions}
 \label{sec:conc}
 In this chapter, we proposed a light-weight sparsity-based classifier (BTC) for HSI classification alternative to the well-known SVM-based approaches and other greedy methods such as SOMP. The proposed method is easy to implement and performs quite fast. One of the most important advantages of the proposal is that it does not require any cross validation or classification experiment for parameter selection. Based on the guidance we have presented in chapter \ref{chp:btc}, one could easily select the threshold parameter once the dictionary is constructed. To improve the classification accuracy, we have proposed quite efficient framework in which the output residual maps of the pixel-wise classification procedure are smoothed using an edge preserving filtering method. Simulation results on the publicly available datasets showed that this intermediate procedure highly improves the classification accuracy. This approach could also be applied to any other sparsity-based classifier.

\chapter{HYPER-SPECTRAL IMAGE CLASSIFICATION via KBTC}
\label{chp:hyper2}

\section{Introduction}
The spectral signatures obtained by remote sensors could be used to distinguish the materials and objects on the surface of the earth. Those signatures are stored in the three dimensional data cubes called hyper-spectral images (HSI). While the first two dimensions of an image are used to describe the spatial coordinates, the last dimension is used to represent the spectral coordinate. The pixels of an image can be interpreted as the feature vectors containing spectral measurements. Various methods have been proposed in the literature to classify those pixels using the spectral measurements. Among those, support vector machines (SVM) \cite{cortes1995support} approach with radial basis function (RBF) kernel is commonly preferred. It is well known that the SVM classifier can achieve satisfactory results using only a few training samples \cite{melgani2004classification}, \cite{camps2005kernel}. It also has low sensitivity to Hughes phenomenon \cite{hughes1968mean}. On the other hand, it has some limitations such as conversion from binary classification to multi-class one and parameter determination via cross validation. When one-against-all approach is used as a conversion technique, some of the samples may be evaluated as unclassified \cite{mountrakis2011support}. If one-against-one approach is preferred, this time an increment in the number of classes significantly increases the binary classifiers used. Based on SVM, spatial-spectral techniques such as composite kernels \cite{camps2006composite}, segmentation maps \cite{tarabalka2009spectral}, \cite{tarabalka2010segmentation}, edge-preserving filtering \cite{kang2014spectral} have been developed to increase the classification accuracies by incorporating the spatial information. Although those approaches achieve promising results, they share common limitations with SVM.

Multinomial logistic regression (MLR)\cite{bohning1992multinomial}, which is based on Bayesian framework, has been used as an alternative to SVM in HSI classification \cite{borges2007evaluation}, \cite{li2012spectral}, and \cite{li2010semisupervised}. One of the MLR-based techniques which uses active learning namely logistic regression via splitting and augmented Lagrangian (LORSAL) has been proposed in \cite{bioucas2009logistic} and achieved promising results with the segmentation framework namely L-MLL (LORSAL multi-level logistic prior). It has the advantage of using active learning and RBF kernel. 

A powerful alternative to SVM is the sparse representation-based classification scheme which is based on the assumption that the test samples can be modeled as the sparse linear combinations of the basis dictionary elements. It was introduced in \cite{wright2009robust} and successfully applied to many classification problems. Unlike classical methods, this technique has low sensitivity to the corrupted or partial features exploiting the fact that the errors are sparse with respect to dictionary elements. Recently, joint sparsity-based methods (JSM) have been successfully applied to HSI classification. A greedy approach namely the simultaneous orthogonal matching pursuit (SOMP) \cite{tropp2006algorithms} was used for simultaneous classification of neighboring pixels in a given region based on the assumption that the pixels in a predetermined window share a common sparsity model \cite{chen2011hyperspectral}. Non-linear kernel versions of this approach have been proposed \cite{chen2013hyperspectral}, \cite{liu2013spatial}, and \cite{li2014column} to achieve better performance when the samples of different classes are linearly non-separable. To improve the classification accuracy, a weighting matrix based version of SOMP namely WSOMP has been proposed in \cite{zhang2014nonlocal}. An adaptive version of SOMP namely ASOMP has been developed in \cite{zou2014classification} to adaptively define the neighborhood pixels in the supervision of a segmentation map. Finally, a multi-scale adaptive sparse representation, which incorporates spatial information at different scales, has been successfully applied to HSI classification. However, JSM-based approaches suffer from the extensive computational cost because of simultaneous sparse recovery of the neighboring pixels. Another drawback is to select the sparsity level and the neighborhood window experimentally which may vary and cause non-optimal results.         

To eliminate the problems with the aforementioned approaches, basic thresholding classifier (BTC) was introduced in \ref{chp:btc}. It is a sparsity-based light-weight linear classifier which is able to achieve promising results and classify test samples extremely rapidly as compared to the other sparsity-based algorithms. In this chapter, we propose the kernelized version of the BTC algorithm namely kernel basic thresholding classifier (KBTC) which is capable of classifying the samples of classes of a given dataset when they are linearly non-separable \cite{toksoz2016kbtc}. Three major contributions of this chapter are follows:
\begin{itemize}
\item We propose a non-linear kernel version of the basic thresholding classification algorithm namely KBTC for HSI classification. The proposed non-linear sparsity-based method can achieve higher classification results as compared to the linear sparsity based ones especially when the classes of a given dataset are linearly non-separable.   
\item Unlike SVM, in which the parameters are obtained via cross validation, we provide a full parameter selection guidance by which the kernel and threshold parameters can easily be estimated without using any experiment and cross validation.
\item We extend the proposal to a spatial-spectral framework in which the final classification is performed based on the smoothed residual maps. We also provide more realistic fixed train sets alternative to those selected randomly from a given dataset.
\end{itemize} 

\section{HSI Classification}
In this section, we will briefly describe the HSI classification problem in the sparse representation model. Let $A_i\in \mathbb{R}^{B\times N_i}$ be the matrix containing $B$ dimensional $N_i$ many training pixels which are belonging to the $i$th class. Then, with $C$ many classes one could construct the dictionary $A$ in such a way that $A=[A_1~A_2~\ldots~A_C]\in \mathbb{R}^{B\times N}$ where $N=\sum_{i=1}^{C}N_i$. In this context, a test pixel $y \in\mathbb{R}^B$ can be classified as follows: Sparse representation model states that there exists a sparse code $x\in\mathbb{R}^{N}$ having minimum $l_1$ norm such that $y=Ax$ or $y=Ax+\epsilon$ where $\epsilon$ is a small error \cite{wright2009robust}. The minimization problem is equivalent to 
\begin{equation}\label{sparse_hyper2}
\hat{x}=\arg\min_x \norm{x}_1~subject~to~y=Ax
\end{equation}
or subject to $\norm{y-A x }_2 \leq \epsilon$. Then, we could find the class of $y$ using the following expression
\begin{equation}\label{sparse_class_hyper2}
 class(y)=\arg \min_i \norm{y-A \hat{x_i}}_2 \; \forall i \in \{1,2,\ldots,C\}
\end{equation}
where $\hat{x_i}$ denotes the $i$th class portion of the estimated sparse code vector $\hat{x}$. The minimization problem (\ref{sparse_hyper2}) could be solved using convex relaxation-based techniques (homotopy \cite{yang2010fast}) or alternatively greedy approaches such as OMP \cite{tropp2006algorithms}. However, those techniques and their variants \cite{chen2011hyperspectral} suffer from high computational cost. To eliminate the problems with those techniques, a simple and light-weight sparsity-based method, basic thresholding classifier (BTC) was introduced in chapter \ref{chp:btc}. Unfortunately, this approach works based on the fact that the classes of a given HSI are linearly separable. In practice, the situation may not be like this. Therefore, in order to eliminate the limitations with the aforementioned algorithms, we propose KBTC algorithm for HSI classification. In this context, one can use the following expression in order to classify a given pixel:

\begin{equation}\label{kbtc_hyper1}
 Class(y) \gets KBTC(A, y, \gamma, M, \alpha)
\end{equation}

The details of the algorithm could be found in Chapter \ref{chp:kbtc}.
\section{Extension to Spatial-Spectral KBTC}
Since incorporating spatial information highly improves the classification accuracy, we extend our pixel-wise proposal (KBTC) to a spatial-spectral framework namely KBTC-WLS. Notice that when a test pixel $y$ is classified, KBTC produces not only the class identity of $y$ but also a residual vector $\epsilon \in\mathbb{R}^{C}$ which contains distances to each class. Suppose that we classified every pixel of a given HSI, $H \in\mathbb{R}^{n_1 \times n_2 \times B}$ consisting of $n_1 \times n_2$ pixels, using the KBTC method. We could interpret the resulting residual vectors as a residual cube $R \in\mathbb{R}^{n_1 \times n_2 \times C}$ in which each layer represents a residual map. The idea is to smooth those maps using an averaging filter and then determine the class of each pixel based on minimal smoothed residuals. This intermediate step highly improves the classification results. We use weighted least squares method (WLS) \cite{farbman2008edge} as the smoothing filter which successfully preserves the edges via a gray scale guidance image. We prefer this filter because it does not cause halo artifacts at the edges. The gray scale guidance image could be obtained via principal component analysis (PCA) technique by reducing the dimensions of the original HSI from $B$ to 1. We give the overall framework step-by-step as follows:

\begin{itemize}
\item The parameters $\gamma$ and $M$ are estimated for a given dictionary $A$ using the procedures provided in chapter \ref{chp:kbtc}.
\item Every pixel of a given HSI is classified via KBTC. Note that the entries of the resulting residual cube are normalized between $0$ and $1$.
\item We obtain the gray scale guidance image by reducing the dimension of the original HSI from $B$ to 1 using the PCA technique.
\item We apply WLS filtering to each residual map by means of the guidance image obtained in the previous step. In order to reduce the computation cost of the WLS filtering, preconditioned conjugate gradient (PCG) method with incomplete Cholesky decomposition \cite{saad2003iterative} could be used when taking the inverses of the large sparse matrices.
\item The class of each pixel is determined based on minimal smoothed residuals.   
\end{itemize}    
The overall framework could also be seen in Fig. \ref{SKBTC}. Please note that since we obtain the pixel-wise class map before the smoothing procedure, we could set the error values to the maximum value 1 in the $i$th residual map for the entries whose labels are not equal to $i$ using the pixel-wise class map. This improves the classification performance further.

\begin{figure*}
\centering
\includegraphics[width = 1.0\textwidth]{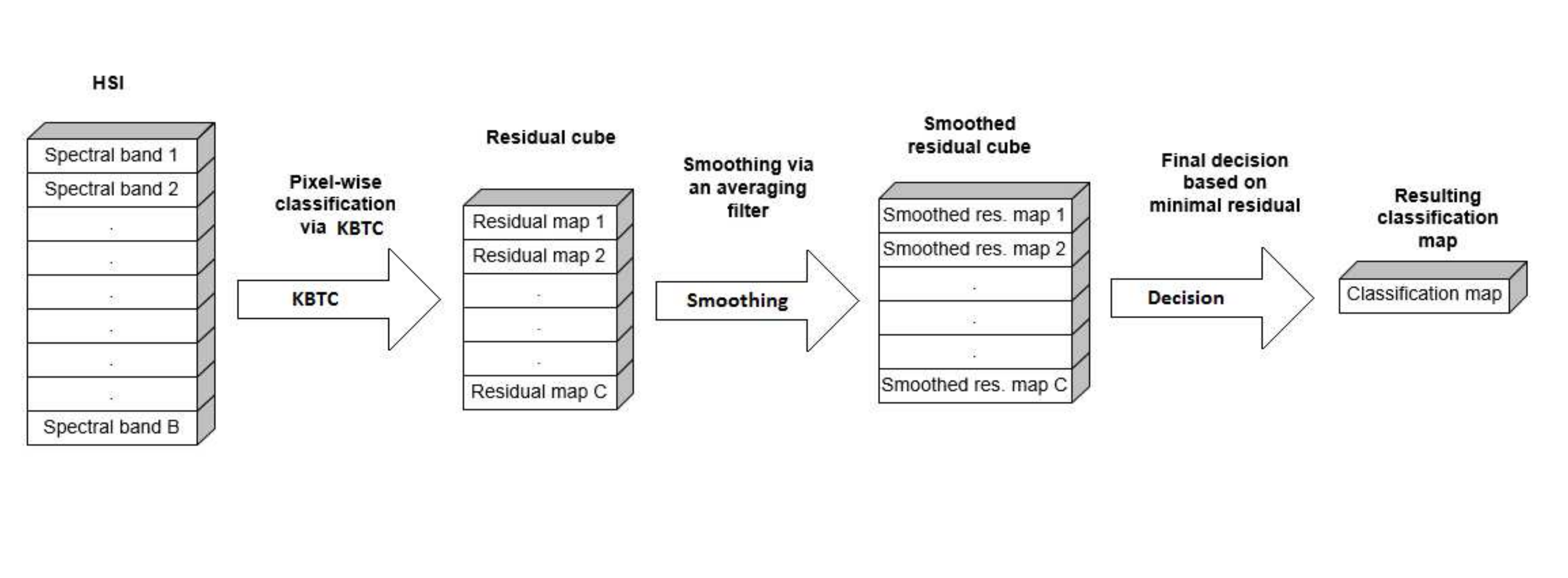}
\centering
\caption{Spatial-Spectral KBTC (KBTC-WLS)}
\label{SKBTC}
\end{figure*}

\section{Experimental Results}
\label{sec:exp_hyper2}

\subsection{Scaling}
Scaling both the testing and training data is quite important before applying the KBTC algorithm. We recommend scaling the entries in the feature vectors to the range $[-1, +1]$ or $[0, 1]$. Suppose that we scaled a training pixel from $[100, 20, 50, ...]^T$ to $[1, 0.2, 0.5, ...]^T$. If a test pixel having the feature vector $[110, 25, 45, ...]^T$ is needed to be classified, then it should be scaled to $[1.1, 0.25, 0.45, ...]^T$. It is similar to the scaling in RBF kernel SVM \cite{hsu2003practical}. Please note that linear BTC uses different approach in which only the training pixel vectors are $l_2$ normalized. 

\subsection{Datasets}
We performed the experiments using three publicly available HSI datasets namely Indian Pines, Salinas, and Pavia University. Detailed description of each dataset is given Table \ref{description_hyper2}. The original Indian Pines image contains $16$ different classes. However, we discarded $7$ of them because of insufficient number of training and testing samples \cite{plaza2009recent}. We carried out the experiments using fixed and grouped training pixels instead of random selection. This is because in case of random selection, it is highly likely to have a training sample which may be closely related to a testing sample. Therefore, this type of experiment may not present realistic results. On the other hand, in our case, we exclude training pixels from the testing areas which is quite similar to real world scenarios. The experiments using fixed and grouped training pixels could also be seen in \cite{li2013spectral},\cite{chen2011hyperspectral}, and \cite{chen2013hyperspectral} for Pavia University dataset. In our experiments, we used $27$ training pixels for each class of Indian Pines and Pavia University datasets. Those samples were taken from $3$ distinct $3\times3$ blocks. For Salinas dataset we used $32$ training pixels for each class. This time the samples were taken from $2$ distinct $4\times4$ blocks. The ground truth image, each class with corresponding number of training and testing pixels, and the starting coordinates of the distinct blocks from which the fixed training pixels were taken are given for each dataset in Fig. \ref{indianPines_hyper2}, Fig. \ref{salinas_hyper2}, and Fig. \ref{paviaUniversity_hyper2}, respectively. We provided the coordinates of the blocks because anyone can easily repeat the experiments and compare the results.   
 
\begin{table*}
\footnotesize
\caption{Description of each dataset}
\label{description_hyper2}
\centering
\begin{tabular}{|c |c | c| c| c| c| c|}
\hline
 Dataset & Size & Spatial  & Spectral  & Num. of & Sensor & Num. of  \\ 
         &      & resolution & coverage & classes &  & bands \\ \hline
Indian Pines &	145 $\times$  145 $\times$  220 &	20 m &	0.4 $-$ 2.5 $\mu$m &	9 &	AVIRIS &	200 \\ \hline
Salinas	& 512 $\times$  217$\times$  224 &	3.7 m &	0.4 $-$ 2.5 $\mu$m &	16 &	AVIRIS &	204 \\ \hline
Pavia University &	610 $\times$ 340 $\times$  115 &	1.5 m &	0.43 $-$ 0.86 $\mu$m &	9 &	ROSIS &	103 \\
\hline
\end{tabular}
\end{table*}

\begin{figure*}
\centering
\includegraphics[width = 1.0\textwidth]{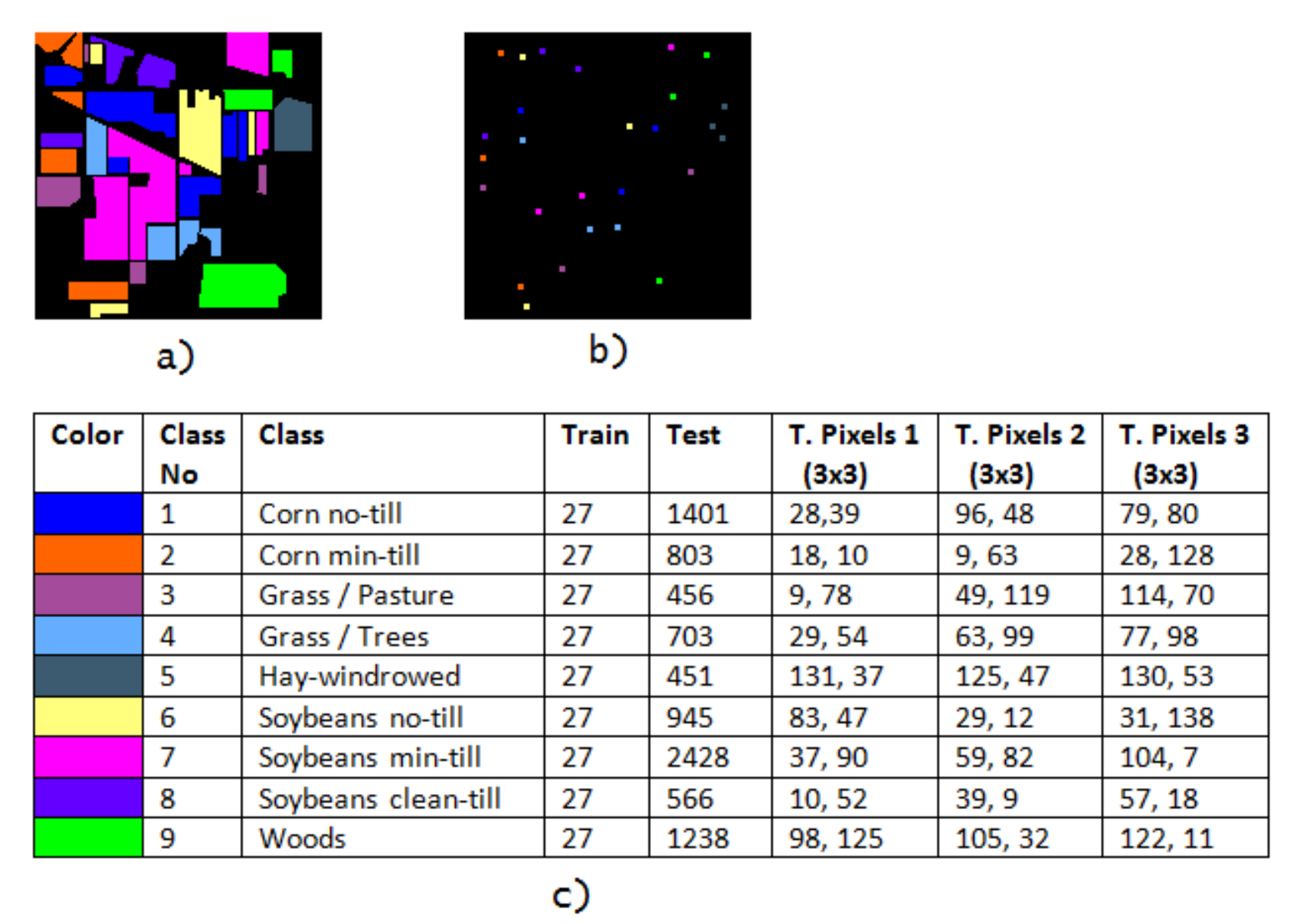}
\centering
\caption{ a-) Ground truth image, b-) fixed training pixels, c-) each class with corresponding number of training and testing pixels, and the coordinates of the grouped training pixels for Indian Pines dataset }
\label{indianPines_hyper2}
\end{figure*} 

\begin{figure*}
\centering
\includegraphics[width = 1.0\textwidth]{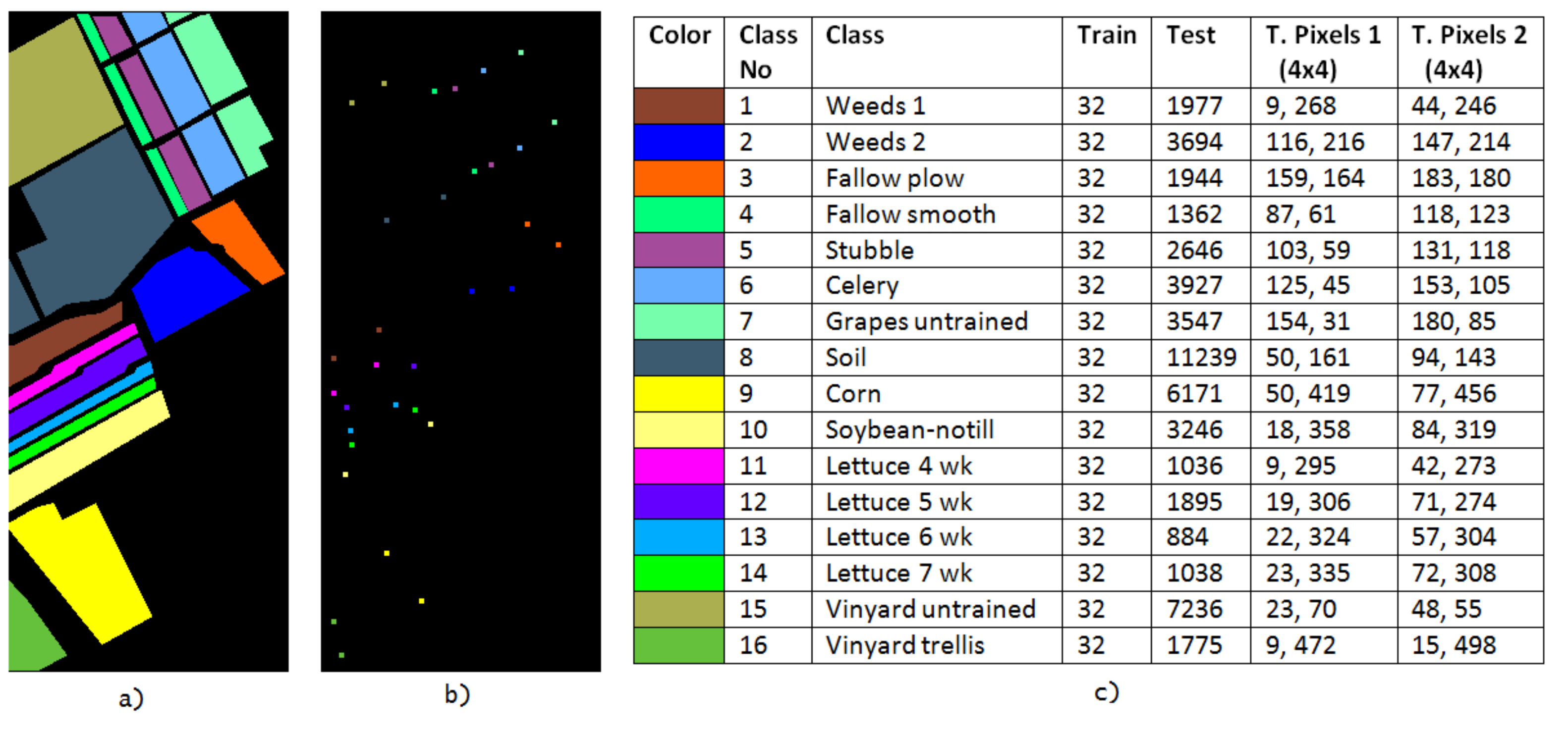}
\centering
\caption{  a-) Ground truth image, b-) fixed training pixels, c-) each class with corresponding number of training and test pixels, and the coordinates of the grouped training pixels for Salinas dataset }
\label{salinas_hyper2}
\end{figure*} 

\begin{figure*}
\centering
\includegraphics[width = 1.0\textwidth]{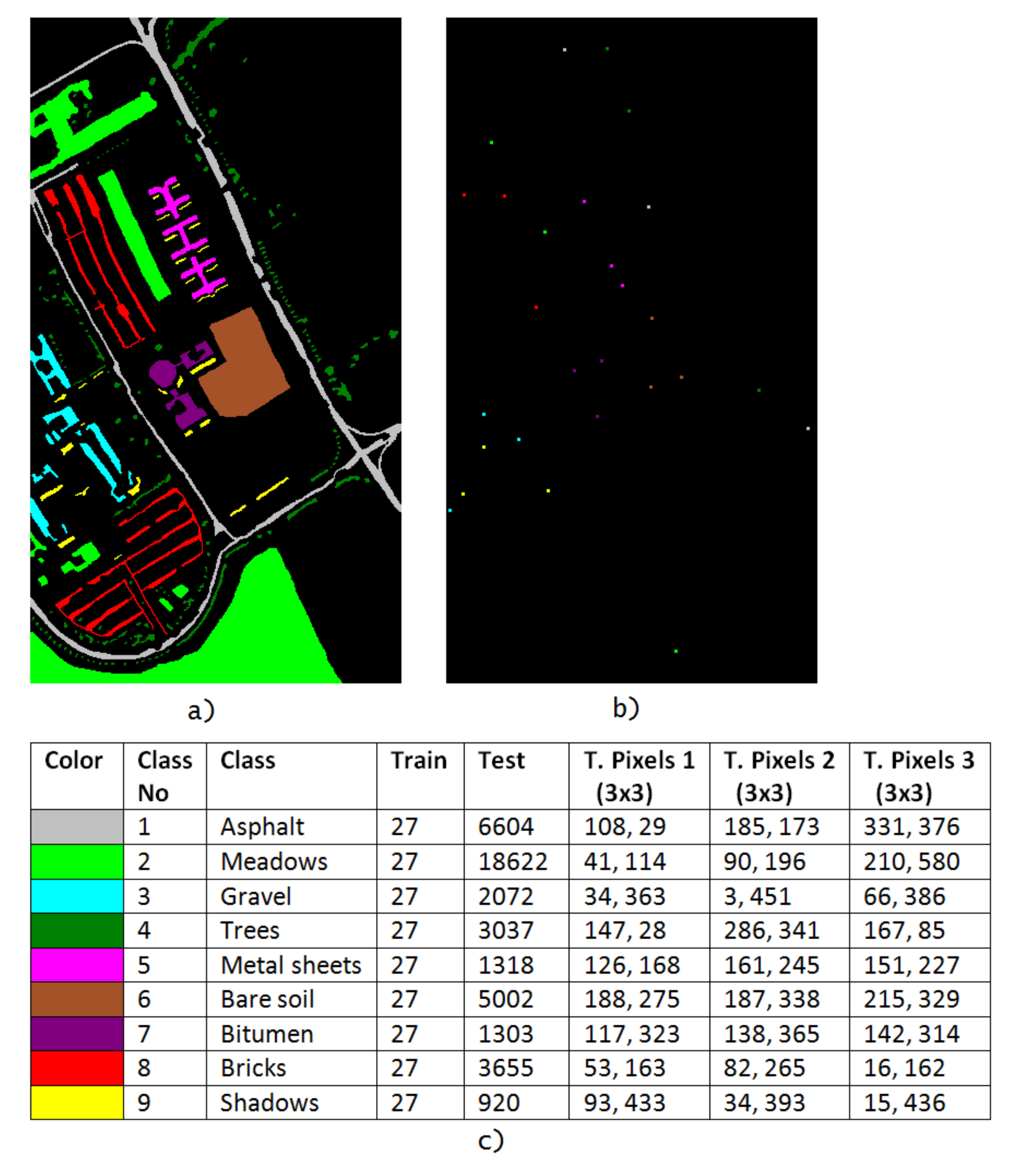}
\centering
\caption{  a-) Ground truth image, b-) fixed training pixels, c-) each class with corresponding number of training and test pixels, and the coordinates of the grouped training pixels for Pavia University dataset }
\label{paviaUniversity_hyper2}
\end{figure*} 

\subsection{Experimental Setup}
Before the experiments, we estimated the $\gamma$ and $M$ parameters of KBTC using the training sets and the procedures provided in the previous section. For each dataset, we computed the values of $\overline{\beta}(\gamma)$ by varying $\gamma$ from $2^{-10}$ to $2^1$. The results could be seen in Table \ref{gamma_estimate_table}. If we keep track of the computed values of $\overline{\beta}(\gamma)$ for any dataset, we observe that the resulting function is strictly convex. This implies that $\overline{\beta}(\gamma)$ has unique minimum. Therefore, the best value of $\gamma$ in the described sense could easily be estimated. The estimated $\gamma$ values which minimize $\overline{\beta}(\gamma)$ are $2^{-6}$, $2^{-6}$, $2^{-1}$ for the dictionaries constructed using Indian Pines and Salinas and Pavia University datasets, respectively.  If we insert the estimated $\gamma$ values to $\overline{\beta}(\hat{\gamma}, M)$ function, then we could estimate the best value of the threshold by varying $M$ from $1$ to $B-1$. The $M$ value that minimizes $\overline{\beta}(\hat{\gamma}, M)$ is considered to be the best estimate of it in the described sense. We performed this procedure and obtained $\overline{\beta}(\hat{\gamma}, M)$ plots in Fig. \ref{avgSicValues_hyper2} for each dataset used in this chapter. The estimated threshold values are $95$, $92$, and $35$ for Indian Pines, Salinas, and Pavia University datasets, respectively.
\begin{table*}
\footnotesize
\caption{$\overline{\beta}(\gamma)$ values for each $\gamma$ and dataset}
\label{gamma_estimate_table}
\centering
\begin{tabular}{|c |c | c| c|}
\cline{2-4}
\multicolumn{1}{ c }{}& \multicolumn{3}{ |c| }{$\overline{\beta}(\gamma)$} \\ \hline
$\gamma$ & Indian Pines & Salinas & Pavia University \\ \hline
$2^{1}$ & 0.9314 & 0.3322 & 0.4612 \\ \hline
$2^{0}$ & 0.8250 & 0.2410 & 0.4263 \\ \hline
$2^{-1}$ & 0.6950 & 0.1836 & \textbf{0.4176} \\ \hline
$2^{-2}$ & 0.5719 & 0.1499 & 0.4265 \\ \hline
$2^{-3}$ & 0.4751 & 0.1302 & 0.4549 \\ \hline
$2^{-4}$ & 0.4093 & 0.1183 & 0.5025 \\ \hline
$2^{-5}$ & 0.3743 & 0.1123 & 0.5712 \\ \hline
$2^{-6}$ & \textbf{0.3644} & \textbf{0.1106} & 0.6630 \\ \hline
$2^{-7}$ & 0.3737 & 0.1127 & 0.7669 \\ \hline
$2^{-8}$ & 0.3949 & 0.1176 & 0.8925 \\ \hline
$2^{-9}$ & 0.4263 & 0.1251 & 1.0429 \\ \hline
$2^{-10}$ & 0.4681 & 0.1355 & 1.2670 \\ 
\hline
\end{tabular}
\end{table*}

\begin{figure*}
\centering
\includegraphics[width = 0.9\textwidth]{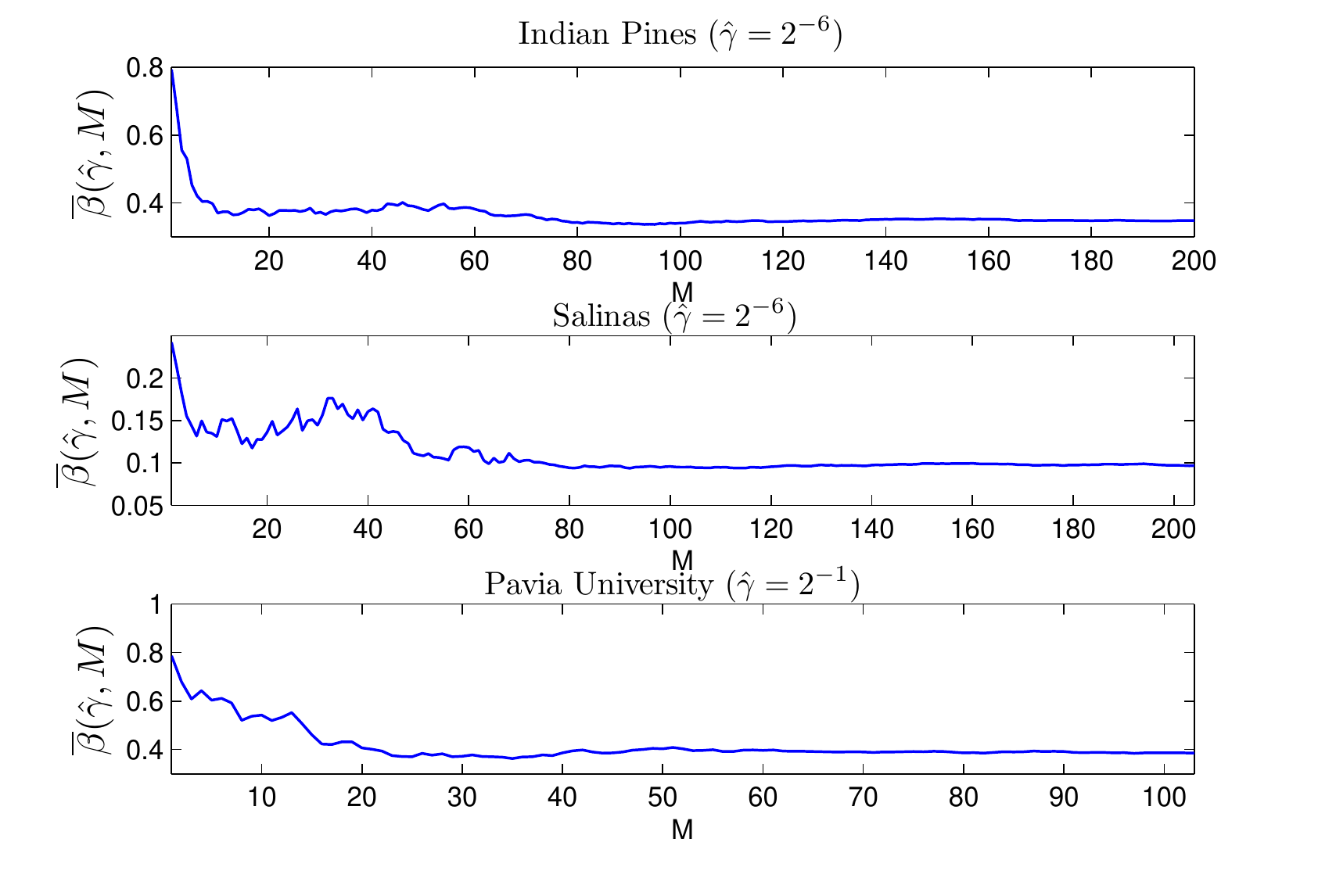}
\caption{$\overline{\beta}(\hat{\gamma}, M)$ vs threshold values for Indian Pines, Salinas, and Pavia University}
\label{avgSicValues_hyper2}
\end{figure*} 

We included the spectral-only classifiers as well as the spatial-spectral ones in the experiments. RBF kernel SVM \cite{melgani2004classification} could be considered one of the strongest spectral-only classifiers in the literature. Another approach in this category is the LORSAL \cite{li2011hyperspectral} technique which also uses the RBF kernel. It has the advantage of using active learning. Although it is not fair, we included the linear similarity-based BTC in order to show how the non-linear KBTC algorithm is superior to it. For SVM method, we used LIBSVM library \cite{chang2011libsvm} and we selected the parameters ($C$, $\gamma$) of it via 5-fold cross validation by varying $C$ from $10^{-5}$ to $10^5$ and $\gamma$ from $2^{-5}$ to $2^5$. The spatial extensions of those algorithms highly improve the results obtained in the pixel-wise classification stage. Those techniques used in this chapter are SVM-GF \cite{kang2014spectral} (based on guided filter \cite{he2010guided}), L-MLL (based on LORSAL), BTC-WLS, and KBTC-WLS. In order for fair comparison we used SVM-WLS instead of SVM-GF since the WLS filter-based techniques achieve better results. WLS filter has two parameters namely the smoothing ($\lambda$) and the sharpening ($\alpha$) degrees. We set those parameters to $0.4$ and $0.9$ for all experiments, respectively. For LORSAL and L-MLL techniques, we set the initial training samples to the half of the all available training pixels, and during the learning stage we incremented the samples by 20 using random selection (RS). Under this configuration we repeated the experiments for all datasets in a PC having a quad-core 3.60 GHz processor and 16GB of memory.

\subsection{Performance Indexes}
The performance indexes used in this work are as follows:
\begin{itemize}
\item Overall Accuracy (OA): It is the percentage of correctly classified pixels among the whole test samples.
\item Average Accuracy (AA): It shows the mean of individual class accuracies. 
\item The $\kappa$ coefficient: It measures the degree of consistency \cite{cohen1960coefficient}.
\item Computation time: It is used to measure the computational complexity of an algorithm. It also determines if an algorithm is suitable for real time applications. 
\end{itemize}

\subsection{Classification Results}
Using the training and testing samples given in Fig. \ref{indianPines_hyper2}, we carried out the first experiment on Indian Pines dataset both for spectral-only and spatial-spectral approaches. The classification results are shown in Table \ref{resultIndianPines_hyper2}. We also provided the corresponding classification maps with OAs(\%) in Fig. \ref{indianClassMap_hyper2}. Both in the spectral-only and spatial-spectral cases, KBTC achieves best results in terms of all metrics except the computation time. The performance differences between the KBTC and the other algorithms are significant. In the first case, KBTC performs about $5\%$ better than LORSAL technique in terms of overall accuracy. It also improves the result of linear BTC approximately $7.5\%$. LORSAL technique performs about $1\%$ better than the SVM method using the advantage of active learning. In the latter case, BTC-WLS achieves promising results by means of smoothing the residual maps. KBTC-WLS exploits the same technique and outperforms the BTC-WLS approach by achieving about $3\%$ better in terms of OA. Although the LORSAL approach performs well in the spectral-only case, L-MLL technique, which is based on LORSAL, performs worse than the other approaches in the latter case. This experiment shows that KBTC significantly improves the performance of linear similarity-based BTC both in spectral-only and spatial-spectral cases. In terms of computation time, LORSAL and its spatial extension are the fastest ones. Although KBTC and KBTC-WLS are the slowest ones, there is no significant difference between the computation times of these methods and those of the other techniques except LORSAL and L-MLL.  

\begin{table*}
\footnotesize
\caption{The results (accuracy per class ($\%$), OA ($\%$), AA ($\%$), $\kappa$ ($\%$), Time (s) using fixed training set) for spectral-only and spatial-spectral methods on Indian Pines dataset}
\label{resultIndianPines_hyper2}
\centering
\begin{tabular}{|c|| c| c| c | c|| c| c| c| c|}
\cline{2-9}
\multicolumn{1}{ c }{}& \multicolumn{4}{ |c|| }{Spectral-Only} & \multicolumn{4}{ |c| }{Spatial-Spectral} \\ \hline
Class No           & SVM & LORSAL & BTC & KBTC & SVM-WLS & L-MLL & BTC-WLS & KBTC-WLS \\ \hline 
1 & 52.53 & 51.46 & 47.47 & \textbf{64.03} & 77.02 & 67.81 & 80.66 & \textbf{81.87}\\
2 & 49.32 & 45.08 & \textbf{56.04} & \textbf{56.04} & 68.99 & 51.93 & \textbf{82.81} & 81.32\\
3 & \textbf{89.47} & 73.68 & 88.82 & 88.82 & \textbf{96.05} & 84.65 & \textbf{96.05} & \textbf{96.05}\\
4 & 92.75 & 95.73 & 90.75 & \textbf{98.58} & 99.86 & 99.57 & \textbf{100.00} & \textbf{100.00}\\
5 & 98.45 & 98.45 & 95.79 & \textbf{99.56} & \textbf{100.00} & \textbf{100.00} & \textbf{100.00} & \textbf{100.00}\\
6 & 65.50 & 56.40 & 44.66 & \textbf{69.95} & 89.63 & 73.54 & 70.05 & \textbf{99.26}\\
7 & 49.63 & 57.50 & \textbf{58.81} & 54.57 & 68.57 & 65.53 & \textbf{89.17} & 83.61\\
8 & 56.36 & 61.13 & 37.81 & \textbf{65.90} & 85.87 & 81.27 & 76.15 & \textbf{93.11}\\
9 & 88.93 & 92.57 & 86.03 & \textbf{92.65} & \textbf{100.00} & 97.25 & 99.19 & \textbf{100.00}\\ \hline
OA & 65.39 & 66.25 & 63.60 & \textbf{71.18} & 82.97 & 76.23 & 87.56 & \textbf{90.36}\\
AA & 71.43 & 70.22 & 67.35 & \textbf{76.67} & 87.33 & 80.17 & 88.23 & \textbf{92.80}\\
$\kappa$ & 60.15 & 60.76 & 57.53 & \textbf{66.68} & 80.27 & 72.32 & 85.26 & \textbf{88.73} \\
Time & 1.89 & \textbf{0.18} &1.78 & 2.90 & 2.44 &\textbf{0.41 } & 2.33 & 3.47 \\ \hline
\end{tabular}
\end{table*}
\begin{figure*}
\centering
\includegraphics[width = 1.0\textwidth]{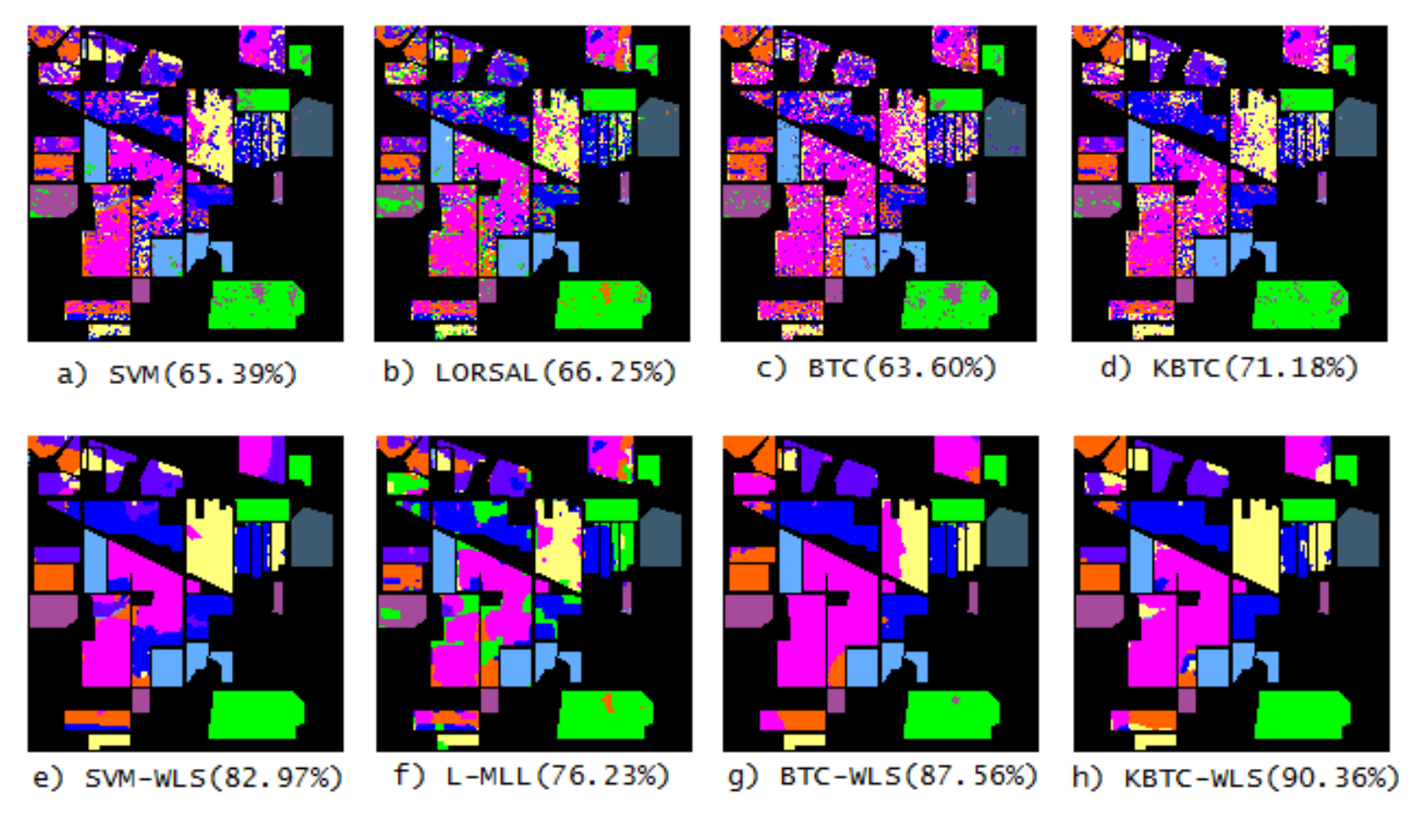}
\centering
\caption{ Classification Maps on Indian Pines Dataset with Overall Accuracies}
\label{indianClassMap_hyper2}
\end{figure*}

We performed the second experiment on the Salinas dataset using the training and testing samples given in Fig. \ref{salinas_hyper2}. The classification results could be seen in Table \ref{resultSalinas_hyper2}. We also included the corresponding classification maps with OAs(\%) in Fig. \ref{indianClassMap_hyper2}. As we can see in Fig. \ref{salinas_hyper2}, this HSI consists of large homogeneous areas as compared to the previous HSI. Therefore, the classification results achieved by the spectral-only techniques are closer to each other. The OA achieved by KBTC is about $1\%$ better than that of BTC approach. This result also shows that the dictionary constructed for Salinas image is more linearly separable than the previous one. KBTC also outperforms the SVM and LORSAL techniques having approximately $3\%$ and $1.5\%$ better accuracies, respectively. In the spatial-spectral case, the accuracies are also closer to each other. In this case, again KBTC-WLS achieves the best results in terms of all metrics except the computation time.

\begin{table*}
\footnotesize
\caption{The results (accuracy per class ($\%$), OA ($\%$), AA ($\%$), $\kappa$ ($\%$), Time (s) using fixed training set) for spectral-only and spatial-spectral methods on Salinas dataset}
\label{resultSalinas_hyper2}
\centering
\begin{tabular}{|c|| c| c| c | c|| c| c| c| c|}
\cline{2-9}
\multicolumn{1}{ c }{}& \multicolumn{4}{ |c|| }{Spectral-Only} & \multicolumn{4}{ |c| }{Spatial-Spectral} \\ \hline
Class No           & SVM & LORSAL & BTC & KBTC & SVM-WLS & L-MLL & BTC-WLS & KBTC-WLS \\ \hline 
1 & 96.56 & 95.45 & 97.02 & \textbf{97.42} & \textbf{100.00} & 96.61 & \textbf{100.00} & \textbf{100.00}\\
2 & \textbf{99.76} & 99.38 & 98.51 & 99.46 & \textbf{100.00} & 99.73 & \textbf{100.00} & \textbf{100.00}\\
3 & 73.15 & 86.37 & \textbf{87.65} & 82.46 & 95.27 & 97.43 & \textbf{100.00} & \textbf{100.00}\\
4 & 98.53 & 91.48 & \textbf{98.83} & 97.21 & \textbf{100.00} & 91.48 & \textbf{100.00} & \textbf{100.00}\\
5 & 97.73 & \textbf{98.11} & 95.39 & 97.73 & 99.36 & 98.37 & \textbf{99.77} & 99.51\\
6 & 96.79 & 95.80 & \textbf{99.41} & 98.55 & 99.92 & 96.87 & \textbf{100.00} & 99.92\\
7 & 97.86 & 96.48 & 97.60 & \textbf{99.10} & \textbf{100.00} & 97.07 & \textbf{100.00} & \textbf{100.00}\\
8 & 73.28 & 73.25 & 75.18 & \textbf{75.35} & 93.28 & 86.24 & 88.81 & \textbf{93.44}\\
9 & 97.18 & 97.21 & 96.29 & \textbf{98.31} & \textbf{100.00} & 97.50 & \textbf{100.00} & \textbf{100.00}\\
10 & 85.77 & 77.85 & 85.37 & \textbf{90.48} & \textbf{99.14} & 80.04 & 96.03 & \textbf{99.14}\\
11 & 96.81 & 96.43 & \textbf{99.52} & 98.75 & \textbf{100.00} & 96.72 & \textbf{100.00} & \textbf{100.00}\\
12 & \textbf{99.63} & 94.67 & 99.47 & 94.72 & \textbf{100.00} & 95.62 & \textbf{100.00} & \textbf{100.00}\\
13 & 97.96 & \textbf{98.98} & 97.85 & 98.30 & 99.66 & 98.98 & 99.21 & \textbf{99.89}\\
14 & 90.46 & 89.40 & \textbf{93.26} & 90.37 & 98.55 & 92.20 & 99.42 & \textbf{99.52}\\
15 & 59.72 & \textbf{74.47} & 69.79 & 69.00 & 81.45 & \textbf{88.82} & 78.15 & 83.32\\
16 & 77.07 & 82.48 & 78.87 & \textbf{94.20} & 92.90 & 91.15 & 92.28 & \textbf{100.00}\\ \hline
OA & 85.08 & 86.67 & 87.39 & \textbf{88.14} & 95.55 & 92.47 & 94.17 & \textbf{96.28}\\
AA & 89.89 & 90.48 & 91.87 & \textbf{92.58} & 97.47 & 94.05 & 97.10 & \textbf{98.42}\\
$\kappa$ & 83.38 & 85.18 & 85.98 & \textbf{86.81} & 95.04 & 91.62 & 93.50 & \textbf{95.85} \\
Time &  8.66& \textbf{1.94} & 6.69& 17.44 & 12.98 &\textbf{3.84 } & 10.93 & 21.71 \\ \hline
\end{tabular}
\end{table*}
\begin{figure*}
\centering
\includegraphics[width = 1.0\textwidth]{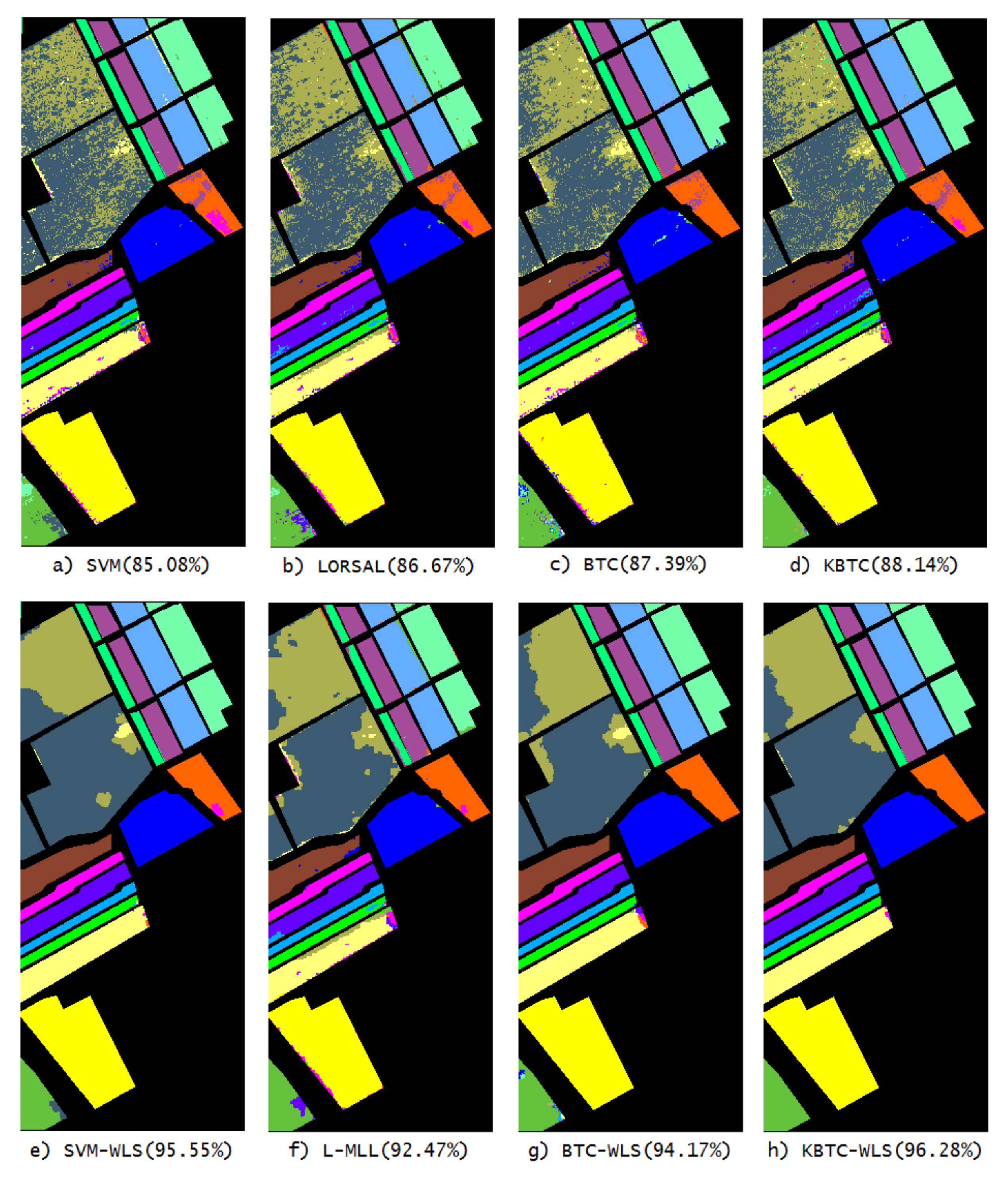}
\centering
\caption{ Classification Maps on Salinas Dataset with Overall Accuracies}
\label{salinasClassMap_hyper2}
\end{figure*}

\begin{table*}
\footnotesize
\caption{The results (accuracy per class ($\%$), OA ($\%$), AA ($\%$), $\kappa$ ($\%$), Time (s) using fixed training set) for spectral-only and spatial-spectral methods on Pavia University dataset}
\label{resultPaviaUni_hyper2}
\centering
\begin{tabular}{|c|| c| c| c | c|| c| c| c| c|}
\cline{2-9}
\multicolumn{1}{ c }{}& \multicolumn{4}{ |c|| }{Spectral-Only} & \multicolumn{4}{ |c| }{Spatial-Spectral} \\ \hline
Class No           & SVM & LORSAL & BTC & KBTC & SVM-WLS & L-MLL & BTC-WLS & KBTC-WLS \\ \hline 
1 & 74.11 & 75.44 & 66.81 & \textbf{75.79} & 89.58 & 89.52 & \textbf{90.48} & 90.40\\
2 & 63.44 & 67.82 & 62.03 & \textbf{73.87} & 68.59 & 74.21 & 70.98 & \textbf{82.92}\\
3 & 69.35 & 63.08 & 66.17 & \textbf{75.63} & 89.48 & 63.75 & \textbf{97.06} & 90.64\\
4 & 97.07 & 90.12 & 94.80 & \textbf{97.14} & 93.97 & 89.96 & \textbf{94.01} & 93.71\\
5 & 86.12 & 93.17 & \textbf{99.09} & 94.01 & 99.24 & 96.97 & \textbf{100.00} & \textbf{100.00}\\
6 & 69.07 & 74.29 & 65.23 & \textbf{78.83} & 77.75 & 80.69 & 85.67 & \textbf{94.86}\\
7 & 85.96 & 89.10 & 82.27 & \textbf{93.17} & \textbf{100.00} & 94.32 & \textbf{100.00} & \textbf{100.00}\\
8 & \textbf{79.48} & 83.31 & 46.62 & 75.73 & 90.78 & \textbf{91.57} & 76.17 & 88.59\\
9 & 99.78 & 95.00 & 81.85 & \textbf{100.00} & \textbf{99.13} & 93.15 & 90.54 & 98.70\\\hline
OA & 72.01 & 74.48 & 66.56 & \textbf{78.43} & 80.23 & 81.18 & 81.30 & \textbf{88.51}\\
AA & 80.48 & 81.25 & 73.87 & \textbf{84.90} & 89.83 & 86.01 & 89.43 & \textbf{93.31}\\
$\kappa$ & 64.92 & 67.83 & 58.22 & \textbf{72.67} & 75.06 & 75.98 & 76.40 & \textbf{85.26} \\
Time & 6.56  & \textbf{1.51} & 6.49 & 7.85 & 11.36 &\textbf{ 3.80} & 11.19 & 12.58 \\ \hline
\end{tabular}
\end{table*}

\begin{figure*}
\centering
\includegraphics[width = 1.0\textwidth]{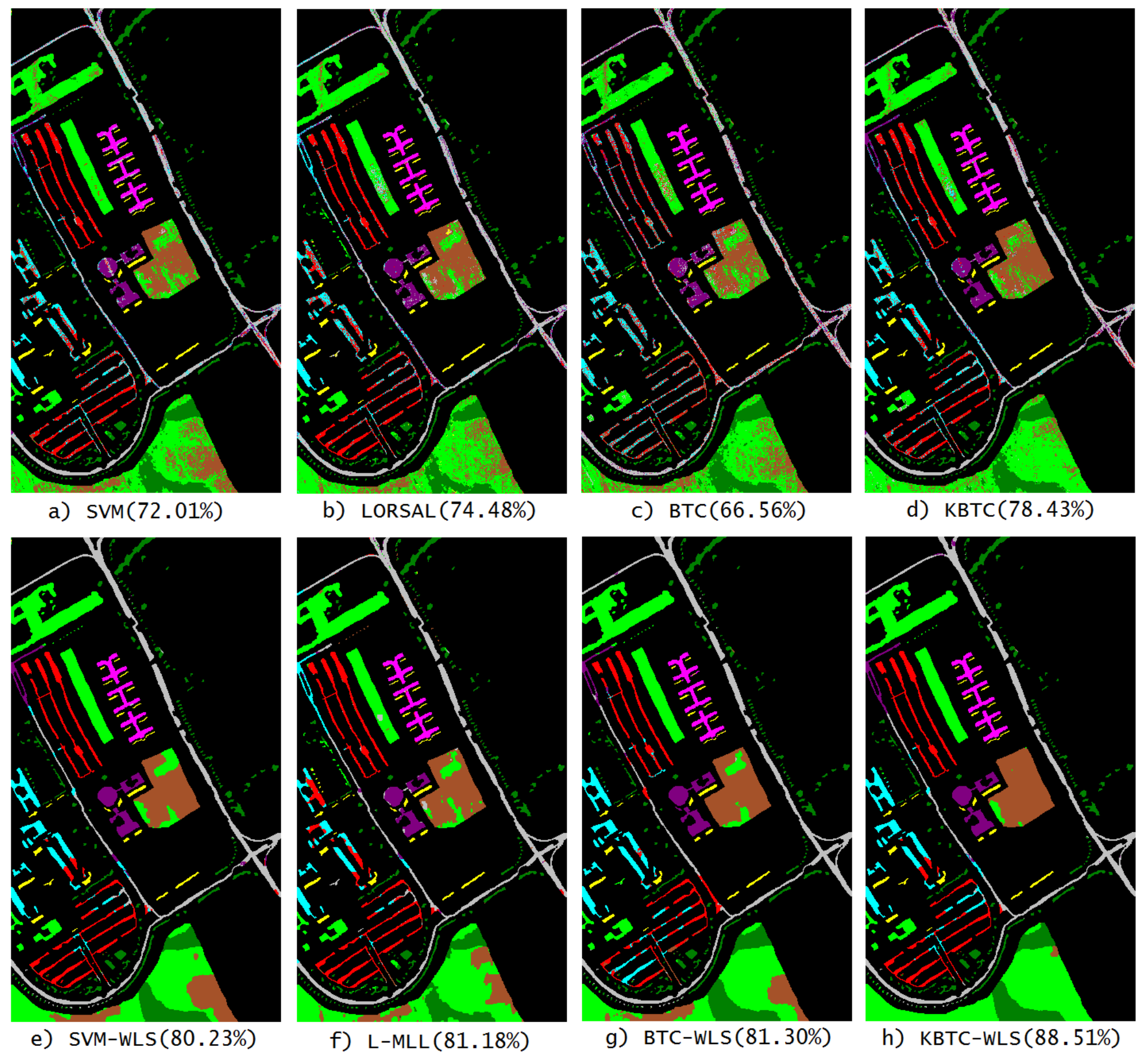}
\centering
\caption{ Classification Maps on Pavia University Dataset with Overall Accuracies}
\label{paviaUniClassMap_hyper2}
\end{figure*}

Finally, the last experiment was performed on the Pavia University dataset using the training and testing samples given in Fig. \ref{paviaUniversity_hyper2}. The classification results are given in Table \ref{resultSalinas_hyper2}. We also included the corresponding classification maps with OAs(\%) in Fig. \ref{paviaUniClassMap_hyper2}. The results obtained both in the spectral-only and spatial-spectral cases show that this dataset is the most difficult one. The pixels of this HSI are quite mixed and those belonging to different classes are highly non-linearly separable. The performance differences between the linear similarity-based BTC and the kernelized approaches are significant. KBTC achieves about $12\%$ overall accuracy improvement over the BTC technique by means of RBF kernel. This time the LORSAL method performs about $2.5\%$ better than the SVM technique. However, its performance is about $4\%$ less than that of KBTC. Although BTC achieves quite low results in the spectral-only case, BTC-WLS closes the gap between the other approaches using the smoothed residual maps. It even outperforms the SVM-GF and L-MLL methods. On the other hand, KBTC-WLS performs about $7\%$ better than BTC-WLS in terms of OA.

\section{Conclusions}
\label{sec:conc_hyper2}
In this chapter, we proposed a non-linear kernel version of the previously introduced basic thresholding classification algorithm for HSI classification. The proposed method achieves significant performance improvement over the linear version of it especially in the experiments in which the samples of the classes are linearly non-separable. The classification results on the publicly available datasets showed that the proposed algorithm also outperforms the well-known RBF kernel SVM and recently introduced logistic regression-based LORSAL technique. Based on the weighted least squares filter, we also presented the spatial-spectral version of the proposal, which achieves better performances as compared to the recently introduced state-of-the-art spatial-spectral approaches such as SVM-WLS (GF) and L-MLL. Another significance of the proposed framework is that the threshold and the kernel parameter could be easily estimated via the procedures we provided in Chapter \ref{chp:kbtc} without any cross validation or experiment.


\chapter{CONCLUDING REMARKS}
\label{chp:conclusions}
\section{Summary}
In this thesis, we have addressed the problem of classification in computer vision and pattern recognition by introducing two sparsity-based methods. While the first algorithm (BTC) refers the applications involving linearly separable data, the other one (KBTC) addresses the problems consisting of non-linearly separable classes. The techniques are easy to understand and require a few steps to implement. In some challenging applications such as face recognition and hyper-spectral image classification, we have shown that the proposed approaches achieve state-of-the-art classification accuracies as compared to the strongest classifiers in the literature. They also outperform those methods by classifying the given testing samples extremely rapidly. The proposals require a few parameters which could be determined via efficient off-line procedures. These procedures do not involve experiments such as cross validation in which the parameters are determined experimentally. Moreover, we have proposed some problem-specific fusion techniques which significantly improve the classification performances of our individual classifiers. For instance, in face recognition, the fusion is performed by means of taking the average of the output residuals provided by individual classifiers having different random projections. In case of HSI classification, this is achieved by smoothing the output residual maps using recently introduced edge preserving filtering techniques. The proposed fusion mechanisms could also be applied to other sparsity-based classification algorithms. We believe that BTC and KBTC algorithms together constitute a complete classification framework. 

\section{Discussion}
Although the proposed algorithms are based on sparse representation, they significantly differ from the other sparsity-based techniques when performing sparse recovery. It is known that conventional methods ($l_1$ minimization, greedy pursuits) use iterative expressions at this stage. However, our proposals consist of non-iterative structures such as thresholding and Tikhonov regularized sparse code estimation which result in fast classification operation. Other than speed issues, it is unclear if an iteration-based approach can perform satisfactory results or it can converge in most cases. Further studies are required on the robustness of the iteration-based methods for classification applications. 

At the dictionary pruning stages of the proposed techniques, we apply a fixed thresholding policy which has been shown to be robust. On the other hand, it is doubtful whether an adaptive pruning stage will improve the classification accuracies or not. Even so, it is difficult to adapt a correlation based adaptive stage and the performance improvement is not guaranteed. It is also worth mentioning that the kernelized version of the proposal uses RBF kernel which is quite common and popular kernel function in classification applications. It is suspicious if the other type of kernel functions such as polynomial kernel will improve the performance. We also note that KBTC is superior to BTC especially in non-linearly separable cases. However, it does not mean that it is always superior. In some applications, linear algorithms outperform the non-linear ones.    

\section{Future Directions}
There are many future directions related to the proposed algorithms to consider. Let us mention them one by one:

\begin{itemize}
\item As we can observe, both proposals involve inverse matrix operation. Using the properties of symmetric positive definite matrices, the inverse operation could be performed more efficiently in order to reduce the computational cost. Also for this purpose, the properties of Gram matrices could be further investigated.
\item Currently, selection of the pruned dictionary is performed based on linear and non-linear correlations. More sophisticated approaches could be utilized in order to improve this step.  
\item We have measured the performances of the proposed techniques using the applications involving elementary features or those containing simple feature projections. It is required to investigate them under more advanced transform techniques such as scale-invariant feature transform (SIFT) \cite{lowe1999object,lowe2004distinctive}, histogram of oriented gradients (HOG) \cite{dalal2005histograms}, Hough transform \cite{ballard1981generalizing}, etc.   
\item In HSI classification, spatial-spectral extensions of the proposals currently utilize the gray-scale guidance image obtained via PCA of the given HSI in the edge preserving smoothing stages. We believe that filtering the guidance image using an aggressive edge-aware filter such as $L_0$ smoothing technique \cite{xu2011image} further improves the classification performance.    
\item It would be interesting to investigate the performance of the KBTC algorithm under multiple kernel learning framework \cite{bach2004multiple,sonnenburg2006large} based on the fact that the real world data in the feature space is highly heterogeneous. Another future direction could be adapting a dictionary learning \cite{mairal2009online} stage which may improve the classification accuracy as well as the computational efficiency.          
\end{itemize} 



\bibliographystyle{IEEEtran}	
%
\bibliography{publications}

\appendix
\chapter{MATLAB CODES}
The following function implements the BTC algorithm.
\begin{lstlisting}[]
function [btcResult, errMatrix] = btc(labels, A, Y, M, alpha)
 &% This function implements the basic thresholding classifier for any
 &% classification applications such as face recognition, hyperspectral
 &% image classification, etc.
 &%
 &% INPUTS:
 &% labels: 1 X N vector of integers consisting of class labels
 &% of the input training samples where N denotes the number of 
 &% training samples.
 &% A: B X N dictionary whose columns represent L2 normalized training
 &% samples where B represents the number of features.
 &% Y: B X L matrix consisting of testing samples where L represents 
 &% the number of testing samples. For fast computation, it is 
 &% recommended that L is less than 10000. If it is larger, then Y can
 &% be partitioned.
 &% M: Threshold parameter (integer) which is less than B.
 &% alpha: Regularization parameter between 0 and 1.
 &%
 &% OUTPUTS:
 &% errMatrix: C X L error matrix containing residual for each sample
 &% btcResult: 1 X L decision vector containing predicted labels
 
 [~, N] = size(A);
 [~, L] = size(Y);
 C = max(labels); &% C shows the number of classes
 X = zeros(N, L); &% initialize the sparse codes for each test sample
 &% calculate gram matrix, this can be performed outside the func
 mappedTrains = A' * A;
 &% linear correlation vector is calculated for each test sample
 mappedAy = A' * Y;
 &% correlation vectors are sorted in descending order
 [~, sortedLabels] = sort(abs(mappedAy), 'descend');
 I = alpha * eye(M);  &% regularization matrix
 for k = 1:L &% sparse code vector is calculated for each sample
     support = sortedLabels(1:M, k);
     X(support, k) = (mappedTrains(support,support) + I) \ ...
     mappedAy(support,k);
 end
 errMatrix = zeros(C, L);  &% initialize error (residual) matrix
 &% calculate class-wise regression error vector for each sample
 for classIndex = 1:C
     ind = (labels == classIndex);
     Yp = A(:,ind)*X(ind, :);
     errMatrix(classIndex,:) = sqrt(sum((Y-Yp).^2, 1));
 end
 &% make decision for each sample based on minimum residual
 [~, btcResult] = min(errMatrix, [], 1);
end
\end{lstlisting}
The function below is used to calculate $\overline{\beta}$ quantity for BTC. It also estimates the threshold parameter $M$.
\begin{lstlisting}[]
function [avgBeta, bestM] = averageBetaBtc(labels, A, alpha)
 &% This function implements the quantity average beta for btc.
 &%
 &% INPUTS:
 &% labels: 1 X N vector of integers consisting of class labels
 &% of the input training samples where N denotes the number of 
 &% training samples.
 &% A: B X N dictionary whose columns represent L2 normalized training
 &% samples where B represents the number of features.
 &% alpha: Regularization parameter between 0 and 1.
 &%
 &% OUTPUTS:
 &% avgBeta: 1 X B-1 vector containing beta values. Plot it and see at
 &% which index it becomes minimum. This index gives the best M.
 &% bestM: The best value of the threshold M (estimated)
 
[B, N] = size(A);
C = max(labels);
avgBeta = zeros(1, B-1);
mappedTrains = A' * A;
[~, sortedLabels] = sort(abs(mappedTrains), 'descend');
&% for each M, determine the beta value
for m=1:B-1
    X = zeros(N, N);
    I = alpha * eye(m);
    for k = 1:N
        support = sortedLabels(2:m+1, k);
        X(support, k) = (mappedTrains(support,support) + I) \...
            mappedTrains(support,k);
    end
    &% for each sample determine the residual
    errMatrix = zeros(C, N);
    for classIndex = 1:C
        ind = (labels == classIndex);
        Yp = A(:,ind)*X(ind, :);
        errMatrix(classIndex,:) = sqrt(sum(abs(A-Yp).^2,1));
    end
    &% calculate the average beta
    [res, sorted] = sort(errMatrix);
    res1 = res(1,:);
    res2 = res(2,:);
    ind = (labels ~= sorted(1,:));
    res2(ind) = res1(ind);
    trueResiduals = errMatrix(labels + (0:N-1)*C);
    avgBeta(m) = mean(trueResiduals./res2);
    
end
[~, bestM] = min(avgBeta); &% estimate the threshold M. 

end
\end{lstlisting}
The following function implements the KBTC algorithm. Please make sure that the columns of the training and testing matrices are normalized as described in Chapter 4. 
\begin{lstlisting}[]
function [kbtcResult, errMatrix] = kbtc(labels, A, Y, M, alpha, gamma)
 &% This function implements the kernel basic thresholding classifier 
 &% for any classification applications such as face recognition,
 &% hyperspectral image classification, etc.
 &%
 &% INPUTS:
 &% labels: 1 X N vector of integers consisting of class labels
 &% of the input training samples where N denotes the number of 
 &% training samples.
 &% A: B X N dictionary whose columns represent normalized training
 &% samples where B represents the number of features.
 &% Y: B X L matrix consisting of testing samples where L represents 
 &% the number of testing samples. For fast computation, it is 
 &% recommended that L is less than 10000. If it is larger, then Y can
 &% be partitioned.
 &% M: Threshold parameter (integer) which is less than B.
 &% alpha: Regularization parameter between 0 and 1.
 &% gamma: Kernel parameter.
 &%
 &% OUTPUTS:
 &% errMatrix: C X L error matrix containing residual for each sample
 &% kbtcResult: 1 X L decision vector containing predicted labels
 
[~, N] = size(A);
[~, L] = size(Y);
C = max(labels); &% C shows the number of classes
X = zeros(N, L); &% initialize sparse codes for each sample
&% calculate gram matrix, this can be performed outside the func
mappedTrains = kernelFunction(gamma, A, A);
&% calculate nonlinear correlations
mappedAy = kernelFunction(gamma, A, Y); 
[~, sortedLabels] = sort(abs(mappedAy), 'descend');
I = alpha * eye(M); &% regularization matrix
 
for k = 1:L
    support = sortedLabels(1:M, k);
    X(support, k)  = (mappedTrains(support,support) + I) \...
        mappedAy(support,k);
end
errMatrix = zeros(C, L);&% initialize error (residual) matrix
&% calculate class-wise regression error vector for each sample
c1 = kernelFunction(gamma,Y);
for classIndex = 1:C
    ind = (labels == classIndex);
    KA = mappedTrains(ind,ind); 
    c2 = -2*dot(X(ind,:),mappedAy(ind,:));
    c3 = X(ind,:)'*KA; 
    c3 = dot(c3', X(ind,:));
    errMatrix(classIndex,:) = sqrt(abs(c1 + c2 + c3));
end
 &% make decision for each sample based on minimum residual
[~, kbtcResult] = min(errMatrix, [], 1);
end

\end{lstlisting}
The function below is used to determine the parameters of the KBTC algorithm.
\begin{lstlisting}[]
function [gamma, M] = determineKbtcParams(labels, A, alpha)
 &% This function determines the parameters of kbtc
 &%
 &% INPUTS:
 &% labels: 1 X N vector of integers consisting of class labels
 &% of the input training samples where N denotes the number of 
 &% training samples.
 &% A: B X N dictionary whose columns represent normalized training
 &% samples where B represents the number of features.
 &% alpha: Regularization parameter between 0 and 1.
 &%
 &% OUTPUTS:
 &% gamma: Estimated kernel parameter
 &% M: Estimated threshold
 
[B, ~] = size(A);
minBeta = inf;
bestGamma = 0;
bestGammaIndex = 1;
k = 1;
&% call average beta function for each gamma and determine the best
&% gamma at which beta becomes minimum.
for gammaIndex = 5:-1:-10
    gamma = 2^(gammaIndex);
    avgBetaGamma(k,:) = averageBetaKbtc(labels, A, alpha, gamma);
    avgBeta = mean(avgBetaGamma(k,:));
    if avgBeta < minBeta
        minBeta = avgBeta;
        bestGamma = gamma;
        bestGammaIndex = k;
    end
    disp(['gamma = ', num2str(gamma), ', avg_beta = ', num2str(avgBeta)])
    k = k + 1;
end
gamma = bestGamma;
disp(['best gamma = ', num2str(bestGamma)]);
&% estimate the threshold parameter
figure;
plot(1:B-1, avgBetaGamma(bestGammaIndex, :));
xlabel('threshold (M)');
ylabel('average beta');
[~, M] = min(avgBetaGamma(bestGammaIndex, :));
disp(['best threshold = ', num2str(M)]);
end
\end{lstlisting}
The following function implements the $\overline{\beta}$ quantity for KBTC.
\begin{lstlisting}[]
function [avgBeta] = averageBetaKbtc(labels, A, alpha, gamma)
 &% This function implements the quantity average beta for kbtc.
 &%
 &% INPUTS:
 &% labels: 1 X N vector of integers consisting of class labels
 &% of the input training samples where N denotes the number of 
 &% training samples.
 &% A: B X N dictionary whose columns represent normalized training
 &% samples where B represents the number of features.
 &% alpha: Regularization parameter between 0 and 1.
 &% gamma: RBF Kernel parameter.
 &%
 &% OUTPUTS
 &% avgBeta: 1 X B-1 vector containing beta values. Plot it and see at
 &% which index it becomes minimum. This index gives the best M.
 
 
[B, N] = size(A);
C = max(labels);
avgBeta = zeros(1, B-1);
mappedTrains = kernelFunction(gamma, A, A);
[~, sortedLabels] = sort(abs(mappedTrains), 'descend');

for m =  1:B-1 &% for each M, determine the beta value
    X = zeros(N, N);
    I = alpha * eye(m);
    for k = 1:N
        support = sortedLabels(2:m+1, k);
        X(support, k)  = (mappedTrains(support,support) + I) \...
            mappedTrains(support,k);
    end
    errMatrix = zeros(C, N);
    &% for each sample determine the residual
    c1 = kernelFunction(gamma,A);
    for classIndex = 1:C
        ind = (labels == classIndex);
        KA = mappedTrains(ind,ind);
        c2 = -2*dot(X(ind,:), mappedTrains(ind,:));
        c3 = X(ind,:)'*KA;
        c3 = dot(c3',X(ind,:));
        errMatrix(classIndex,:) = sqrt(abs(c1 + c2 + c3));
    end
    &% calculate the average beta
    [res, sorted] = sort(errMatrix);
    res1 = res(1,:);
    res2 = res(2,:);
    ind = (labels ~= sorted(1,:));
    res2(ind) = res1(ind);
    trueResiduals = errMatrix(labels + (0:N-1)*C);
    avgBeta(m) = mean(trueResiduals./res2);
    
end

end


\end{lstlisting}
The function below is used to determine the kernel matrix for KBTC algorithm.
\begin{lstlisting}[]
function [corr] = kernelFunction(gamma, D1, D2)
&% This function implements the RBF kernel
&% 
&% INPUTS:
&% gamma: Kernel parameter
&% D1: First matrix
&% D2: Second matrix
&% OUTPUT:
&% corr: Correlation matrix or vector

if nargin > 2
    n1sq = sum(D1.^2,1);
    n1 = size(D1,2);
    n2sq = sum(D2.^2,1);
    n2 = size(D2,2);
    c = (ones(n2,1)*n1sq)' + ones(n1,1)*n2sq -2*(D1'*D2);
else
    c = sum((D1-D1).^2,1);
end
corr = exp(-gamma*c);
end
\end{lstlisting}
\curriculumvitae
\label{chapter:vita}

\section*{\uppercase{Personal Information}}

\textbf{Surname, Name: } Toksöz, Mehmet Altan\\
\textbf{Nationality:} Turkish (TC) \\
\textbf{Date and Place of Birth:} 01.05.1982, Ankara\\
\textbf{Marital Status:} Single \\
\textbf{Phone:} +90 506 959 16 09 \\
\textbf{Fax:} +90 312 291 6004 \\

\section*{\uppercase{Education}}

\begin{tabular}{lll}
\textbf{Degree} & \textbf{Institution} & \textbf{Year of Graduation} \\
M.S. & Bilkent University & 2009 \\
 &  Electrical and Electronics Engineering &  \\
B.S. & University of Applied Sciences Upper Austria & 2006 \\
 & Electrical Engineering and Computer Science & \\
B.S. & Anadolu University & 2006\\
 &  Electrical and Electronics Engineering & \\
\end{tabular}

\section*{\uppercase{Professional Experience}}

\begin{tabular}{lll}
\textbf{Year} & \textbf{Place} & \textbf{Enrollment} \\
2010-2016 & Tübitak Bilgem İltaren & Senior Researcher \\
2007-2009 & Bilkent University & Research and Teaching Assistant\\ 
2006-2007 & Akad Elektronik & Software Engineer
\end{tabular}

\section*{\uppercase{HONORS, SCHOLARSHIPS AND AWARDS}}

\begin{itemize}
\item  Ranked 79th in the quantitative area among 1.5 million students in the National
University Entrance Exam of 2001.
\item Scholarship for undergraduate study in Austria for one year, 2005-2006.
\item Full scholarship for Master of Science Study awarded by Bilkent University.
\item Master of Science scholarship awarded by the Scientific and Technological Research
Council of Turkey.
\end{itemize}

\section*{\uppercase{Publications}}
\subsection*{International Journal Papers}

\begin{itemize}
\item M. A. Toksoz and I. Ulusoy, “Hyperspectral image classification via kernel basic thresholding classifier,” IEEE Transactions on Geoscience and Remote Sensing, vol. 55, no. 2, pp. 715–728, 2017.

\item M. A. Toksöz and I. Ulusoy, “Hyperspectral image classification via basic thresholding classifier,” IEEE Transactions on Geoscience and Remote Sensing, vol. 54, no. 7, pp. 4039–4051, 2016.

\item M. A. Toksöz and I. Ulusoy, “Classification via ensembles of basic thresholding classifiers,” IET Computer Vision, vol. 10, no. 5, pp. 433–442, 2016.

\item N. Akar and M. A. Toksoz, “MPLS Automatic Bandwidth Allocation via Adaptive Hysteresis”, Computer Networks, vol. 55, no. 5, pp. 1181-1196, April 2011.

\item M. A. Toksoz and N. Akar, “Dynamic Threshold-based Assembly Algorithms for Optical Burst Switching Networks Subject to Burst Rate Constraints”, Photonic Network Communications, vol. 20, no. 2, pp. 120-130, Oct. 2010. 
\end{itemize}

\end{document}